\documentclass{article}

\usepackage{microtype}
\usepackage{graphicx}
\usepackage{booktabs} % for professional tables
\usepackage{multirow}
\usepackage{multicol}
\usepackage{tablefootnote}
\usepackage{xcolor}
\usepackage{fix-cm} % arbitrary font scaling
\DeclareMathSizes{10}{10}{6}{4}

\usepackage{hyperref}

\usepackage[accepted]{icml2024}
\usepackage{amsmath}
\usepackage{amssymb}
\usepackage{mathtools}
\usepackage{amsthm}
\usepackage{caption}
\usepackage{subcaption}
\usepackage{booktabs}
\usepackage[flushleft]{threeparttable}
\usepackage{MnSymbol}
\usepackage{lipsum}
\usepackage{algorithmic}

\usepackage{enumitem}
\usepackage[capitalize,noabbrev]{cleveref}

\makeatletter
\newtheorem*{rep@theorem}{\rep@title}
\newcommand{\newreptheorem}[2]{%
\newenvironment{rep#1}[1]{%
 \def\rep@title{#2 \ref{##1}}%
 \begin{rep@theorem}}%
 {\end{rep@theorem}}}
\makeatother

\theoremstyle{plain}
\newtheorem{theorem}{Theorem}[section]
\newreptheorem{theorem}{Theorem}

\newtheorem{lemma}[theorem]{Lemma}
\newreptheorem{lemma}{Lemma}

\newreptheorem{corollary}{Corollary}

\theoremstyle{definition}

\theoremstyle{remark}

\usepackage[textsize=tiny]{todonotes}

\usepackage{amsmath}
\usepackage{amssymb}
\usepackage{ulem}
\usepackage{bm}
\usepackage[flushleft]{threeparttable}

\newcommand\bolden[1]{{\boldmath\bfseries#1}}
\newcommand{\smallsection}[1]{{\vspace{0.02in} \noindent {\bolden{\uline{\smash{#1}}}}}}

\definecolor{lucky}{RGB}{120, 130, 150}
\usepackage{graphicx}
\usepackage{xspace}
\usepackage{amsthm}
\usepackage{amsfonts}

\def\mydefbb#1{\expandafter\def\csname bb#1\endcsname{\ensuremath{\mathbb{#1}}}}
\def\mydefallbb#1{\ifx#1\mydefallbb\else\mydefbb#1\expandafter\mydefallbb\fi}
\mydefallbb ABCDEFGHIJKLMNOPQRSTUVWXYZ\mydefallbb

\def\mydefcal#1{\expandafter\def\csname cal#1\endcsname{\ensuremath{\mathcal{#1}}}}
\def\mydefallcal#1{\ifx#1\mydefallcal\else\mydefcal#1\expandafter\mydefallcal\fi}
\mydefallcal ABCDEFGHIJKLMNOPQRSTUVWXYZ\mydefallcal

\newcommand{\abs}[1]{\left|#1\right|}
\newcommand{\set}[1]{\{#1\}}

\newcommand{\black}[1]{{\color{black}{#1}}}

\newcommand{\Var}{\ensuremath{\operatorname{Var}}}

\renewcommand{\setminus}{\text{\textbackslash}}

\newcommand{\hdot}{\ensuremath{\mathbf{\Tilde{h}}(\cdot)}\xspace}
\newcommand{\Hdot}{\ensuremath{{\Tilde{H}}(\cdot)}\xspace}

\newcommand{\ax}{$\mathrm{A}\text{-}\mathrm{X}$ dependence\xspace}

\newcommand{\hg}{\ensuremath{\mathbf{\Tilde{h}}^{(G)}}\xspace}
\newcommand{\hp}{\ensuremath{\mathbf{\Tilde{h}}^{(p)}_{ij}}\xspace}
\newcommand{\hn}{\ensuremath{\mathbf{\Tilde{h}}^{(v)}_i}\xspace}
\newcommand{\hc}{\ensuremath{\mathbf{h}_c}\xspace}

\newcommand{\gmodel}{CSBM-X\xspace}

\theoremstyle{plain}
\newtheorem{innercustomgeneric}{\customgenericname}
\providecommand{\customgenericname}{}
\newcommand{\newcustomtheorem}[2]{%
  \newenvironment{#1}[1]
  {%
   \renewcommand\customgenericname{#2}%
   \renewcommand\theinnercustomgeneric{##1}%
   \innercustomgeneric
  }
  {\endinnercustomgeneric}
}

\newcustomtheorem{customprop}{Property}
\newcustomtheorem{customobs}{Observation}
\DeclareMathOperator*{\argmin}{arg\,min}

\icmltitlerunning{\black{Feature Distribution on Graph Topology Mediates the Effect of Graph Convolution}}

\begin{document}

\twocolumn[
\icmltitle{Feature Distribution on Graph Topology \\ Mediates the Effect of Graph Convolution: Homophily Perspective}



\icmlsetsymbol{equal}{*}

\begin{icmlauthorlist}
\icmlauthor{Soo Yong Lee}{KAIST_AI}
\icmlauthor{Sunwoo Kim}{KAIST_AI}
\icmlauthor{Fanchen Bu}{KAIST_AI,KAIST_EE}
\icmlauthor{Jaemin Yoo}{KAIST_EE}
\icmlauthor{Jiliang Tang}{MSU}
\icmlauthor{Kijung Shin}{KAIST_AI,KAIST_EE}
\end{icmlauthorlist}

% \icmlaffiliation{KAIST}{KAIST, Daejeon, the Republic of Korea}
% \icmlaffiliation{MSU}{Michigan State University, Michigan, the United States}

\icmlaffiliation{KAIST_AI}{Kim Jaechul Graduate School of Artificial Intelligence, KAIST, Daejeon, Republic of Korea}
\icmlaffiliation{KAIST_EE}{School of Electrical Engineering, KAIST, Daejeon, Republic of Korea}
\icmlaffiliation{MSU}{Department of Computer Science and Engineering, Michigan State University, Michigan, US}

\icmlcorrespondingauthor{Kijung Shin}{kijungs@kaist.ac.kr}

\icmlkeywords{}

\vskip 0.3in
]



\printAffiliationsAndNotice{}  % leave blank if no need to mention equal contribution
% \onecolumn
\begin{abstract}
How would randomly shuffling feature vectors among nodes from the same class affect graph neural networks (GNNs)?
The feature shuffle, intuitively, perturbs the dependence between graph topology and features ({\ax}) for GNNs to learn from.
Surprisingly, we observe a consistent and significant improvement in GNN performance following the feature shuffle.
Having overlooked the impact of \ax on GNNs, the prior literature does not provide a satisfactory understanding of the phenomenon.
Thus, we raise two research questions.
{First, how should \ax be measured, while controlling for potential confounds?}
Second, how does \ax affect GNNs?
In response, we 
(\textit{i}) propose a principled measure for \ax, 
(\textit{ii}) design a random graph model that controls \ax,
(\textit{iii}) establish a theory on how \ax relates to graph convolution, and
(\textit{iv}) present empirical analysis on real-world graphs that align with the theory.
We conclude that \ax mediates the effect of graph convolution, 
{such that smaller dependence improves GNN-based node classification.}
\end{abstract}

\section{Introduction}\label{sec:intro}
\begin{figure}[t] 
    \centering
    
    \begin{subfigure}[b]{\linewidth}
    \centering
    \includegraphics[scale=0.35]{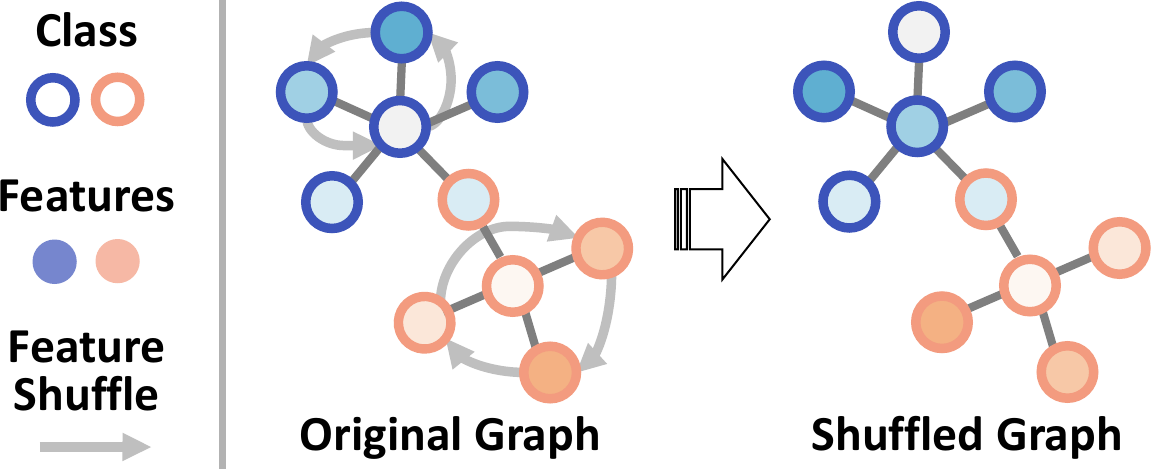}
    \caption{Visualization of the feature shuffle. 
    Features of the same class nodes are shuffled. 
    The shuffled node ratio is 0.6 in the example.}
    \end{subfigure}

    \vspace{2mm}

    \begin{subfigure}[b]{\linewidth}
    \centering
    \includegraphics[scale=0.21]{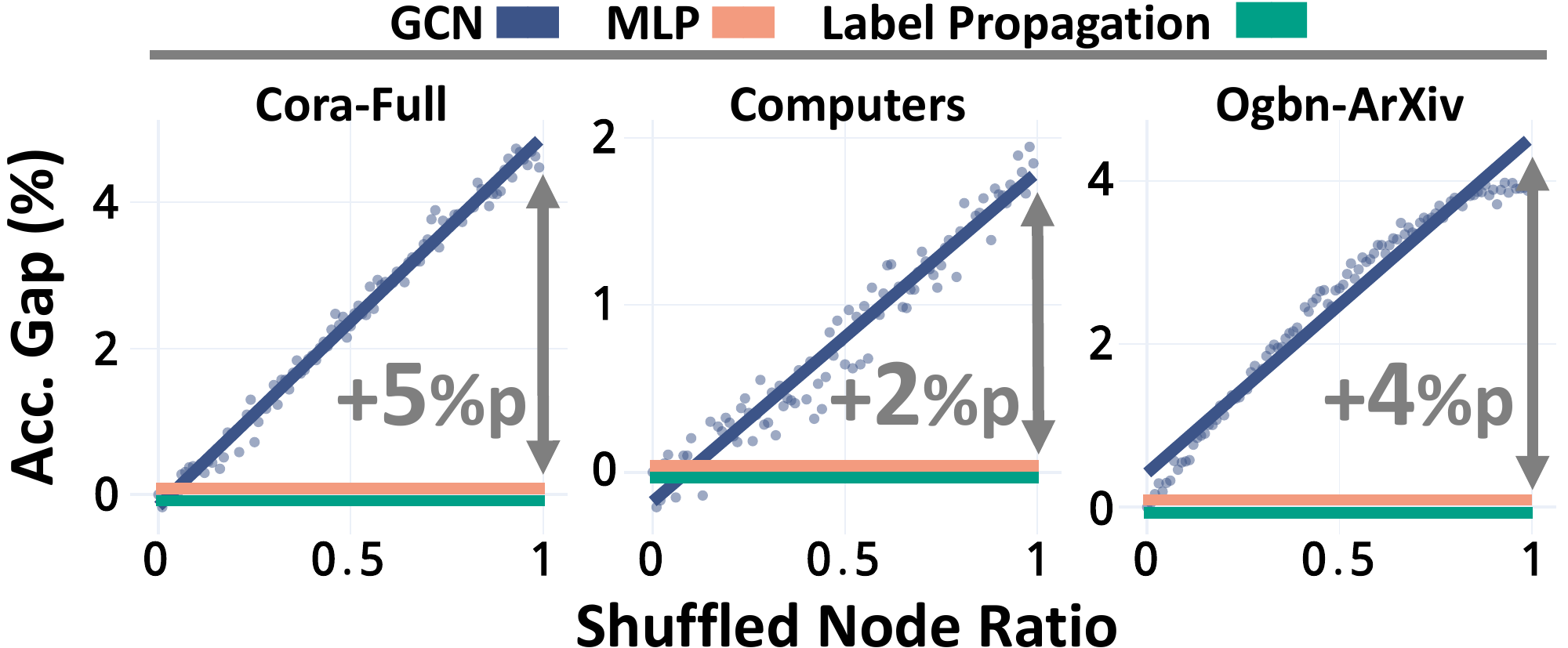}
    \caption{Accuracy gaps between the original and shuffled graphs.}
    \end{subfigure}

    \vspace{2mm}

    \caption{\label{fig:strange_phe}
    \bolden{An Intriguing Phenomenon.}
    GCN performance increases significantly over the feature shuffle, 
    while those of MLP and label propagation remain stationary.
    }
    \vspace{2mm} 
\end{figure}

Graph neural networks (GNNs) are functions of graph topology and features.
Understanding the conditions in which GNNs become powerful is the key to their improvement and effective applications.
As such, many prior works have investigated the conditions that affect GNN effectiveness, especially from the node representation learning perspective~\cite{Oono2020GraphClassification, Abboud2021TheInitialization,  You2021Identity-awareNetworks, Wang2022HowNetworks, Wei2022UnderstandingPerspective, Wu2023ANetworks, Baranwal2021GraphGeneralization, Baranwal2023EffectsNetworks}.

However, in this work, we report an intriguing phenomenon not well accounted for by the prior studies.
How would randomly shuffling feature vectors among nodes from the same class affect GNNs?
The feature shuffle, intuitively, disrupts the dependence between graph topology and node features (\ax).  
Rather surprisingly, increasing the shuffled node ratio consistently improves GCN~\cite{Kipf2017Semi-supervisedNetworks} performance (Fig.~\ref{fig:strange_phe}).
The performances of MLP and label propagation (LP), however, remain the same (for experiment details, refer to Sec.~\ref{sec:exp_setting}).

The prior studies on GNN theory do not provide a satisfactory understanding of the phenomenon.
One line of studies indicates how class distribution on graph topology, such as class-homophily, can be critical for effective GNNs~\cite{Luan2022RevisitingNetworks, Luan2023WhenDistinguishability, Ma2022IsNetworks, Mao2023DemystifyingAll, Platonov2023CharacterizingBeyond}. 
The feature shuffle, however, does not intervene with label distribution because the class labels are not shuffled, 
causing the LP performance to be unchanged.

Some studies point out feature informativeness for node class as another crucial factor for effective GNNs~\cite{Baranwal2021GraphGeneralization, Baranwal2023EffectsNetworks, Wei2022UnderstandingPerspective, Wu2023ANetworks}. 
These do not explain the reported phenomenon, either, since the feature shuffle is done only among nodes from the same class. 
Namely, feature informativeness for node class remains the same, leaving the MLP performance unaffected.

Others stress GNN's efficacy as a node signal denoiser~\cite{NT2019RevisitingFilters, Ma2021ADenoising}.
From such perspective, whether the signals are well denoised after graph convolution is important for effective node classification~\cite{Luan2023WhenDistinguishability}.
However, they do not discuss conditions in which the convoluted features are well-denoised.
The reason, thus, is vague for why the feature shuffle affects GNN performance.

The limitation of the prior works in understanding the observed phenomenon stems from overlooking the impact \ax on GNNs.
Thus, we \textit{advance the findings from prior works} by investigating how \ax affects GNNs.
We raise two research questions (RQs).
\begin{itemize}[leftmargin=*]
    \item \textbf{RQ1}. How should \ax be measured, {while controlling for potential confounds}? 
    \vspace{-0.5mm}
    \item \textbf{RQ2}. How does \ax affect GNNs?
\end{itemize}
In response, we propose a principled measure for \ax, \textit{class-controlled feature homophily} (CFH), that mitigates potential confounding by node class.
We propose a random graph model, \textit{\gmodel}, to control CFH.
With the measure, graph model, and feature shuffle, 
we establish a theory that CFH mediates the effect of graph convolution.
Specifically, CFH moderates its force to pull each node feature toward the feature mean of the respective node class, with smaller CFH increasing the force.

\section{Preliminaries}\label{sec:prelim_background}

\smallsection{Graphs.}
A graph $G = (V, E)$ is defined by a \textit{node set} $V = V(G)$
and an \textit{edge set} $E = E(G) \subseteq \binom{V}{2}$.
We denote an edge between two nodes $u$ and $v$ as $(u,v) \in E$, and $(u,v) = (v,u)$ holds unless otherwise stated.

Let $n = n(G)$ denote the number of nodes in $G$ with $V = \set{v_i \colon i \in [n]}$.
Let $X = X(G) \in \mathbb{R}^{n \times k}$ denote node feature matrix, where the $i$-th row corresponds to the feature vector $X_i \in \mathbb{R}^{k}$ of node $v_i \in V$, where $k = k(G)$ is the feature dimension.
For each node $v_i \in V$,
its class is $Y_i \in [c]$, where $c$ is the number of node classes.
Its neighbor set is $N_i = \{v_{j} \in V \colon (v_{i}, v_{j}) \in E\}$.
Its degree is $d_i = \abs{N_i}$, and its same- and different-class degrees are
$d^{+}_i = \abs{\{v_{j} \in N_i \colon Y_{j} = Y_{i}\}}$ and 
$d^{-}_i = \abs{\{v_{j} \in N_i \colon Y_{j} \neq Y_{i}\}}$, respectively, with
$d_i = d^{+}_i + d^{-}_i$.

We define $V'_{i} = V$ \textbackslash $\{v_{i}\}$ as the set of nodes excluding $v_i$.
Also, for each class $\ell \in [c]$, 
we use $C^+_\ell = \set{ v_i \in V \colon Y_i = \ell }$ to denote node set of class $\ell$,
and $C^-_\ell = V \text{\textbackslash} C^+_\ell$ denotes the rest.

\smallsection{Feature distance (FD)}.
For measuring feature distance between classes, we adopt 
a simplified version of Bhattacharyya distance~\cite{kailath1967divergence}.
Specifically, given two data classes $C_{0}$ and $C_{1}$ 
with feature means $\mu_i \in \mathbb{R}^{k}$ and covariance matrices $\Sigma_i \in \mathbb{R}^{k \times k}$ with $i = 0$ or $1$ respectively,
we define the feature distance FD between $C_{0}$ and $C_{1}$ as:
\begin{equation}\label{eq.fd}
    \text{FD} (C_0, C_1) \coloneqq  
    \sqrt{(\mu_0 - \mu_1)^\intercal \left( \frac{\Sigma_0 + \Sigma_1}{2} \right)^{-1}(\mu_0 - \mu_1)}.
\end{equation}
A higher feature distance FD indicates a larger (normalized) distance between the two classes, i.e., the two classes are more distinct.
{If both classes follow a Gaussian distribution, 
roughly speaking, the difficulty in classifying $C_0$ and $C_1$ decreases as $\text{feature distance FD}(C_0, C_1) \in [0, \infty)$ increases~\cite{kailath1967divergence}.}

\smallsection{Homophily}. 
From a network perspective, homophily (\textit{love of the same})
refers to the \textit{positive dependence} between node similarity and connection~\cite{McPherson2001BirdsNetworks}.
Heterophily (\textit{love of the different}) is considered as the opposite, 
describing the \textit{negative dependence} in that dissimilar nodes tend to connect~\cite{Rogers2019DiffusionInnovations}.
Importantly, we distinguish {\textit{impartiality}} from both
for networks having \textit{no dependence} between node similarity and their connection.

The vast majority of works on GNN-homophily connection focus specifically on class-homophily. 
We use $\mathbf{h}_{c}$ to denote the class-homophily defined by~\citet{Lim2021LargeMethods}:
\begin{equation}\label{eq.class_homophily}
    \mathbf{h}_{c} = \frac{1}{c} \sum_{\ell \in [c]}
    % \left[
    \texttt{max}\left(
    \frac
    {\sum_{v_i \in C^+_\ell}  d^+_i }
    {\sum_{v_i \in C^+_\ell}  d_i }
    - \frac{\vert C^+_\ell \vert}{\vert V \vert}, ~0\right)
    % \right]_{+}.
\end{equation}

\smallsection{Contextual stochastic block models (CSBMs)}. 
Stochastic block models (SBMs) are widely used graph models for network analysis~\cite{Holland1983StochasticSteps}, with distinct communities, or blocks, consisting of same-class nodes.
CSBMs~\cite{Deshpande2018ContextualModels} supplement SBMs by considering node features.
Recently, many researchers have used CSBMs and developed their variants for GNN analysis~\cite{Wei2022UnderstandingPerspective, Palowitch2022Graphworld:Gnns, Wu2023ANetworks, Baranwal2021GraphGeneralization, Baranwal2023EffectsNetworks, Luan2023WhenDistinguishability},
where they directly control dependence between (\textit{i}) topology and class (i.e., class-homophily \hc) and (\textit{ii}) features and class (i.e., feature distance FD).
It is, however, non-trivial to control dependence between topology and features with the prior CSBMs, while holding the other two dependence (i.e., \hc and FD) constant.
% To our best knowledge, no prior CSBMs directly control dependence between topology and features.

% \clearpage
\section{Measure and Patterns}\label{sec:measure}
In this section, we address the first research question (\textbf{RQ1})
on the measure of \ax (i.e., dependence between graph topology and features).

%%%%%%%%%%%%%%%%%%%%%%%%%%%%%%%%%%%%%%%%%%%%%%%%%%%%%%%%%%%%%%%%%%%%%%%%%%%%%%%%%%%%%%%%%%%%%%%%%%%%%%%%%%%%%%%%%%%%%%%%%

\subsection{Design Goals and Intuition} \label{sec:measure_goal}
We target two central goals in designing \ax measure \hdot.
First, the measure \hdot should distinguish positive, negative, and no dependence.
Second, if a third variable is available (i.e., node class $Y$), the measure \hdot should control its potential effects on \ax.

Their intuition is as follows.
Let us first illustrate the examples of positive and negative dependence.
Consider a \textit{friendship network} (FN) of adults in the ages of 20-50s.
People with similar \textit{political inclinations} (PI) tend to become friends, so FN positively depends on PI.
People with \textit{care need} (CN) tend to seek friends without CN to receive their support, making FN negatively dependent on CN. 

Now, we further illustrate no dependence and class-control.
Consider the number of \textit{wrinkles} (WK).
People generally do not make friends based on their WK, but WK still positively depends on FN.
This is because the \textit{age group} (AG; i.e., node class) confounds the dependence between WK and FN.
Controlling for the effect of AG on WK, the dependence between WK and FN should no longer exist.

%%%%%%%%%%%%%%%%%%%%%%%%%%%%%%%%%%%%%%%%%%%%%%%%%%%%%%%%%%%%%%%%%%%%%%%%%%%%%%%%%%%%%%%%%%%%%%%%%%%%%%%%%%%%%%%%%%%%%%%%%

\subsection{Measure Design}
To achieve the design goals, we propose \textbf{\underline{C}}lass-controlled \textbf{\underline{F}}eature \textbf{\underline{H}}omophily (\textbf{CFH}) measure \hdot.

%%% Define debiased feature
\smallsection{Class-controlled features}.
Assuming a linear relation between classes and features, we mitigate their association to define \textbf{\textit{class-controlled features}} $X \vert Y$.
% \vspace{-2mm}
\begin{equation}\label{eq.class_control}
    {X_{i} \vert Y} = X_{i} - \left(\frac{1}{\vert C^+_{Y_{i}} \vert}\sum_{v_{j}\in C^+_{Y_{i}}} X_{j}\right).
    % \vspace{-1mm}
\end{equation}
For intuition, consider the former example. 
Let AG be the class and WK be the feature.
Since AG affects WK, WK distributions are different for each AG.
However, the AG-controlled WK, obtained with Eq.~\eqref{eq.class_control}, would have similar distributions across AGs.
Namely, Eq.~\eqref{eq.class_control} mitigates the association between AG and WK.
Eq.~\eqref{eq.class_control} is analogous to the variable control method of \textit{partial} and \textit{part correlation}~\cite{Stevens2009AppliedEd}.
We discuss their connection in Appendix~\ref{app:class-control}.

%%% Define measure

\smallsection{Measuring CFH}. 
We measure CFH \hdot with class-controlled features $X \vert Y$.
Let us define a distance function.

\begin{enumerate}[start=1,label={\bfseries D\arabic*)}]
    %\item Distance function $\mathbf{d}(\cdot)$: 
    \item {Distance function $\mathbf{d}: (V \times 2^{V}) \mapsto \mathbb{R}_{\geq 0}$:} 
    \begin{equation}\label{eq.distance_func}
        \mathbf{d}(v_{i}, V') \coloneqq \frac{1}{\vert V'\vert}\sum_{v_{j} \in V'}\lVert (X_{i}~\vert~Y) - (X_{j}~\vert~Y) \rVert_{2}.
    \end{equation}
\end{enumerate}
Recall that $V'_i = V$\textbackslash $\set{v_i}$.
Given the distance function $\mathbf{d}(\cdot)$, 
we define \textit{\textbf{homophily baseline}} $b(v_{i}) = \mathbf{d}(v_{i}, V_{i}')$.
Homophily baseline $b(v_{i})$ can be interpreted as node $v_i$'s expected (i.e., average) distance to random nodes or to its neighbors $N_i$ when no \ax is assumed.

Based on the distance functions, we define node pair-level, node-level, and graph-level CFHs as follows:
\begin{enumerate}[start=1,label={\bfseries H\arabic*)}]
    \item Node pair-level CFH $\mathbf{h}^{(p)}_{ij}$: 
    \begin{equation}\label{eq.edge_homophily}
        \mathbf{h}^{(p)}_{ij} = \mathbf{h}((v_{i},v_{j}) \vert X,Y,E) \coloneqq b(v_{i}) - \mathbf{d}(v_{i}, \{v_{j}\}).
    \end{equation}

    \item Node-level CFH $\mathbf{h}^{(v)}_{i}$: 
    \begin{equation}\label{eq.node_homophily}
        \mathbf{h}^{(v)}_{i} = \mathbf{h}(v_{i}~\vert~X,Y,E) \coloneqq \frac{1}{\vert N_i \vert} \sum_{v_{j} \in N_i} \mathbf{h}^{(p)}_{ij}.
    \end{equation}
    \vspace{-5mm}

    \item Graph-level CFH $\mathbf{h}^{(G)}$: 
    \begin{equation}\label{eq.graph_homophily}
        \mathbf{h}^{(G)} = \mathbf{h}(G~\vert~X,Y,E) \coloneqq \frac{1}{\vert V\vert} \sum_{v_{j} \in V} \mathbf{h}^{(v)}_{j}.
    \end{equation}
    \vspace{-5mm}
\end{enumerate}
Simply put, CFH $\mathbf{h}(\cdot)$ measures neighbor distance relative to homophily baseline $b(\cdot)$, 
and {it meets the two design goals discussed in Sec.~\ref{sec:measure_goal}.
% First, by utilizing distance function $\mathbf{d}(\cdot)$, 
% $\mathbf{h}(\cdot)$ is defined for both continuous and discrete features with any number of dimensions. 
With the homophily baseline $b(\cdot)$, $\mathbf{h}(\cdot)$ distinguishes homophily (positive dependence), heterophily (negative dependence), and impartiality (no dependence).
At the same time, 
by measuring the distance with class-controlled features $X \vert Y$ (see Eq.~\eqref{eq.distance_func}), 
CFH $\mathbf{h}(\cdot)$ mitigates potential confounding by node class.}

Finally, we normalize $\mathbf{h}(\cdot)$ for good mathematical properties (Lemma~\ref{lem:measure_bound}-~\ref{lem:measure_mono}), which allows for its intuitive interpretation (discussed in Sec.~\ref{sec:measure_interpretation}).
\footnote{For completeness, if $b(\cdot) = 0$, we let $\mathbf{\Tilde{h}}^{(v)}_{i}, \mathbf{\Tilde{h}}^{(G)}=0$.}
\begin{enumerate}[start=1,label={\bfseries N\arabic*)}]
    \item Node-level normalization: 
    \begin{equation}\label{eq.node_norm}
        \mathbf{\Tilde{h}}^{(v)}_{i} = 
        \frac{\mathbf{h}^{(v)}_{i}}
        {\texttt{max}(b(v_{i}),\mathbf{d}(v_{i}, N_i))}.
    \end{equation}
    \vspace{-5mm}

    \item Graph-level normalization:
    \begin{equation}\label{eq.graph_norm}
        \mathbf{\Tilde{h}}^{(G)} = 
        \frac{\mathbf{h}^{(G)}}
        {\frac{1}{\abs{V}}~
        {\texttt{max}(\sum_{v_i \in V} b(v_{i}) , \sum_{v_i \in V} \mathbf{d}(v_{i}, N_i))}}.
    \end{equation} 
    \vspace{-2mm}
\end{enumerate}

\begin{lemma}({Boundedness})\label{lem:measure_bound}
    \hg, \hn $\in [-1,1]$, 
    and the bound is tight, i.e., $\inf_{G}\tilde{\mathbf{h}}^{(G)} = -1$ and $\sup_{G}\hg = 1$.
\end{lemma}
\vspace{1mm}

\begin{lemma}({Scale-Invariance})\label{lem:measure_scale}
    $\mathbf{\Tilde{h}}(v_i \vert \ X, \cdot) = \mathbf{\Tilde{h}}(v_i \vert \ cX, \cdot)$ and 
    $\mathbf{\Tilde{h}}(G \vert \ X, \cdot) = \mathbf{\Tilde{h}}(G \vert \ cX, \cdot)$,
    $\forall c \in \mathbb{R}$\textbackslash $\set{0}$.
\end{lemma}
\vspace{1mm}

{\begin{lemma}(Monotonicity)\label{lem:measure_mono}
    Fix features of $v_{j} \in V \setminus N_{i}$, 
    \hn is a monotonically decreasing function of  $\mathbf{d}(v_i, N_i)$. 
\end{lemma}}

All the proofs are in Appendix~\ref{app:proofs}.

%%%%%%%%%%%%%%%%%%%%%%%%%%%%%%%%%%%%%%%%%%%%%%%%%%%%%%%%%%%%%%%%%%%%%%%%%%%%%%%%%%%%%%%%%%%%%%%%%%%%%%%%%%%%%%%%%%%%%%%%%
\begin{figure}[t] 
    \centering
    % \vspace{2mm}
    
    \includegraphics[width=\linewidth]{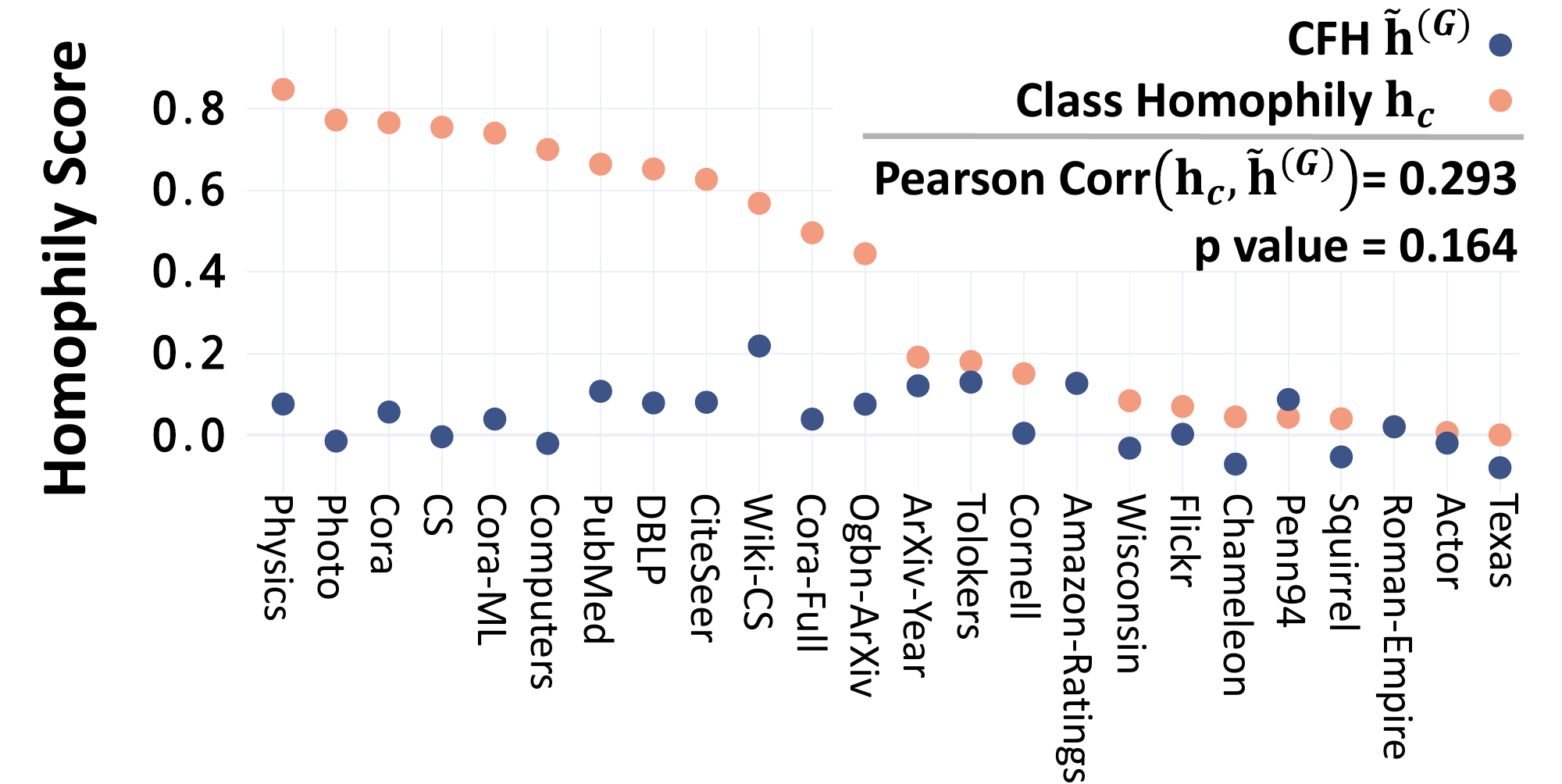}

    \caption{\label{fig:data_overall_statistics}
    \bolden{Benchmark Graph Statistics. }
     Graph-level CFH scores $\mathbf{\Tilde{h}}^{(G)}$ 
     (\textit{i}) are generally positive and small, with 
     (\textit{ii}) low correlation to class-homophily $\mathbf{h}_c$.}
    \vspace{2mm}
\end{figure}

\subsection{Measure Interpretation}\label{sec:measure_interpretation}

We first focus on the node-level interpretation of CFH \hdot. 
Recall that $\mathbf{d}(v_{i}, N_i)$ and $b(v_i)$ respectively represent node $v_i$'s distance to neighbors and random nodes.

\smallsection{Sign}.
Node-level CFH \hn$>0$ means that the node $v_i$ is closer to its neighbors than to random nodes and, thus, \textit{homophilic}.
\hn$<0$ means that the node $v_i$ is farther to its neighbors than to random nodes and, thus, \textit{heterophilic}.

\smallsection{Zero}.
Node-level CFH $\mathbf{\Tilde{h}}^{(v)}_i = 0$ indicates that the node $v_i$ has the same distance to its neighbors and to random nodes, 
suggesting {\textit{impartiality}} or \textit{no \ax}.
Several different cases entail $\mathbf{\Tilde{h}}^{(v)}_i = 0$ (in expectation).
For example,
(\textit{i}) when the neighbors of $v_i$ are chosen uniformly at random from all the other nodes regardless of their features or
(\textit{ii}) when all the nodes have the same feature, $\mathbb{E}[\mathbf{\Tilde{h}}^{(v)}_i] = 0$.

\smallsection{Magnitude}.
Increasing a node $v_i$'s distance to its neighbors reduces node-level CFH ${\mathbf{\Tilde{h}}^{(v)}_i}$ (Lemma~\ref{lem:measure_mono}). 
We rephrase Eq.~\eqref{eq.node_norm} as follows:
\vspace{1mm}
\\
$
{\mathbf{\Tilde{h}}^{(v)}_i} = 
\begin{cases}
    1 - \frac{\mathbf{d}(v_{i}, N_i)}{b(v_i)}, &\text{~if~} \mathbf{d}(v_{i}, N_i) \leq b(v_i), \\
    \frac{b(v_i)}{\mathbf{d}(v_{i}, N_i)} - 1, &\text{~if~} \mathbf{d}(v_{i}, N_i) > b(v_i).
\end{cases}
\vspace{1mm}
$ \\
Intuitively, 
a node $v_i$ is $\frac{ \vert {\mathbf{\Tilde{h}}^{(v)}_i} \vert }{1- \vert {\mathbf{\Tilde{h}}^{(v)}_i} \vert}$ 
times closer (or farther) to its neighbors than to random nodes,
if $\mathbf{\Tilde{h}}_{i}^{(v)} > 0$ (or $< 0$).

\smallsection{Summary}.
In summary, for each node $v_i$, its distance to random nodes (i.e., $b(v_i)$) serves as an anchor to determine the sign and magnitude of its CFH \hn, making it readily interpretable and comparable across different graphs.

\smallsection{Graph-level interpretation}.
A graph-level CFH \hg is an aggregation of node-level CFH ${\mathbf{h}^{(v)}_i}$'s.
Details are in Appendix~\ref{app:measure}.
% The interpretation of \hn, therefore, readily extends to \hg.
% More details on the measure \hdot and its interpretation can be found in Appendix~\ref{app:measure}.

%%%%%%%%%%%%%%%%%%%%%%%%%%%%%%%%%%%%%%%%%%%%%%%%%%%%%%%%%%%%%%%%%%%%%%%%%%%%%%%%%%%%%%%%%%%%%%%%%%%%%%%%%%%%%%%%%%%%%%%%%

\subsection{Patterns in Benchmark Datasets}

Here, we analyze node classification benchmark datasets using CFH \hdot.
First, we measure graph-level CFH \hg in 24 datasets (Fig.~\ref{fig:data_overall_statistics}).
Most of the graphs (23 out of 24) have \hg scores below 0.13, 
and 16 graphs have positive \hg scores.
Their mean $\vert \mathbf{\Tilde{h}}^{(G)} \vert$ is 0.06.
Recall that the full reachable range of \hg is $[-1,1]$ (Lemma~\ref{lem:measure_bound}).

\begin{customobs}{1}\label{obs:absolte_homophily}
    The real-world graph benchmarks tend to show small, positive CFH scores.%
    \footnote{By small CFH scores, we mean the distances from nodes to their neighbors are \textit{highly} close to their homophily baselines,
    numerically evidenced by the mean $\vert \mathbf{\Tilde{h}}^{(G)} \vert$ of 0.06.
    }
\end{customobs}

We further analyze the relation between CFH \hg and class-homophily $\mathbf{h}_c$ (Fig.~\ref{fig:data_overall_statistics}).
Surprisingly, their correlation is low (Pearson's $r$ = 0.293, Kendall's $\tau$ = 0.196) and not statistically significant ($p$ value = 0.164 and 0.191, respectively).

\begin{customobs}{2}\label{obs:relative_homophily}
    In the real-world graph benchmarks, CFH and class-homophily have a small, positive correlation.
\end{customobs}

From Observations~\ref{obs:absolte_homophily}-\ref{obs:relative_homophily}, 
we conclude that CFH \hdot and class-homophily \hc show distinct patterns in the real-world benchmark graphs.
We, thus, argue that investigating the impact of CFH \hdot for GNNs has a unique significance.

Lastly, we examine how the feature shuffle (recall Fig.~\ref{fig:strange_phe}(a)) affects CFH.
For graph-level CFH \hg, increasing the shuffled node ratio reduces its magnitude $\vert \mathbf{\Tilde{h}}^{(G)} \vert$ (Fig.~\ref{fig:feat_homophiy_over_shuffle}).
Also, the distribution of node-level CFH scores (\hn's) tends to center around 0 after the feature shuffle. 
We find similar results in 
19 out of 24 datasets, while the remaining five do not fully obey the pattern.

\begin{customobs}{3}\label{obs:shuffle_homophily}
    CFH scores tend to approach zero after shuffling the features of nodes from the same class.
\end{customobs}

{In later sections, Observation~\ref{obs:shuffle_homophily} serves to bridge a GNN theory and GNN performance in the real-world graphs, explicating the intriguing phenomenon (Fig.~\ref{fig:strange_phe}).}

To further corroborate each Observation, we provide more in-depth analysis in Appendix~\ref{app:patterns}, together with dataset description and statistics.

\begin{figure}[t] 
    \centering
    % \vspace{2mm}
    
    \includegraphics[width=\linewidth]{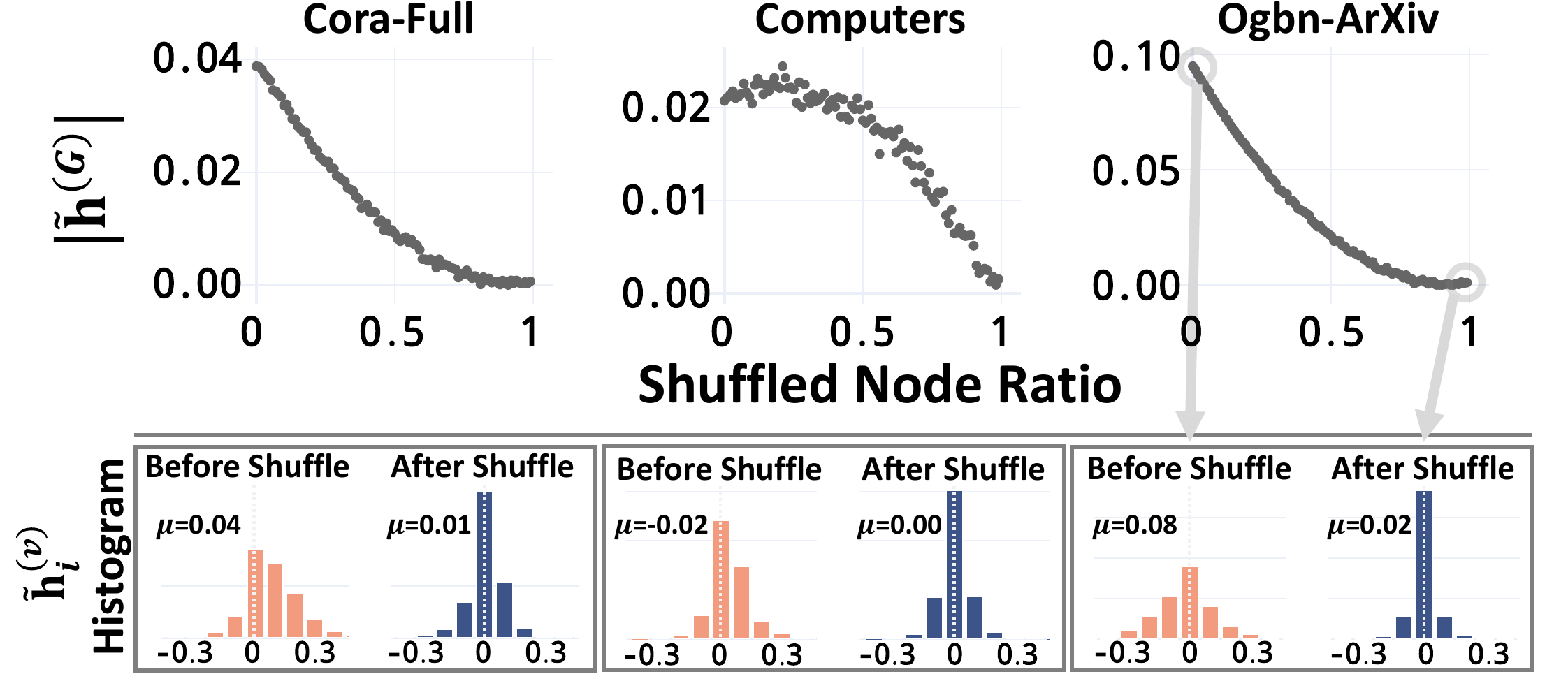}

    \caption{\label{fig:feat_homophiy_over_shuffle}
    \bolden{The Effect of Feature Shuffle on CFH.}
    Both graph- and node-level CFH scores, \hg and \hn, tend to approach zero over the feature shuffles.
    }
    \vspace{2mm}
\end{figure}

\section{Graph Model and GNN Theory}\label{sec:graph_model}

We first address the second research question (\textbf{RQ2}) theoretically {with a random graph model}.

\textbf{RQ2.1} [\textit{Graph Model and Theory}]. How does \ax affect graph convolution in a random graph model?

\subsection{Graph Model: \gmodel}

\smallsection{\gmodel overview}.
{To control class-controlled feature homophily (CFH) \hdot with a graph model, we propose \gmodel. 
Compared to the previous CSBMs, \gmodel is equipped with a new \ax strength parameter $\tau$.
We provide a verbal description here, and its formal mathematical expression can be found in Appendix~\ref{app:csbmx}.}

{\gmodel uses $(n, \mu_0, \mu_1, \Sigma_0, \Sigma_1, d^+, d^-, \tau)$ as its parameters.}
It initializes $n$ (assume even) number of nodes and equally divides them into two classes.
For each node $v_i$, based on its class $Y_i$, \gmodel samples its feature $X_i$ from a Gaussian distribution with a mean vector $\mu_{Y_i}$ and a covariance matrix $\Sigma_{Y_i}$ (i.e., $X_i = \mathcal{N}(\mu_{Y_i}, \Sigma_{Y_i})$).

Then, directed edges are sampled based on the node features $X$ and classes $Y$, 
where parameter $\tau$ influences the sampling weights $\mathbb{P}_{ij}$'s.
Specifically, CSBM-X first computes neighbor sampling weights $\mathbb{P}_{ij}$'s as follows:
\begin{align*}
    \mathbb{P}_{ij} &=
    \begin{dcases}
        \frac{\operatorname{exp}(\tau \mathbf{h}^{(p)}_{ij})}{\sum_{t \in C^+_{Y_i} \setminus \set{v_i}} \operatorname{exp}(\tau \mathbf{h}^{(p)}_{it})} &\quad\text{, if } Y_i = Y_j, \\
        \frac{\operatorname{exp}(\tau \mathbf{h}^{(p)}_{ij})}{\sum_{t \in C^-_{Y_i}} \operatorname{exp}(\tau \mathbf{h}^{(p)}_{it})} &\quad\text{, if } Y_i \neq Y_j.
    \end{dcases}
\end{align*}
A positive (or negative) $\tau$ exaggerates neighbor sampling weight $\mathbb{P}_{ij}$'s among the node pairs with higher (or lower) pair-level CFH $\mathbf{h}^{(p)}_{ij}$ (Eq.~\eqref{eq.edge_homophily}).
By the neighbor sampling weights $\mathbb{P}_{ij}$'s, for each node $v_i$, CSBM-X samples $d^+$ same-class (and $d^-$ {different}-class) neighbors from the same-class node set $C^+_{Y_i} \setminus \set{v_i}$ (and different-class node set $C^-_{Y_i}$) without replacement.
With its sampled nodes, neighbors, features, and classes, 
CSBM-X returns its graph $\mathcal{G} \coloneqq \mathcal{G}(n, \mu_0, \mu_1, \Sigma_0, \Sigma_1, d^+, d^-, \tau) = (V, E, X, Y)$.

\smallsection{\gmodel properties}. 
The key innovation of \gmodel involves satisfying good properties in controlling dependence among classes $Y$, features $X$, and graph topology $A$.
First, the parameters ($\mu_0, \mu_1, \Sigma_0, \Sigma_1$) control feature distance FD (Eq.~\eqref{eq.fd}; $\mathrm{X}\text{-}\mathrm{Y}$ dependence).
Second, the parameters ($d^+, d^-$) control \hc (Eq.~\eqref{eq.class_homophily}; $\mathrm{A}\text{-}\mathrm{Y}$ dependence).
Last, the parameter $\tau$ controls CFH \hdot (Eqs.~\eqref{eq.node_norm}-\eqref{eq.graph_norm}; \ax).
% \footnote{For simplicity of proofs, we focus on unnormalized CFH measures $\mathbf{h}(\cdot)$ (Eqs.~\ref{eq.edge_homophily}-\ref{eq.graph_homophily}) for theoretical investigation of CSBM-X.}

Existing CSBMs can also control $\mathrm{X}\text{-}\mathrm{Y}$ and $\mathrm{A}\text{-}\mathrm{Y}$ dependence~\cite{Deshpande2018ContextualModels, Abu-El-Haija2019Mixhop:Mixing, Chien2021AdaptiveNetwork, Palowitch2022Graphworld:Gnns, Baranwal2023EffectsNetworks, Luan2023WhenDistinguishability,wang2024understanding}.
However, the proposed \gmodel further controls \ax (CFH \hdot), 
satisfying two additional good properties.
\footnote{In theoretical statements, we focus on unnormalized CFH $\mathbf{h}(\cdot)$ for simplicity.}
% while holding feature distance FD and class-homophily \hc constant.
% Formally, \gmodel satisfies two additional good properties.

\vspace{1mm}
\begin{lemma}[$\tau$ controls CFH $\mathbf{h}(\cdot)$ \textit{precisely}]\label{lem:relation_tau_hg}
    {Given $0 < {{\max}}(d^+, d^-) < \frac{n}{2}$ and}
    fix the other parameters except for $\tau$.
    (i) $\bbE[\mathbf{h}^{(G)}]$ strictly increases as $\tau$ increases.
    (ii) When $\Sigma_0 = \Sigma_1 \neq \mathbf{0}$, $\bbE[\mathbf{h}^{(G)}] = 0$ if and only if $\tau = 0$.
\end{lemma}
% \vspace{0.1mm}
\begin{lemma}[$\tau$ controls CFH $\mathbf{h}(\cdot)$ \textit{only}]\label{lem:tau_indep_hdot}
    Fix the other parameters except for $\tau$,
    the FD and \hc of $\calG$ are constant regardless of the value of $\tau$.
\end{lemma}

The proofs are in Appendix~\ref{app:proofs}.
In concert, the above properties highlight that \gmodel flexibly, yet precisely, controls the dependence among classes $Y$, features $X$, and topology $A$ in the generated graph $\mathcal{G}$'s.

\begin{figure*}[t]
\begin{center}
    % \hspace{10mm}
    
    \includegraphics[width=\textwidth]{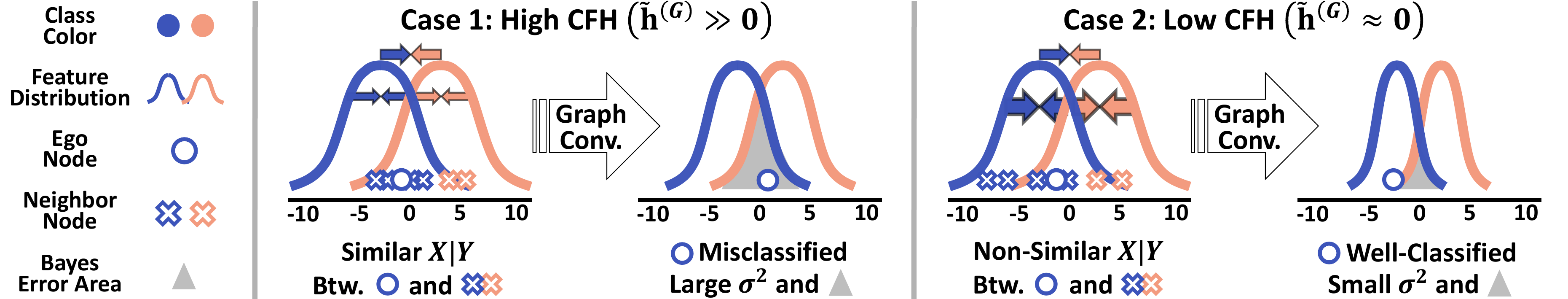}
    % \vspace{1mm}
    \caption{\label{fig:theorem_intuition}
    \bolden{Visual Intuition of Theorem~\ref{thm:main_result}.}
    When CFH is low ($ \mathbf{\Tilde{h}}^{(G)} \approx 0$), 
    the feature distribution of each class shrinks faster (denoted by the arrows) by graph convolution, 
    resulting in a lower Bayes error rate.
    Namely, the power to pull node features towards the feature mean of each class becomes stronger with decreasing $\vert \mathbf{\Tilde{h}}^{(G)} \vert$.
    }
\end{center}
\end{figure*}

\begin{figure*}[t]
\begin{center}
    \includegraphics[width=\textwidth]{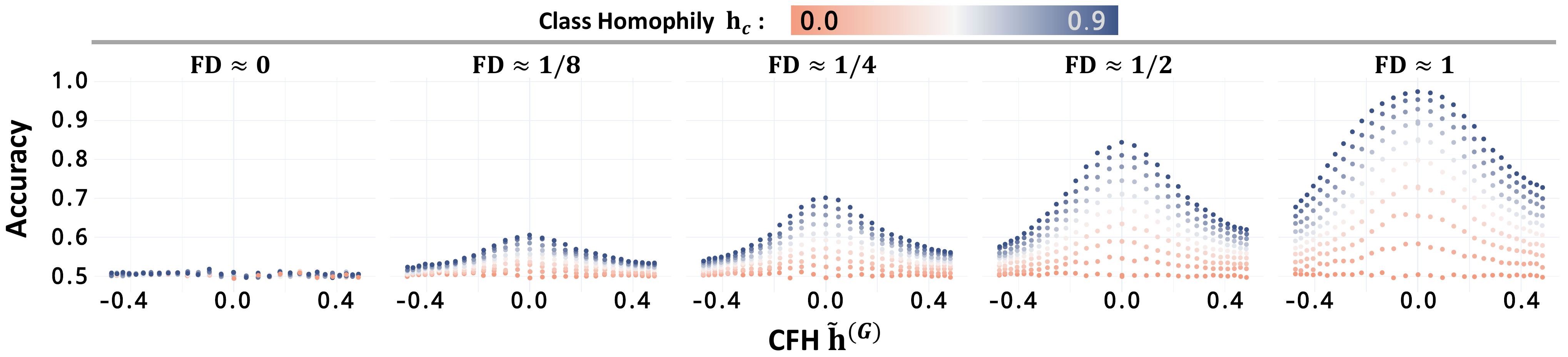}

    \caption{\label{fig:csbm-x-main-exp}
    \bolden{The Simplified GNN Performance in \gmodel Graphs.}
    Consistent with Theorem~\ref{thm:main_result},
    for given feature distance $\text{FD} > 0$ and class homophily $\mathbf{h}_c >0$,
    the simplified GNN performance increases as graph-level CFH $\mathbf{\Tilde{h}}^{(G)} \rightarrow 0~(\text{i.e.,~} \tau \rightarrow 0)$.
    }
\end{center}
\end{figure*}

\subsection{Graph Convolution in CSBM-X Graphs: Theory}
In this section, we theoretically analyze the relationship between CFH \hdot and graph convolution.

\smallsection{Analysis setting}.
For simplicity, we assume that the features are
(\textit{i}) 1-dimensional ($\mu_0, \mu_1, \Sigma_0, \Sigma_1 \in \mathbb{R}$) and
(\textit{ii}) symmetric with identical variances ($\mu_0 = -\mu_1 < 0$ and $\Sigma_0 = \Sigma_1 = 1$).
We focus on an asymptotic setting with 
% (\textit{iii}) fixed $p^- \neq p^+ \in (0, \frac{1}{2})$ with
% (\textit{iv}) the number of nodes $n \rightarrow \infty$ ($d^+ = np^+$ and $d^- = np^-$).
(\textit{iii}) {the number of nodes} $n \rightarrow \infty$. 
(\textit{iv}) {The same- and different-class degree parameters respectively are} $d^+ = np^+$ and $d^- = np^-$, with fixed $p^- \neq p^+ \in (0, \frac{1}{2})$.

Following some prior works on GNN theory~\cite{Wu2023ANetworks, Luan2023WhenDistinguishability},
we define graph convolution as $D^{-1}AX$, a convolution of feature matrix $X$ on an adjacency matrix $A$ left-normalized by a (diagonal) degree matrix $D$.

% \footnotetext{In \gmodel, $(v_i, v_j)$ denotes a directed edge from $v_i$ to $v_j$.}

Given the setting, 
after a graph convolution,
the expected feature means of the two classes are constant and symmetric regardless of parameter $\tau$.
Specifically, the expected means are
$\frac{d^- - d^+}{d^+ + d^-} \mu_1$ for class-$0$ and 
$\frac{d^+ - d^-}{d^+ + d^-} \mu_1$ for class-$1$.
Thus, we consider a classifier $\mathcal{F}$ predicting node classes as follows:
\[
\mathcal{F}(x) = 
\begin{cases}
1 &\quad\text{if } x \geq 0, \\
0 &\quad\text{otherwise}.
\end{cases}
\]
\hspace{-1mm}
\smallsection{Theoretical analysis}.
We analyze how parameter $\tau$, controlling for CFH \hdot, affects the Bayes error rate of the classifier $\mathcal{F}$, given the \textit{convoluted node features} (i.e., features after convolution $D^{-1}AX$).
Formally, we denote the expected Bayes error rate of $\mathcal{F}$ for classifying the two classes in a CSBM-X graph $\mathcal{G}(\cdot, \tau)$ as $\mathcal{B}_{\mathcal{F}}(\mathcal{G}(\cdot, \tau))$.

\vspace{1.5mm}
\begin{theorem} \label{thm:main_result} 
    Fix the other parameters except for $\tau$, 
    after a step of graph convolution,
    $\mathcal{B}_{\mathcal{F}}(\mathcal{G}(\cdot, \tau))$ is     
    minimized at $\tau = 0$ and 
    {strictly} increases as $\abs{\tau}$ increases, i.e.,
    $\argmin_\tau \mathcal{B}_{\mathcal{F}}(\mathcal{G}(\cdot, \tau)) = 0$;
    $\mathcal{B}_\mathcal{F}(\mathcal{G}(\cdot, \tau_0)) < \mathcal{B}_\mathcal{F}(\mathcal{G}(\cdot, \tau_1))$ for any $\tau_0$ and $\tau_1$ such that $\abs{\tau_0} < \abs{\tau_1}$ and $\tau_0 \tau_1 > 0$.
\end{theorem}

\begin{proof}[Proof sketch]
WLOG, assume a node $v_i \in V$ is 
assigned to class-1 (i.e., $Y_i = 1$) and 
has class-controlled feature $x_i$ (i.e., $X_i = \mu_1 + x_i$).
\footnote{Recall that class-controlled feature $x_i = (X_i \vert Y)$ in Eq.~\eqref{eq.class_control}.}

We obtain closed-form formulae of the distributions of
(\textit{i}) edge sampling probabilities and, subsequently, 
(\textit{ii}) neighbor sampling probabilities.
This allows us to calculate the distribution of class-controlled, convoluted node features.
\begin{align*}
    \bbE[x_{N}] 
    &=
    \int_{-\infty}^{\infty} x \Pr[x_{N} = x] dx \\
    &=
    \frac{\tau \left(\text{erfc}\left(\frac{\tau-x_i}{\sqrt{2}}\right)-\exp{(2 \tau x_i)} \text{erfc}\left(\frac{\tau+x_i}{\sqrt{2}}\right)\right)}{\text{erfc}\left(\frac{\tau-x_i}{\sqrt{2}}\right)+\exp{(2 \tau x_i)} \text{erfc}\left(\frac{\tau+x_i}{\sqrt{2}}\right)},
\end{align*}
where 
$x_{N}$ denotes the class-controlled feature of node $v_i$'s neighbor and erfc denotes the complementary error function of the Gaussian error function. 

Then, we obtain a closed-form formula of $\mathcal{B}_{\mathcal{F}}(\mathcal{G}(\cdot, \tau))$.
\[
\mathcal{B}_{\mathcal{F}}(\mathcal{G}(\cdot, \tau))
=
\int_{-\infty}^{\infty} \Pr[\mathcal{F}({x}_i) = 0] \Pr[{x}_i \mid Y_i = 1] d {x}_i 
\]
\begin{equation}\label{eq.closed_form}
    \xrightarrow{n \to \infty} 
    \int_{-\infty}^{\infty} \bm{1}[\bbE[x_{N}] < \frac{d^- - d^+}{d^- + d^+} \mu_1] \varphi[{x}_i] d {x}_i,
\end{equation}
where $\varphi$ is the PDF of the standard Gaussian distribution.
By analyzing the derivatives of the formulae, we show that $\mathcal{B}_{\mathcal{F}}(\mathcal{G}(\cdot, \tau))$ is minimized at $\tau = 0$ and increases as $\abs{\tau}$ increases in both positive and negative directions.
The full proof is in Appendix~\ref{app:proofs}. 
\end{proof}

\smallsection{Summary}.
{For each class $\ell$'s convoluted feature distribution,
degree parameters $(d^+$, $d^-)$ and the feature mean parameters $(\mu_0, \mu_1)$ determine its mean,
while \ax strength parameter $\tau$ determines the distance between each node and the mean 
(or the distribution variance; see Eq.~\eqref{eq.closed_form} and Fig.~\ref{fig:theorem_intuition}).}
Thus, the simplified GNN's Bayes error rate $\mathcal{B}_{\mathcal{F}}(\mathcal{G}(\cdot, \tau))$
decreases as $\abs{\tau}$ decreases, reaching its minimum at $\tau=0$ (Theorem~\ref{thm:main_result}).
\footnote{The graph convolution followed by the defined classifier $\mathcal{F}$ can be considered a simplified GNN~\cite{Wu2019SimplifyingNetworks}.}
With Lemma~\ref{lem:relation_tau_hg}, we conclude that $\mathcal{B}_{\mathcal{F}}(\mathcal{G}(\cdot, \tau))$ is the lowest when CFH $\mathbf{{h}}^{(G)}$ is zero (i.e., no \ax). 

\subsection{Empirical Elaboration on Theory}\label{sec:graph-model-exp}

Here, we empirically validate and elaborate on Theorem~\ref{thm:main_result}.

\smallsection{Experiment setting}.
We generate \gmodel graphs with various parameter configurations.
We fix the number of nodes $n=10000$ and node degree $(d^+ + d^-) = 20$, assuming sparse graph topology.
The features in each class are sampled from a Gaussian distribution.
The generated graphs have a wide range of feature distance FD, class-homophily \hc, and CFH \hdot.

\vspace{-2mm}
\begin{itemize}[leftmargin=*]
    \item {FD}: $\vert \mu_{0} - \mu_{1} \vert \in \set{0, ~1/8, ~1/4, ~1/2, ~1}$; $\Sigma_{0} = \Sigma_{1} = 1$
    \item {\hc}: $d^+ \in \set{10,11,\ldots,18,19}$; $d^- = 20-d^+$, 
    \item {\hdot}: $\tau \in \set{-1.5, -1.4,\ldots, -0.1,0,0.1,\ldots, 1.4, 1.5}$.
\end{itemize}
\vspace{-2mm}

On the CSBM-X graphs $\mathcal{G}$'s, 
we train a simplified GNN $D^{-1}AXW$, 
where $W \in \mathbb{R}$ is a learnable parameter.
We report the test accuracy averaged over 5 trials, each with a train/val/test split of 50/25/25.

\smallsection{Finding 1} (\textit{Effect of \hdot}).
As shown in Fig.~\ref{fig:csbm-x-main-exp}, 
given class-homophily \hc $> 0$ (i.e., degree parameters $d^+ > d^-$) and feature distance FD $> 0$ (i.e., feature mean parameters $\mu_0 \neq \mu_1$), 
the simplified GNN achieves the best accuracy when graph-level CFH \hg$\approx 0$,
and the accuracy decreases as $\vert \hg \vert$ increases, in both positive and negative directions.

\smallsection{Finding 2} (\textit{Interplay among} FD, \hc, \textit{and} \hdot).
{Aligned with our theoretical outcomes (Eq.~\eqref{eq.closed_form}, Fig.~\ref{fig:theorem_intuition}),
class-homophily \hc and feature distance FD moderate the beneficial effect of small \hdot (Fig.~\ref{fig:csbm-x-main-exp}).
For understanding, recall that our theoretical findings roughly indicate that FD and \hc affect the mean, whereas \hdot the variance, of the convoluted feature distribution of each class. Intuitively, consider the two cases below.}

{If feature distance FD and class homophily \hc are moderate-sized, the convoluted feature means of the two classes would be \textit{somewhat distant}.
Small \hdot, then, can significantly benefit GNNs, 
since small variances of the two feature distributions would markedly reduce their overlap.
A very small (or large) FD and \hc, on the contrary, would cause the convoluted feature distributions to be \textit{too close} (or \textit{too distant}).
Then, reducing variances of the convoluted feature distributions may not significantly improve GNNs, mitigating the beneficial effect of \hdot.}

In conclusion, the empirical outcomes are highly consistent with Theorem~\ref{thm:main_result}.
In Appendix~\ref{app:csbmx}, we report consistent results with 
(\textit{i}) two graph convolution layers, 
(\textit{ii}) symmetrically normalized graph convolution,
(\textit{iii}) high-dimensional $X$,
(\textit{iv}) imbalanced variances $\Sigma_0 \neq \Sigma_1$,
(\textit{v}) larger \ax strength $\vert \tau \vert$'s, and
(\textit{vi}) a more complex version of CSBM-X reflecting a power-law degree distribution.

\section{Feature Shuffle in Real-World Graphs}\label{sec:swapping_exp}

In this section, we finalize our investigation of the second research question (\textbf{RQ2}) with {the feature shuffle}.

\textbf{RQ2.2} [\textit{Feature Shuffle}]. In real-world graphs, how does reducing \ax with feature shuffle affect GNNs?

\subsection{Experiment Setting}\label{sec:exp_setting}

For each class, we randomly choose the nodes to be shuffled by a given shuffled node ratio $\in \set{0.00, 0.01,...,1.00}$.
For the chosen same-class nodes, their feature vectors are shuffled randomly.
{To ensure that the train/val/test split is not affected, shuffle is done only within the same split.}

The feature shuffle can reduce \ax (i.e., CFH \hdot; Observation~\ref{obs:shuffle_homophily}) without perturbing X-Y and A-Y dependence, providing a suitable experimental setting to answer \textbf{RQ2}. 
Namely, the feature shuffle serves to generate synthetic versions of the benchmark graphs.

For each shuffled graph, we initialize, train, and evaluate a GNN model.
We report the mean test accuracy over 5 trials, with a train/val/test split of 50/25/25.
For the GNN model, we use GCNII~\cite{Chen2020SimpleNetworks}, GPR-GNN~\cite{Chien2021AdaptiveNetwork}, and AERO-GNN~\cite{Lee2023TowardsRemedies}.
We mainly use GCNII due to its 
(\textit{i}) non-adaptive graph convolution layer and
(\textit{ii}) empirical strengths in both high and low class-homophily \hc graphs.
For more training details, refer to Appendix~\ref{app:train_details}.

\subsection{Connecting Theory and Real-World Graphs}

\smallsection{High class-homophily graphs}.
As shown in Fig.~\ref{fig:shuffle-main-exp},
in all 12 high class-homophily \hc benchmark datasets, GCNII performance improves consistently over increasing shuffled node ratio (the mean increase of 4\%p).
The largest performance gain is 10\%p in the Cora-Full dataset.

\smallsection{Low class-homophily graphs}.
Meanwhile, in low class-homophily \hc benchmark datasets, 
GCNII shows small to no performance improvement in 11/12 datasets (the mean increase of 0.5\%p; Fig.~\ref{fig:shuffle-heterophily}).
As demonstrated with \gmodel (Fig.~\ref{fig:csbm-x-main-exp}), 
low class-homophily \hc reduces the beneficial effect of small \ax in real-world graphs. 
Unexpectedly, in the Roman-Empire dataset, GCNII suffers from a steady performance decline over increasing shuffled node ratio.
The reason may relate to its abnormally large diameter of 6,824.
We provide an in-depth analysis in Appendix~\ref{app:roman_empire}.

\smallsection{The role of FD}.
Increasing feature noise generally decreases feature distance FD.
Specifically, we randomly chose 50\% of all nodes and randomly permuted the feature vectors \textit{irrespective of their classes} to add such noise.
Fig.~\ref{fig:feature-informativeness} shows that, after adding the noise, the slope of performance over the feature shuffles becomes smaller, 
suggesting that significantly increasing the feature noise reduces the beneficial effect of the feature shuffle.
The finding echoes the results from the \gmodel experiment (Fig.~\ref{fig:csbm-x-main-exp}), such that FD moderates the beneficial effect of low \ax.

\smallsection{Other GNN architectures}.
We use other GNN architectures to test if the effect of the feature shuffle relies on GNN architecture choice.
Specifically, we use GPR-GNN, a decoupled GNN with an adaptive graph convolution.
For an attention-based GNN, we use AERO-GNN, capable of stacking deep layers.
In the considered models, the trends are similar to those of GCNII (Fig.~\ref{fig:shuffle-model}), 
suggesting that the other GNN architectures also leverage small \ax to improve their prediction.
We also found consistent results with transformer- and neural-ODE-based GNNs~\cite{gravina2022anti, deng2024polynormer}.

\smallsection{Proximity-based features}.
GNN node classification performance often degrades when using proximity-based information as the \textit{only} node features~\cite{Duong2019OnNetworks, Cui2022OnGraphs, lee2024villain}.
We find consistent results. 
However, we are astonished to find that, after the feature shuffle, a GNN trained with proximity-based features can be as competitive as the one trained with the original features (Fig.~\ref{fig:shuffle-feat}).
The results highlight that, given some feature distance FD and class-homophily \hc, reducing \ax can improve GNN regardless of the feature types.

In summary, the results with the real-world graphs and advanced GNNs are well-aligned with the theoretical results, underscoring the validity and generalizability of the proposed theory.
In Appendix~\ref{app:shuffle}, we further report consistent results with lower train node ratios.

\begin{figure*}[t]
    \centering
    % \hspace{10mm}
    
    \includegraphics[width=\textwidth]{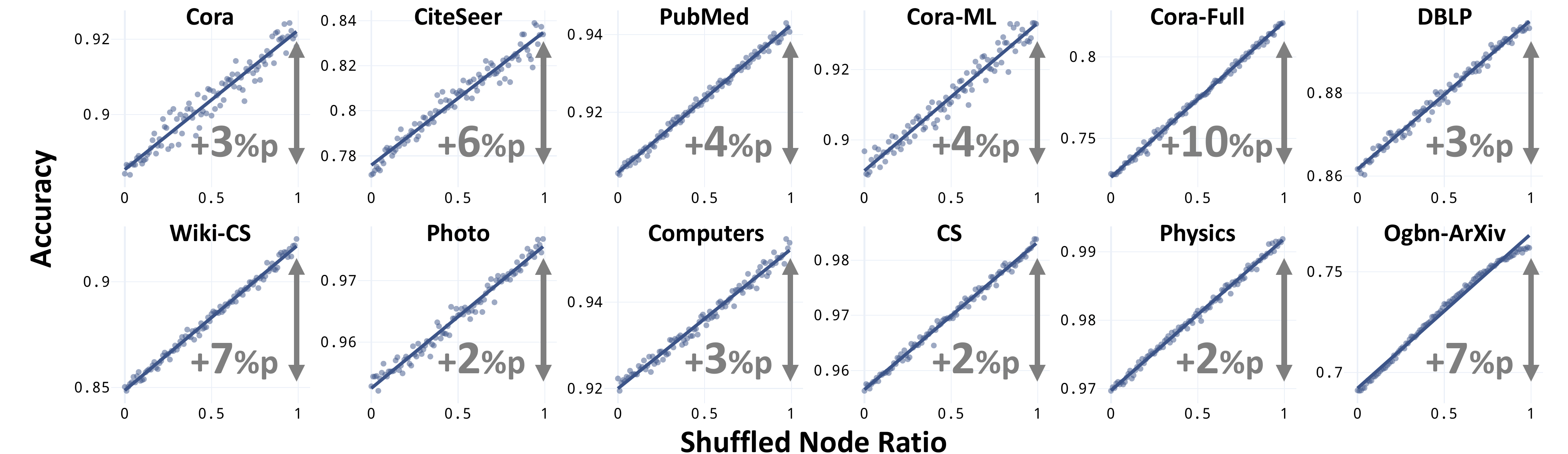}

    \centering
    \caption{\label{fig:shuffle-main-exp}
    \bolden{GNN Performance After the Feature Shuffles: \textit{High Class-Homophily}.}
    Consistent with the findings from \gmodel (Theorem~\ref{thm:main_result}; Fig.~\ref{fig:csbm-x-main-exp}; high \hc case), 
    reducing \ax with the feature shuffle improves GNN performance consistently and significantly in high class-homophily graphs.
    }
    \vspace{2mm}
\end{figure*}
\begin{figure}[t] 
    \centering
    \vspace{2mm}
    
    \includegraphics[width=\linewidth]{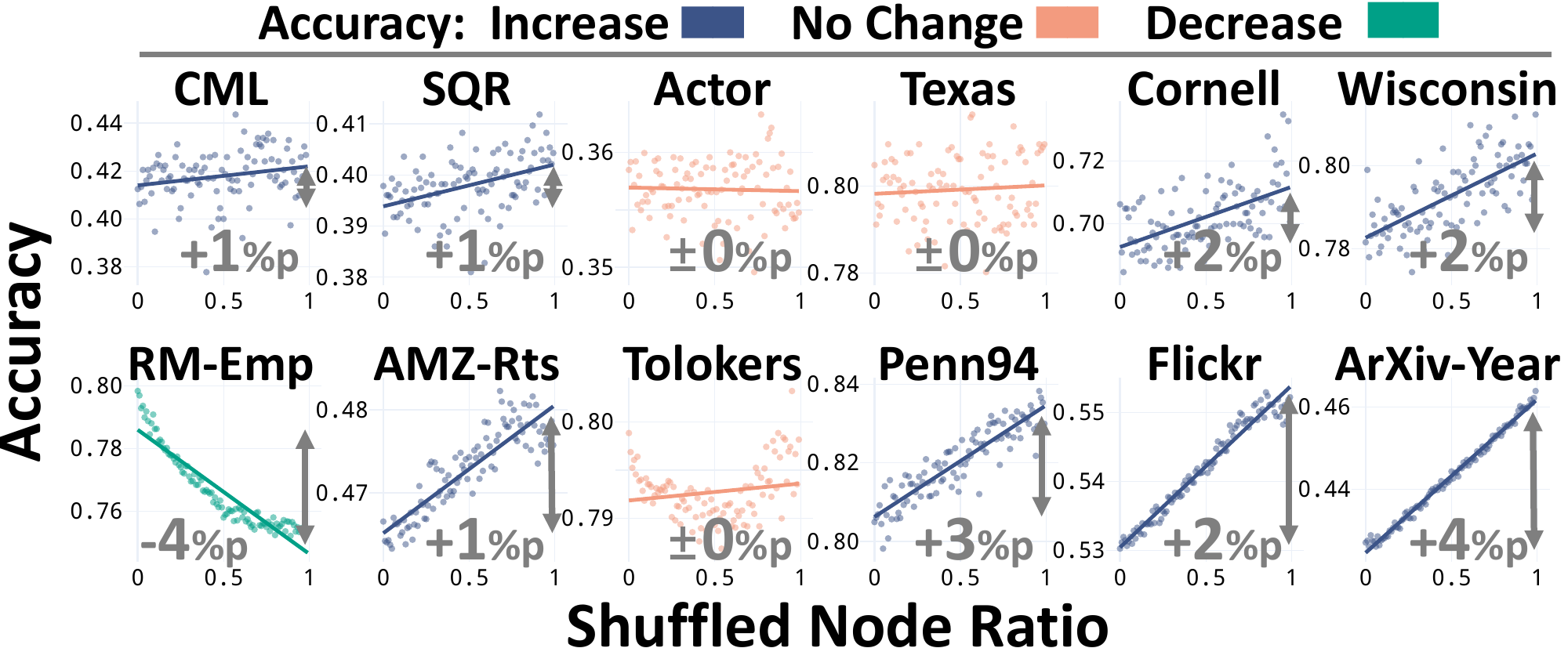}

    \caption{\label{fig:shuffle-heterophily}
    \bolden{GNN Performance After the Feature Shuffles: \textit{Low Class-Homophily}.}
    Consistent with the findings from \gmodel (Theorem~\ref{thm:main_result}; Fig.~\ref{fig:csbm-x-main-exp}; low \hc case), 
    the effect of the feature shuffle is smaller when class-homophily is low.
    \footnotemark
    }
    % \vspace{2mm}
\end{figure}
\footnotetext{CML, SQR, RM-Emp, and AMZ-Rts respectively stand for Chameleon, Squirrel, Roman-Empire, and Amazon-Ratings.}
\begin{figure}[t] 
    \centering
    % \vspace{2mm}
    
    \includegraphics[width=\linewidth]{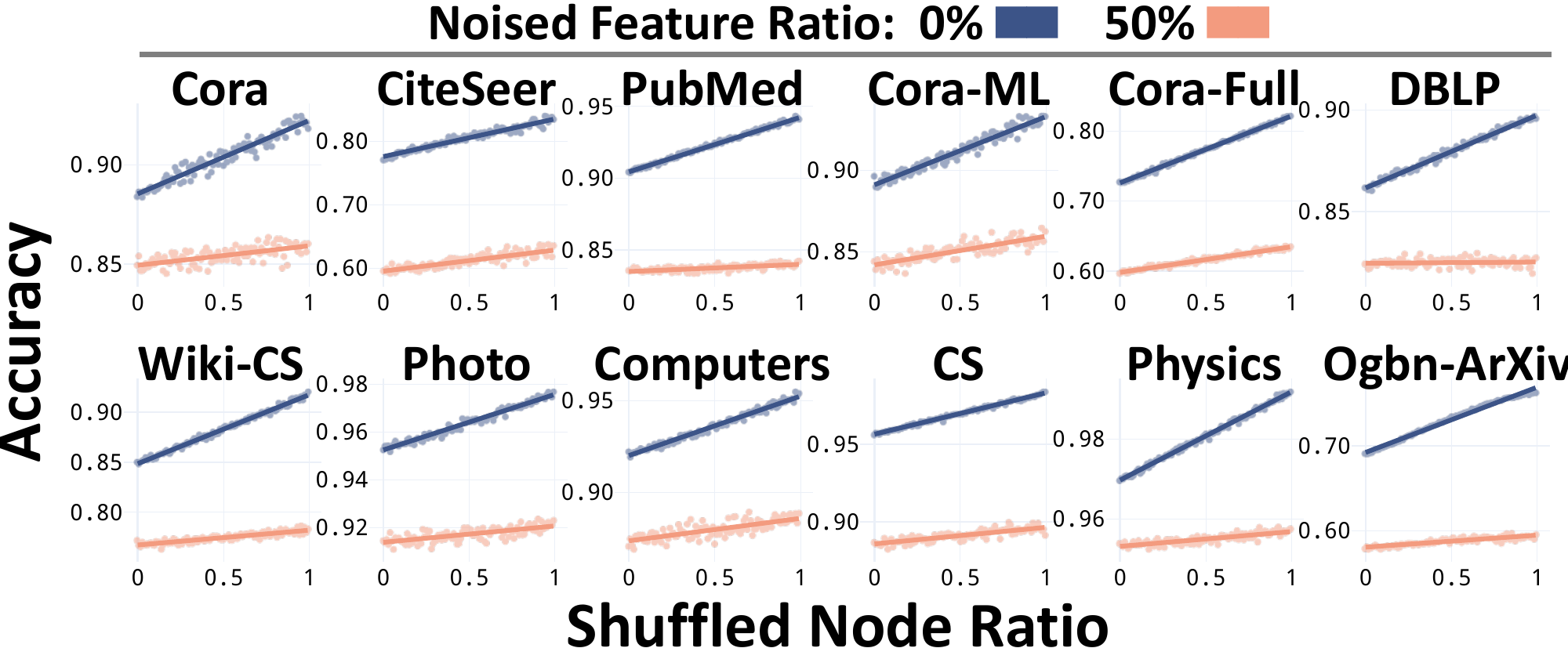}

    \caption{ \label{fig:feature-informativeness}
    \bolden{GNN Performance After the Feature Shuffles: \textit{Noisy Features}.}
    Consistent with the findings from \gmodel (Theorem~\ref{thm:main_result}; Fig.~\ref{fig:csbm-x-main-exp}; low FD case), 
    the effect of the feature shuffle is smaller with noisy features.
    }
    \vspace{3.3mm}
\end{figure}
% \vspace{7mm}
\begin{figure}[t] 
    \centering
    \vspace{2mm}
    
    \includegraphics[width=\linewidth]{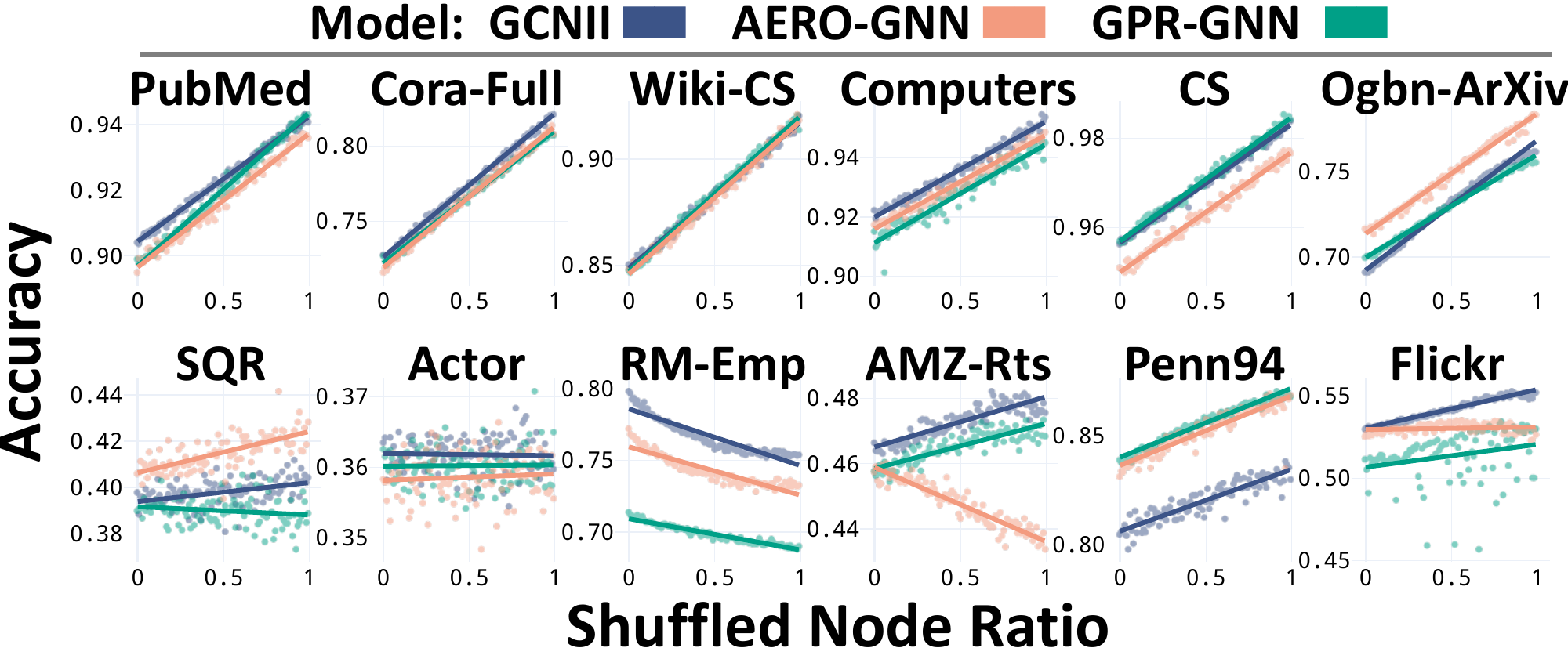}

    \caption{\label{fig:shuffle-model}
    \bolden{GNN Performance After the Feature Shuffles: \textit{Different Models}.}
    The adaptive convolution- and attention-based GNNs (GPR-GNN and AERO-GNN, respectively)
    generally show similar trends with GCNII.
    }
    % \vspace{2mm}
\end{figure}

\begin{figure}[t] 
    \centering
    % \vspace{2mm}
    
    % \includegraphics[scale=0.175]{results/raw_data/shuffle-feat.PNG}
    \includegraphics[width=\linewidth]{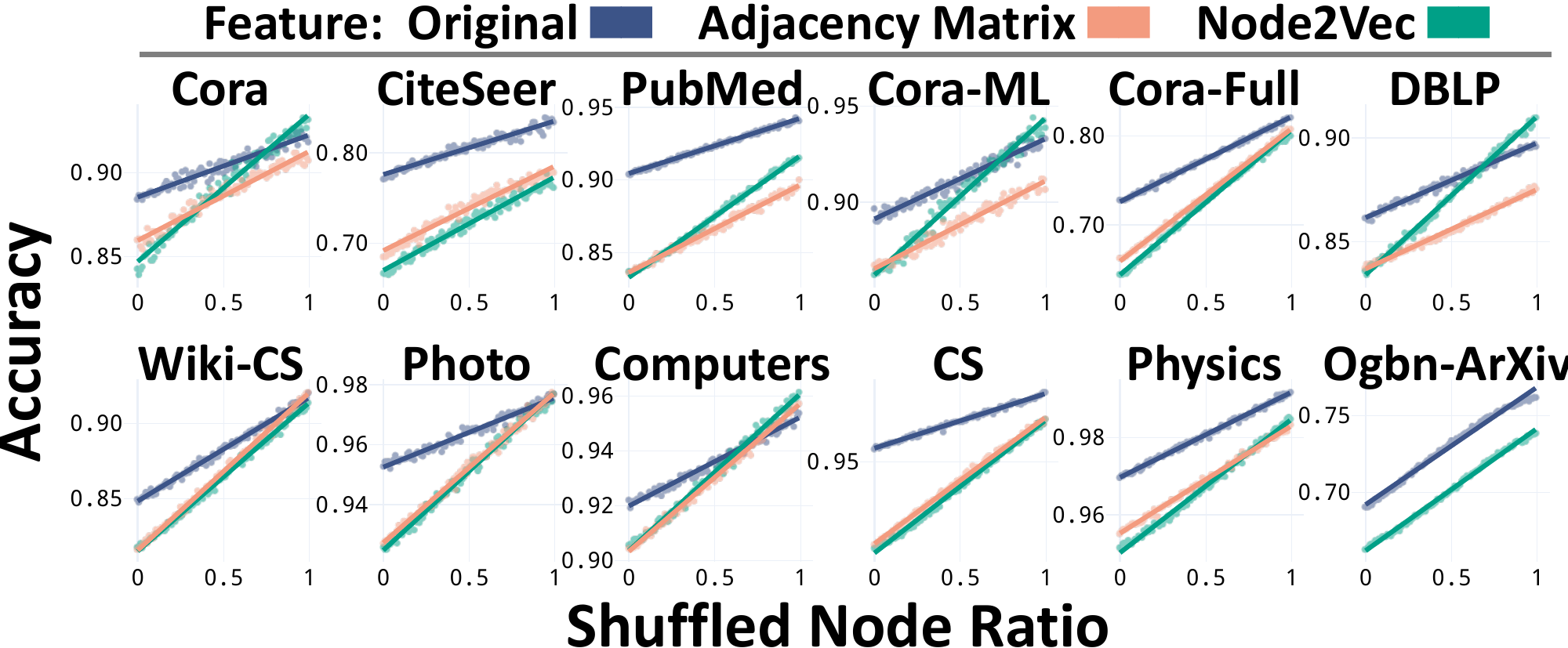}

    \caption{\label{fig:shuffle-feat}
    \bolden{GNN Performance After the Feature Shuffles: \textit{Proximity-Based Features}.}
    Regardless of the feature types, the feature shuffle improves GNN node classification performance.
    \footnotemark
    }    
    \vspace{3.5mm}
\end{figure}
\footnotetext{Out-of-memory occurs when using adjacency matrix as the node features for Ogbn-ArXiv.}

\section{Discussion}\label{sec:discussion}
In this work, we analyze the impact of \ax 
(i.e., dependence between graph topology and node features) on GNNs
with (\textit{i}) a class-controlled feature homophily (CFH) measure \hdot, 
(\textit{ii}) a random graph model \gmodel, and
(\textit{iii}) the feature shuffle.
In both \gmodel and the real-world graphs, 
we demonstrate that \ax, measured by CFH, significantly influences GNN performance.

\smallsection{GNN theory: the prior literature}.
The early studies found some failure cases of GNNs.
\citet{NT2019RevisitingFilters} claimed that a graph convolution layer is simply a low-pass filter for node features.
Under noisy features and non-linear feature spaces, they showed that GNN-based node classification may readily become ineffective.
\citet{Oono2020GraphClassification} further showed that over-smoothing of node features in GNNs may inevitably occur at infinite model depth.

The following works analyzed how GNNs behave at shallower model depths, 
demonstrating that the effect of graph convolution depends on feature informativeness, class-homophily, and node degree.
\citet{Baranwal2021GraphGeneralization} focused on how they let GNNs obtain more linearly separable features for each class. 
\citet{Wei2022UnderstandingPerspective} studied how the factors interact with GNNs' non-linearity,
and \citet{Wu2023ANetworks} investigated their role in triggering over-smoothing.

Aligned with the theory, low class-homophily (often just called heterophily) has received significant attention as GNNs' `nightmare.'
A stream of empirical findings continued to show that GNN performance drops significantly in low class-homophily benchmark datasets~\cite{Pei2020Geom-gcn:Networks, Zhu2020BeyondDesigns, Zhu2021GraphHeterophily, Chien2021AdaptiveNetwork}, and some works investigated the relationship between low class-homophily and over-smoothing~\cite{Bodnar2022NeuralGnns, Yan2022TwoNetworks}.

However, studies began to discover that low class-homophily, per se, does not deteriorate GNN performance.
\citet{Ma2022IsNetworks} and \citet{Platonov2023CharacterizingBeyond} demonstrated that as long as the class distribution is informative w.r.t. node class, GNNs can effectively perform node classification even with low class-homophily.

Recently, studies have delved into how mesoscopic patterns of class-homophily affect GNNs.
\citet{Luan2023WhenDistinguishability} (roughly) argued that, for GNNs to well-classify a node class, its `intra-class distance' should be smaller than the `inter-class distance' after graph convolution. 
That is, low class-homophily may trigger the `inter-class distance' to be smaller to degrade GNN performance.
\citet{Mao2023DemystifyingAll} investigated mixed patterns of class-homophily and heterophily,
showing that GNNs better classify the nodes with the majority pattern in the mixture. 
Lastly, \citet{wang2024understanding} investigated an array of low class-homophily patterns and showed that there exist good, mixed, and bad patterns for GNNs to learn from.

\smallsection{GNN theory: the present work}.
Not to mention that the role of \ax (i.e., CFH \hdot) has not been adequately addressed by the prior literature,
the present work can also be interpreted as an extension of the works on homophily-GNN connection to continuous feature domain.
Intuitively, a large homophily slows feature mixing by graph convolution, and a small homophily accelerates it. 
From such a perspective, our conclusion that CFH should ideally be small, while class-homophily be large, is an intuitive outcome.
To better classify node classes, the mixing between classes should occur at a slow rate, whereas the mixing within each class should occur faster.

To conclude, we argue that CFH mediates the effect of graph convolution by moderating the force to pull each node feature toward the feature mean of the respective node class.
From the node classification perspective,
even with high class-homophily and informative features, a large CFH can result in degraded GNN performance (Fig.~\ref{fig:csbm-x-main-exp}).

The central implications are two-fold.
In hindsight, our findings in concert suggest that the recent success of GNNs may have relied on the generally small CFH of the benchmark datasets.
Looking forward, investigating the role of CFH on GNNs is a promising research direction.

\smallsection{Limitations and future works}.
We close with our discussion on limitations and potential research directions.
First, we did not provide new benchmark datasets or GNNs that address varying levels of CFH.
The current benchmark datasets generally have low and positive CFH (see Observation~\ref{obs:absolte_homophily}).
According to our findings, the existing GNNs may significantly underperform in datasets with large CFH.
Thus, proposing new datasets with a large and/or negative CFH and designing methods for such datasets would be valuable contributions.

Second, we did not explore potential applications of our findings.
Low CFH can significantly contribute to GNN performance (see Fig.~\ref{fig:shuffle-main-exp}).
The feature shuffle algorithm in Sec.~\ref{sec:swapping_exp}, however, requires test labels, which are not known, and CFH values are also not known without test labels.
Thus, a method that estimates and adaptively lowers CFH, while keeping class-homophily and feature informativeness intact, may substantially improve GNN performance.
We explore a potential application in Appendix~\ref{app:applications}.

Last, generalization of our theoretical findings is limited since both CFH measure \hdot and \gmodel, implicitly and explicitly, 
assume that the relation between node features and class is linear (see Eq.~\eqref{eq.class_control}) and 
that the node-level CFH distribution is symmetric and unimodal (see Eq.~\eqref{eq.graph_homophily}).
However, the patterns in the real-world graphs should be more complex.
Exploring how more realistic patterns interact with GNNs would be a valuable next step.

\section*{Impact statement}
%Authors are required to include a statement of the potential broader impact of their work, including its ethical aspects and future societal consequences. 
We do not expect any immediate, negative societal impact of the present work.
It delves into a theory of graph neural networks and, thereby, may indirectly benefit their applications to be more effective and reliable.

\section*{Acknowledgements}
This work was supported by Institute of Information \& Communications Technology Planning \& Evaluation (IITP) grant funded by the Korea government (MSIT) (No. 2022-0-00871, Development of AI Autonomy and Knowledge Enhancement for AI Agent Collaboration) (No. 2019-0-00075, Artificial Intelligence Graduate School Program (KAIST)). This work was partly supported by the National Research Foundation of Korea (NRF) grant funded by the Korea government (MSIT) (RS-2024-00341425).

\normalem
\bibliography{bib_final}

\begin{thebibliography}{50}
\providecommand{\natexlab}[1]{#1}
\providecommand{\url}[1]{\texttt{#1}}
\expandafter\ifx\csname urlstyle\endcsname\relax
  \providecommand{\doi}[1]{doi: #1}\else
  \providecommand{\doi}{doi: \begingroup \urlstyle{rm}\Url}\fi

\bibitem[Abboud et~al.(2021)Abboud, Ceylan, Grohe, and Lukasiewicz]{Abboud2021TheInitialization}
Abboud, R., Ceylan, I.~I., Grohe, M., and Lukasiewicz, T.
\newblock {The surprising power of graph neural networks with random node initialization}.
\newblock In \emph{IJCAI}, 2021.

\bibitem[Abu-El-Haija et~al.(2019)Abu-El-Haija, Perozzi, Kapoor, Alipourfard, Lerman, Harutyunyan, Ver~Steeg, and Galstyan]{Abu-El-Haija2019Mixhop:Mixing}
Abu-El-Haija, S., Perozzi, B., Kapoor, A., Alipourfard, N., Lerman, K., Harutyunyan, H., Ver~Steeg, G., and Galstyan, A.
\newblock {Mixhop: Higher-order graph convolutional architectures via sparsified neighborhood mixing}.
\newblock In \emph{ICML}, 2019.

\bibitem[Baranwal et~al.(2021)Baranwal, Fountoulakis, and Jagannath]{Baranwal2021GraphGeneralization}
Baranwal, A., Fountoulakis, K., and Jagannath, A.
\newblock {Graph convolution for semi-supervised classification: Improved linear separability and out-of-distribution generalization}.
\newblock In \emph{ICML}, 2021.

\bibitem[Baranwal et~al.(2023)Baranwal, Fountoulakis, and Jagannath]{Baranwal2023EffectsNetworks}
Baranwal, A., Fountoulakis, K., and Jagannath, A.
\newblock {Effects of graph convolutions in multi-layer networks}.
\newblock In \emph{ICLR}, 2023.

\bibitem[Bodnar et~al.(2022)Bodnar, Di~Giovanni, Chamberlain, Li{\`{o}}, and Bronstein]{Bodnar2022NeuralGnns}
Bodnar, C., Di~Giovanni, F., Chamberlain, B., Li{\`{o}}, P., and Bronstein, M.
\newblock {Neural sheaf diffusion: A topological perspective on heterophily and oversmoothing in gnns}.
\newblock In \emph{NeurIPS}, 2022.

\bibitem[Bojchevski \& G{\"u}nnemann(2018)Bojchevski and G{\"u}nnemann]{citataion_full}
Bojchevski, A. and G{\"u}nnemann, S.
\newblock Deep gaussian embedding of graphs: Unsupervised inductive learning via ranking.
\newblock In \emph{ICLR}, 2018.

\bibitem[Chen et~al.(2020)Chen, Wei, Huang, Ding, and Li]{Chen2020SimpleNetworks}
Chen, M., Wei, Z., Huang, Z., Ding, B., and Li, Y.
\newblock {Simple and deep graph convolutional networks}.
\newblock In \emph{ICML}, 2020.

\bibitem[Chien et~al.(2021)Chien, Peng, Li, and Milenkovic]{Chien2021AdaptiveNetwork}
Chien, E., Peng, J., Li, P., and Milenkovic, O.
\newblock {Adaptive universal generalized pageRank graph neural network}.
\newblock In \emph{ICLR}, 2021.

\bibitem[Cui et~al.(2022)Cui, Lu, Li, and Yang]{Cui2022OnGraphs}
Cui, H., Lu, Z., Li, P., and Yang, C.
\newblock {On positional and structural node features for graph neural networks on non-attributed graphs}.
\newblock In \emph{CIKM}, 2022.

\bibitem[Deng et~al.(2024)Deng, Yue, and Zhang]{deng2024polynormer}
Deng, C., Yue, Z., and Zhang, Z.
\newblock Polynormer: Polynomial-expressive graph transformer in linear time.
\newblock In \emph{ICLR}, 2024.

\bibitem[Deshpande et~al.(2018)Deshpande, Montanari, Mossel, and Sen]{Deshpande2018ContextualModels}
Deshpande, Y., Montanari, A., Mossel, E., and Sen, S.
\newblock {Contextual stochastic block models}.
\newblock In \emph{NeurIPS}, 2018.

\bibitem[Duong et~al.(2019)Duong, Hoang, Dang, Nguyen, and Aberer]{Duong2019OnNetworks}
Duong, C.~T., Hoang, T.~D., Dang, H. T.~H., Nguyen, Q. V.~H., and Aberer, K.
\newblock {On node features for graph neural networks}.
\newblock \emph{arXiv preprint arXiv:1911.08795}, 2019.

\bibitem[Gravina et~al.(2023)Gravina, Bacciu, and Gallicchio]{gravina2022anti}
Gravina, A., Bacciu, D., and Gallicchio, C.
\newblock Anti-symmetric dgn: A stable architecture for deep graph networks.
\newblock In \emph{ICLR}, 2023.

\bibitem[Grover \& Leskovec(2016)Grover and Leskovec]{Grover2016Node2vec:Networks}
Grover, A. and Leskovec, J.
\newblock {Node2vec: Scalable feature learning for networks}.
\newblock In \emph{KDD}, 2016.

\bibitem[Holland et~al.(1983)Holland, Laskey, and Leinhardt]{Holland1983StochasticSteps}
Holland, P.~W., Laskey, K.~B., and Leinhardt, S.
\newblock {Stochastic blockmodels: First steps}.
\newblock \emph{Social Networks}, 5\penalty0 (2), 1983.

\bibitem[Hu et~al.(2020)Hu, Fey, Zitnik, Dong, Ren, Liu, Catasta, and Leskovec]{ogb_graphs}
Hu, W., Fey, M., Zitnik, M., Dong, Y., Ren, H., Liu, B., Catasta, M., and Leskovec, J.
\newblock Open graph benchmark: Datasets for machine learning on graphs.
\newblock In \emph{NeurIPS}, 2020.

\bibitem[Kailath(1967)]{kailath1967divergence}
Kailath, T.
\newblock The divergence and bhattacharyya distance measures in signal selection.
\newblock \emph{IEEE Transactions on Communication Technology}, 15\penalty0 (1):\penalty0 52--60, 1967.

\bibitem[Kingma \& Ba(2015)Kingma and Ba]{Kingma2015Adam:Optimization}
Kingma, D.~P. and Ba, J.
\newblock {Adam: A method for stochastic optimization}.
\newblock In \emph{ICLR}, 2015.

\bibitem[Kipf \& Welling(2017)Kipf and Welling]{Kipf2017Semi-supervisedNetworks}
Kipf, T.~N. and Welling, M.
\newblock {Semi-supervised classification with graph convolutional networks}.
\newblock In \emph{ICLR}, 2017.

\bibitem[Lee et~al.(2024)Lee, Lee, and Shin]{lee2024villain}
Lee, G., Lee, S.~Y., and Shin, K.
\newblock Villain: Self-supervised learning on hypergraphs without features via virtual label propagation.
\newblock In \emph{WWW}, 2024.

\bibitem[Lee et~al.(2023)Lee, Bu, Yoo, and Shin]{Lee2023TowardsRemedies}
Lee, S.~Y., Bu, F., Yoo, J., and Shin, K.
\newblock {Towards deep attention in graph neural networks: Problems and remedies}.
\newblock In \emph{ICML}, 2023.

\bibitem[Lim et~al.(2021)Lim, Hohne, Li, Huang, Gupta, Bhalerao, and Lim]{Lim2021LargeMethods}
Lim, D., Hohne, F., Li, X., Huang, S.~L., Gupta, V., Bhalerao, O., and Lim, S.~N.
\newblock {Large scale learning on non-homophilous graphs: New benchmarks and strong simple methods}.
\newblock In \emph{NeurIPS}, 2021.

\bibitem[Luan et~al.(2022)Luan, Hua, Lu, Zhu, Zhao, Zhang, Chang, and Precup]{Luan2022RevisitingNetworks}
Luan, S., Hua, C., Lu, Q., Zhu, J., Zhao, M., Zhang, S., Chang, X.~W., and Precup, D.
\newblock {Revisiting heterophily for graph neural networks}.
\newblock In \emph{NeurIPS}, 2022.

\bibitem[Luan et~al.(2023)Luan, Hua, Xu, Lu, Zhu, Chang, Fu, Leskovec, and Precup]{Luan2023WhenDistinguishability}
Luan, S., Hua, C., Xu, M., Lu, Q., Zhu, J., Chang, X.-W., Fu, J., Leskovec, J., and Precup, D.
\newblock {When do graph neural networks help with node classification? Investigating the impact of homophily principle on node distinguishability}.
\newblock In \emph{NeurIPS}, 2023.

\bibitem[Ma et~al.(2021)Ma, Liu, Zhao, Liu, Tang, and Shah]{Ma2021ADenoising}
Ma, Y., Liu, X., Zhao, T., Liu, Y., Tang, J., and Shah, N.
\newblock {A unified view on graph neural networks as graph signal denoising}.
\newblock In \emph{CIKM}, 2021.

\bibitem[Ma et~al.(2022)Ma, Liu, Shah, and Tang]{Ma2022IsNetworks}
Ma, Y., Liu, X., Shah, N., and Tang, J.
\newblock {Is homophily a necessity for graph neural networks?}
\newblock In \emph{ICLR}, 2022.

\bibitem[Mao et~al.(2023)Mao, Chen, Jin, Han, Ma, Zhao, Shah, and Tang]{Mao2023DemystifyingAll}
Mao, H., Chen, Z., Jin, W., Han, H., Ma, Y., Zhao, T., Shah, N., and Tang, J.
\newblock {Demystifying structural disparity in graph neural networks: Can one size fit all?}
\newblock \emph{arXiv preprint arXiv:2306.01323}, 2023.

\bibitem[McPherson et~al.(2001)McPherson, Smith-Lovin, and Cook]{McPherson2001BirdsNetworks}
McPherson, M., Smith-Lovin, L., and Cook, J.~M.
\newblock {Birds of a feather: Homophily in social networks}.
\newblock \emph{Annual Review of Sociology}, 27\penalty0 (1):\penalty0 415--444, 2001.

\bibitem[Mernyei \& Cangea(2020)Mernyei and Cangea]{Mernyei2020Wiki-cs:Networks}
Mernyei, P. and Cangea, C.
\newblock {Wiki-cs: A wikipedia-based benchmark for graph neural networks}.
\newblock \emph{arXiv preprint arXiv:2007.02901}, 2020.

\bibitem[NT \& Maehara(2019)NT and Maehara]{NT2019RevisitingFilters}
NT, H. and Maehara, T.
\newblock {Revisiting graph neural betworks: All we have is low-pass filters}.
\newblock \emph{arXiv preprint arXiv:1905.09550}, 2019.

\bibitem[Oono \& Suzuki(2020)Oono and Suzuki]{Oono2020GraphClassification}
Oono, K. and Suzuki, T.
\newblock {Graph neural networks exponentially lose expressive power for node classification}.
\newblock In \emph{ICLR}, 2020.

\bibitem[Palowitch et~al.(2022)Palowitch, Tsitsulin, Mayer, and Perozzi]{Palowitch2022Graphworld:Gnns}
Palowitch, J., Tsitsulin, A., Mayer, B., and Perozzi, B.
\newblock {Graphworld: Fake graphs bring real insights for gnns}.
\newblock In \emph{KDD}, 2022.

\bibitem[Pei et~al.(2020)Pei, Wei, Chang, Lei, and Yang]{Pei2020Geom-gcn:Networks}
Pei, H., Wei, B., Chang, K. C.-C., Lei, Y., and Yang, B.
\newblock {Geom-gcn: Geometric graph convolutional networks}.
\newblock In \emph{ICLR}, 2020.

\bibitem[Platonov et~al.(2023{\natexlab{a}})Platonov, Kuznedelev, Babenko, and Prokhorenkova]{Platonov2023CharacterizingBeyond}
Platonov, O., Kuznedelev, D., Babenko, A., and Prokhorenkova, L.
\newblock {Characterizing graph datasets for node classification: Homophily-heterophily dichotomy and beyond}.
\newblock In \emph{NeurIPS}, 2023{\natexlab{a}}.

\bibitem[Platonov et~al.(2023{\natexlab{b}})Platonov, Kuznedelev, Diskin, Babenko, and Prokhorenkova]{Platonov2023AProgress}
Platonov, O., Kuznedelev, D., Diskin, M., Babenko, A., and Prokhorenkova, L.
\newblock {A critical look at the evaluation of gnns under heterophily: Are we really making progress?}
\newblock In \emph{ICLR}, 2023{\natexlab{b}}.

\bibitem[Rogers et~al.(2014)Rogers, Singhal, and Quinlan]{Rogers2019DiffusionInnovations}
Rogers, E.~M., Singhal, A., and Quinlan, M.~M.
\newblock Diffusion of innovations.
\newblock In \emph{An integrated approach to communication theory and research}, pp.\  432--448. Routledge, 2014.

\bibitem[Shchur et~al.(2018)Shchur, Mumme, Bojchevski, and G{\"u}nnemann]{coauthor_copurchase}
Shchur, O., Mumme, M., Bojchevski, A., and G{\"u}nnemann, S.
\newblock Pitfalls of graph neural network evaluation.
\newblock \emph{arXiv preprint arXiv:1811.05868}, 2018.

\bibitem[Stevens(2012)]{Stevens2009AppliedEd}
Stevens, J.~P.
\newblock \emph{Applied multivariate statistics for the social sciences}.
\newblock Routledge, 2012.

\bibitem[Tang et~al.(2009)Tang, Sun, Wang, and Yang]{actor_graph}
Tang, J., Sun, J., Wang, C., and Yang, Z.
\newblock Social influence analysis in large-scale networks.
\newblock In \emph{KDD}, 2009.

\bibitem[Wang et~al.(2024)Wang, Guo, Yang, and Wang]{wang2024understanding}
Wang, J., Guo, Y., Yang, L., and Wang, Y.
\newblock Understanding heterophily for graph neural networks.
\newblock \emph{arXiv preprint arXiv:2401.09125}, 2024.

\bibitem[Wang \& Zhang(2022)Wang and Zhang]{Wang2022HowNetworks}
Wang, X. and Zhang, M.
\newblock {How powerful are spectral graph neural networks}.
\newblock In \emph{ICML}, 2022.

\bibitem[Wei et~al.(2022)Wei, Yin, Jia, Benson, and Li]{Wei2022UnderstandingPerspective}
Wei, R., Yin, H., Jia, J., Benson, A.~R., and Li, P.
\newblock {Understanding non-linearity in graph neural networks from the bayesian-inference perspective}.
\newblock In \emph{NeurIPS}, 2022.

\bibitem[Wu et~al.(2019)Wu, Souza, Zhang, Fifty, Yu, and Weinberger]{Wu2019SimplifyingNetworks}
Wu, F., Souza, A., Zhang, T., Fifty, C., Yu, T., and Weinberger, K.
\newblock {Simplifying graph convolutional networks}.
\newblock In \emph{ICML}, 2019.

\bibitem[Wu et~al.(2023)Wu, Chen, Wang, and Jadbabaie]{Wu2023ANetworks}
Wu, X., Chen, Z., Wang, W., and Jadbabaie, A.
\newblock {A non-asymptotic analysis of oversmoothing in graph neural networks}.
\newblock In \emph{ICLR}, 2023.

\bibitem[Yan et~al.(2022)Yan, Hashemi, Swersky, Yang, and Koutra]{Yan2022TwoNetworks}
Yan, Y., Hashemi, M., Swersky, K., Yang, Y., and Koutra, D.
\newblock {Two sides of the same coin: Heterophily and oversmoothing in graph convolutional neural networks}.
\newblock In \emph{ICDM}, 2022.

\bibitem[Yang et~al.(2016)Yang, Cohen, and Salakhudinov]{planetoid}
Yang, Z., Cohen, W., and Salakhudinov, R.
\newblock Revisiting semi-supervised learning with graph embeddings.
\newblock In \emph{ICML}, 2016.

\bibitem[You et~al.(2021)You, Gomes-Selman, Ying, and Leskovec]{You2021Identity-awareNetworks}
You, J., Gomes-Selman, J.~M., Ying, R., and Leskovec, J.
\newblock {Identity-aware graph neural networks}.
\newblock In \emph{AAAI}, 2021.

\bibitem[Zeng et~al.(2020)Zeng, Zhou, Srivastava, Kannan, and Prasanna]{flickr}
Zeng, H., Zhou, H., Srivastava, A., Kannan, R., and Prasanna, V.
\newblock Graphsaint: Graph sampling based inductive learning method.
\newblock In \emph{ICML}, 2020.

\bibitem[Zhu et~al.(2020)Zhu, Yan, Zhao, Heimann, Akoglu, and Koutra]{Zhu2020BeyondDesigns}
Zhu, J., Yan, Y., Zhao, L., Heimann, M., Akoglu, L., and Koutra, D.
\newblock {Beyond homophily in graph neural networks: Current limitations and effective designs}.
\newblock In \emph{NeurIPS}, 2020.

\bibitem[Zhu et~al.(2021)Zhu, Rossi, Rao, Mai, Lipka, Ahmed, and Koutra]{Zhu2021GraphHeterophily}
Zhu, J., Rossi, R.~A., Rao, A., Mai, T., Lipka, N., Ahmed, N.~K., and Koutra, D.
\newblock {Graph neural networks with heterophily}.
\newblock In \emph{AAAI}, 2021.

\end{thebibliography}
\bibliographystyle{icml2024}

% \clearpage    
\appendix   
\onecolumn

\section{Proofs and Additional Theoretical Results}\label{app:proofs}
\subsection{Proofs for Measure \hdot}

Throughout our proof w.r.t. measure \hdot,
let us assume that a graph $G$ has no isolated nodes.
Also, recall that if $b(\cdot) = 0$ (i.e., all nodes have the same class-controlled features), we define \hn and \hg as 0.

%%%%%%%%%%%%%%%%%%%%%%%%%%%%%%%%%%%%%%%%%%%%%%%%%%%%%%%%%%%%%%%%%%%%%%%%%%%%%%%%%%%%%%%%%%

\smallsection{Proof of Lemma~\ref{lem:measure_bound} (Boundedness).}

\begin{proof}
\textbf{\textit{Bound of \hn}}.
The node-level CFH \hn is defined as follows:
\begin{align*}
    \mathbf{\Tilde{h}}^{(v)}_{i} 
    = \frac{\mathbf{h}^{(v)}_{i}} {\max(b(v_{i}) , \mathbf{d}(v_{i}, N_i))}
    = \frac{b(v_{i}) - \mathbf{d}(v_{i}, N_i)} {\max(b(v_{i}) , \mathbf{d}(v_{i}, N_i))}.
\end{align*}

L2 norm is non-negative, and thus, both $\mathbf{d(\cdot)}, b(\cdot)$ are non-negative.
Since $\abs{b(v_{i}) - \mathbf{d}(v_{i}, N_i)} \leq \max(b(v_{i}) , \mathbf{d}(v_{i}, N_i))$, 
$\mathbf{\Tilde{h}}^{(v)}_{i} \in [-1, 1]$ holds, completing the proof of bound for node-level CFH \hn.
\end{proof}

\begin{proof}
\textbf{\textit{Bound of \hg}}.
The graph-level CFH \hg can be rewritten as:
\begin{align*}
    \mathbf{\Tilde{h}}^{(G)} 
    &= \frac{\mathbf{h}^{(G)}} {\frac{1}{\abs{V}} \max(\sum_{v_i \in V} b(v_{i}),  \sum_{v_i \in V} \mathbf{d}(v_{i}, N_i))} \\
    &= \frac{ \frac{1}{\abs{V}} \sum_{v_i \in V}\mathbf{{h}}^{(v)}_i} {\frac{1}{\abs{V}} \max(\sum_{v_i \in V} b(v_{i}),  \sum_{v_i \in V} \mathbf{d}(v_{i}, N_i))} \\
    %&= \frac{ \frac{1}{\abs{V}} \sum_{v_i \in V}\mathbf{{h}}^{(v)}_i} {\frac{1}{\abs{V}} \max(\sum_{v_i \in V} b(v_{i}),  \sum_{v_i \in V} \mathbf{d}(v_{i}, N_i))} \\
    % &= \frac{\sum_{v_i \in V}(b(v_{i}) - \mathbf{d}(v_{i}, N_i))} {\max(\sum_{v_i \in V} b(v_{i}),  \sum_{v_i \in V} \mathbf{d}(v_{i}, N_i))} \\
    &= \frac{\sum_{v_i \in V}b(v_{i}) - \sum_{v_i \in V}\mathbf{d}(v_{i}, N_i)} {\max(\sum_{v_i \in V} b(v_{i}),  \sum_{v_i \in V} \mathbf{d}(v_{i}, N_i))}.
\end{align*}
For the same reason as \hn, $\mathbf{\Tilde{h}}^{(G)} \in [-1, 1]$, completing the proof of bound for graph-level CFH \hg.
\end{proof}

\begin{proof}
\textbf{\textit{Existence Claim}}.
We show that the upper/lower bound is achievable under a non-asymptotic/asymptotic setting.
First, we show that $\sup_{G} \mathbf{\Tilde{h}}^{(G)} = 1$ holds.
Consider a disconnected $G$ such that class-controlled features of neighboring nodes are all equal (i.e., $X_{i}|Y =X_{j}|Y, \forall (v_{i},v_{j}) \in E$), while that of disconnected nodes are different (i.e., $X_{k}|Y \neq X_{\ell}|Y$, where there does not exist a path between $v_{k}$ and $v_{\ell}$).
In such a case, $b({v_{i}}) \neq 0$ and $\mathbf{d}(v_{i},N_{i}) = 0$ hold $\forall v_{i} \in V$. 
Thus, $\mathbf{\Tilde{h}}^{(v)}_{i} = 1, \forall v_{i} \in V$ also holds, and consequently, $\max_{G} \mathbf{\Tilde{h}}^{(G)} = \sup_{G} \mathbf{\Tilde{h}}^{(G)} = 1$ holds.

Second, we show that $\inf_{G} \mathbf{\Tilde{h}}^{(G)} = -1$ holds.
Consider a case where $\mathbf{d}(v_{i}, N_{i}) \rightarrow \infty$ and $o(b(v_{i})) < o(\mathbf{d}(v_{i}, N_{i})), \forall v_{i} \in V$ hold.
In such a case, the following holds:
\begin{equation}
    \lim_{\mathbf{d}(v_{i}, N_{i}) \rightarrow \infty} \tilde{\mathbf{h}}^{(v)}_{i}= \frac{b(v_{i}) - \mathbf{d}(v_{i}, N_i)} {\mathbf{d}(v_{i}, N_i)} \equiv -\frac{\mathbf{d}(v_{i}, N_i)} {\mathbf{d}(v_{i}, N_i)} = -1, \forall v_{i} \in V.
\end{equation}
Consequently, $\inf_{G} \mathbf{\Tilde{h}}^{(G)} = -1$ hold.
Note that the second result is derived under the asymptotic scenario, and thus, the result does not indicate the exact minimum.
\end{proof}

%%%%%%%%%%%%%%%%%%%%%%%%%%%%%%%%%%%%%%%%%%%%%%%%%%%%%%%%%%%%%%%%%%%%%%%%%%%%%%%%%%%%%%%%%%

\smallsection{Proof of Lemma~\ref{lem:measure_scale} (Scale-Invariance)}.
\begin{proof}
\textbf{\textit{Scale-Invariance of \hn}}. 
Denote the distance function (Eq.~\eqref{eq.distance_func}) with node feature $X$ as $\mathbf{d}'(v_{i},V'_{i}, X)$.
Then, for any $c\in \mathbb{R}\setminus \{0\}$, the following holds:
\begin{align*}
    \mathbf{d}'(v_{i}, V'_i, cX) 
    &\coloneqq \frac{1}{\vert V'_i\vert}\sum_{v_{j} \in V'_i}\lVert (c \cdot X_{i} \vert Y) - (c \cdot X_{j} \vert Y) \rVert_{2} \\
    &= \vert c \vert\cdot \left( \frac{1}{\vert V'_i\vert}\sum_{v_{j} \in V'_i}\lVert (X_{i} \vert Y) - (X_{j} \vert Y) \rVert_{2} \right) \\
    &= \vert c \vert \cdot \mathbf{d}'(v_{i}, V'_i, X).
\end{align*} 
Likewise, we denote a homophily baseline $b(v_{i})$ with a node feature $X$ as $b'(v_{i}, X)$.
Then, since $b(v_{i}, cX)$ is a special case of Eq.~\eqref{eq.distance_func}, the following holds: $b'(v_{i},cX) = \vert c \vert~b'(v_{i}, X)$.
Lastly, we denote \hn with a node feature $X$ as $\mathbf{\Tilde{h}}^{(v)}_{i} (X)$.
Then, by showing the below, we finalize the proof for node-level CFH \hn.
\begin{align*}
    \mathbf{\Tilde{h}}^{(v)}_{i}(cX) 
    &= \frac{b'(v_{i})(cX) - \mathbf{d}'(v_{i}, N_i, cX)} {\max(b'(v_{i}, cX) , \mathbf{d}'(v_{i}, N_i, cX))} \\
    &= \frac{\vert c\vert \cdot (b'(v_{i}, X) - \mathbf{d}'(v_{i}, N_i, X))} {\vert c\vert \cdot (\max(b'(v_{i}, X) , \mathbf{d}'(v_{i}, N_i, X)))} \\
    &= \frac{b'(v_{i}, X) - \mathbf{d}'(v_{i}, N_i, X)} {\max(b'(v_{i}, X) , \mathbf{d}'(v_{i}, N_i, X))} = \mathbf{\Tilde{h}}^{(v)}_{i}(X).
\end{align*}
\end{proof}

\begin{proof}
\textbf{\textit{Scale-Invariance of \hg}}. 
We denote \hg with a node feature $X$ as $\mathbf{\Tilde{h}}^{(G)} (X)$.
Then, we finalize the proof for graph-level CFH \hg by extending the above results. %of \textbf{Proof of Lemma~\ref{lem:measure_scale}.
 \begin{align*}
     \mathbf{\Tilde{h}}^{(G)}(cX) 
     &= \frac{\sum_{v_i \in V}b(v_{i}, cX) - \sum_{v_i \in V}\mathbf{d}(v_{i}, N_i, cX)} {\max(\sum_{v_i \in V} b(v_{i}, cX),  \sum_{v_i \in V} \mathbf{d}(v_{i}, N_i, cX))} \\
     &= \frac{\vert c\vert \cdot (\sum_{v_i \in V}b(v_{i}, X) - \sum_{v_i \in V}\mathbf{d}(v_{i}, N_i, X))} {\vert c\vert \cdot (\max(\sum_{v_i \in V} b(v_{i}, X),  \sum_{v_i \in V} \mathbf{d}(v_{i}, N_i, X)))} \\
     &= \frac{\sum_{v_i \in V}b'(v_{i}, X) - \sum_{v_i \in V}\mathbf{d}'(v_{i}, N_i, X)} {\max(\sum_{v_i \in V} b'(v_{i}, X),  \sum_{v_i \in V} \mathbf{d}'(v_{i}, N_i, X))} = \mathbf{\Tilde{h}}^{(G)}(X) .
\end{align*}
\end{proof}

%%%%%%%%%%%%%%%%%%%%%%%%%%%%%%%%%%%%%%%%%%%%%%%%%%%%%%%%%%%%%%%%%%%%%%%%%%%%%%%%%%%%%%%%%%

\smallsection{Proof of Lemma~\ref{lem:measure_mono} (Monotonicity)}.

\begin{proof}
First, since node features of $v_{k} \in V \setminus N_{i}$ are fixed, we rewrite homophily baseline $b(v_{i})$ as $b(v_{i}) = \frac{\vert N_{i}\vert}{\vert V'_{i}\vert}\mathbf{d}(v_{i}, N_i) + C$, where $C$ is a fixed constant.
For simplicity, denote $\mathbf{d}(v_{i}, N_i)$ and $\frac{\vert N_{i}\vert}{\vert V'_{i}\vert}$ as $K$ and $a$, respectively. 
Thus, the following holds: $b(v_{i}) := aK + C$.
We break the rest of the proof down into two parts.

\textbf{Case 1}: $b(v_i) \geq \mathbf{d}(v_{i}, N_i)$.
Node-level CFH \hn can be rewritten as
\begin{align}
    &\mathbf{\Tilde{h}}^{(v)}_{i} 
    = \frac{b(v_{i}) - \mathbf{d}(v_{i}, N_i)} {b(v_{i})} = \frac{(a - 1)K + C}{aK + C} \nonumber \\
    &\frac{\partial \mathbf{\Tilde{h}}^{(v)}_{i} }{\partial K} = \frac{-1}{(aK + C)^{2}} < 0.\label{eq:monoton1}
\end{align}

\textbf{Case 2}: $b(v_i) < \mathbf{d}(v_{i}, N_i)$.
Node-level CFH \hn can be rewritten as
\begin{align}
    &\mathbf{\Tilde{h}}^{(v)}_{i} 
    = \frac{b(v_{i}) - \mathbf{d}(v_{i}, N_i)} {\mathbf{d}(v_{i}, N_i)} = \frac{(a - 1)K + C}{K} \nonumber \\ 
    &\frac{\partial \mathbf{\Tilde{h}}^{(v)}_{i} }{\partial K} = \frac{a - 2}{K^{2}} < 0, \quad \because a < 1.\label{eq:monoton2}
\end{align}
By merging the result of Eq~\eqref{eq:monoton1} and Eq~\eqref{eq:monoton2}, the monotonic decreasing property is guaranteed. 
\end{proof}

%%%%%%%%%%%%%%%%%%%%%%%%%%%%%%%%%%%%%%%%%%%%%%%%%%%%%%%%%%%%%%%%%%%%%%%%%%%%%%%%%%%%%%%%%%

\subsection{Proofs for \gmodel Properties}

%%%%%%%%%%%%%%%%%%%%%%%%%%%%%%%%%%%%%%%%%%%%%%%%%%%%%%%%%%%%%%%%%%%%%%%%%%%%%%%%%%%%%%%%%%

\comment{
\smallsection{Proof of Lemma~\ref{lem:expected_snr_wrt_mu} (Independent control of FD)}.

\begin{proof}
    Assume a \gmodel graph with $q \in  (0,1)$ and $(\Sigma_{0} \Sigma_{1}) > 0$.
    the expected FD of its node classes $C_0$ and $C_1$ is:
    \begin{equation*}
        \bbE[\text{FD}] (C_0, C_1) \coloneqq  
        \sqrt{(\mu_0 - \mu_1)^\intercal \left( \frac{\Sigma_0 + \Sigma_1}{2} \right)^{-1}(\mu_0 - \mu_1)}.
    \end{equation*}
    It is trivial to see that, 
    for a given $(\Sigma_{0} \Sigma_{1})>0$, 
    $\bbE[\text{FD}](C_0, C_1)$ is a strictly increasing function of $(\mu_0 - \mu_1)$ and is invariant to changes in other parameters.
\end{proof}

\smallsection{Proof of Lemma~\ref{lem:expected_hc_wrt_d} (Independent control of \hc)}.

\begin{proof}

    Since we use weighted sampling without replacement to sample edges, the number of same- and different-class neighbors $d^+_i$ and $d^-_i$ are identical constants for all node $v_i \in V$ 
    (i.e., $d^+_i = d^+_j ~\text{and}~ d^-_i = d^-_j ~,~ \forall v_i, v_j \in V$). 
    Thus, the equation for \hc can be rewritten with \gmodel parameters as follows:
    \begin{equation*}
        \mathbf{h}_{c} = \frac{1}{c} \sum_{\ell \in [c]}
        \left[ 
        \frac
        {d^+ }
        {d^+ + d^-}
        - \frac{\abs{C^+_\ell}}{n}
        \right]_{+}.
    \end{equation*}

    Meanwhile, the number of nodes in each class $\abs{C^+_\ell}$ is determined by the Bernoulli distribution with the probability parameter $q$.
    Thus, \bbE[$\mathbf{h}_c$] can be expressed as follows:
    \begin{align*}
        \bbE[\mathbf{h}_{c}] 
        &= \frac{1}{c} \sum_{\ell \in [c]}
        \left[ \frac { d^+ } {d^+ + d^-}
        - \frac{ \abs{\ell-q} \times n} {n} \right]_{+} \\
        &= \frac{1}{c} \sum_{\ell \in [c]}
        \left[ \frac { d^+ } {d^+ + d^-}
        - \abs{\ell-q} \right]_{+} 
    \end{align*}
    Now, let $d^+ \geq d^-$ and $q \in (0,1)$.
    For a given $q$,
    it is straight-forward to see that \bbE[$\mathbf{h}_c$] is a strictly increasing function of $\frac { d^+ } {d^+ + d^-}$ and invariant of other parameters,
    completing the proof.
\end{proof}
}

\smallsection{Proof of Lemma~\ref{lem:relation_tau_hg} 
($\tau$ controls CFH $\mathbf{h}(\cdot)$ \textit{precisely})}.

\begin{proof}
    \textbf{Regarding claim (i).}
    When the other parameters are fixed, for each 
    $\calX = (X_i)_{i \in [n]} \in \bbR^{n \times k}$ and 
    $\calY = (Y_i)_{i \in [n]} \in [c]^n$.
    The joint probability $\Pr[(\calX, \calY)]$ is fixed regardless of the value of $\tau$.
    Moreover, for each node $v_i$, the numbers of same-class and different-class neighbors are fixed.    
    Now, let us fix any $\calX$ and $\calY$, it suffices to show for each node $v_i$,
    $\bbE[\mathbf{h}_i^{(v)} \mid \tau_1] < \bbE[\mathbf{h}_i^{(v)} \mid \tau_2]$.
    
    To see this, first,
    $\mathbf{h}_i^{(v)} = \sum_{v_j \in N_i} (\mathbf{d}(v_i, V'_i) - \mathbf{d}(v_i, \set{v_j}))$, where $\mathbf{d}(v_i, V'_i)$ is fixed when $\calX$ is fixed.
    Hence, we only need to show that
    \[
    \bbE[\sum_{v_j \in N_i} \mathbf{d}(v_i, \set{v_j})] = 
    \sum_{v_j \in V'_i} \Pr[v_j \in N_i] \mathbf{d}(v_i, \set{v_j})
    \]
    decreases as $\tau$ increases.
    Indeed, as $\tau$ increases,
    as long as 
    $(\mathbf{d}(v_i, \set{v_j}))$'s for $v_j \in N_i$ are not all identical (since $\Sigma_0, \Sigma_1 > 0$, there must be cases satisfying this),
    there exists a threshold $\mathbf{d}_{th}$ such that
    all the $v_j$'s with $(\mathbf{d}(v_i, \set{v_j})) < \mathbf{d}_{th}$ whose edge sampling weights (i.e., $\mathbb{P}_{ij}$'s) increase and 
    all the $v'_j$'s with $(\mathbf{d}(v_i, \set{v'_j})) > \mathbf{d}_{th}$ whose edge sampling weights decrease, which makes the ``weighted sum'' $\sum_{v_j \in V'_i} \Pr[v_j \in N_i] \mathbf{d}(v_i, \set{v_j})$ smaller.

    \textbf{Regarding claim (ii).}
    Above, we have proved that $\mathbf{h}(\cdot)$ is a strictly increasing function of $\tau$, which also means that $\mathbf{h}(\cdot)$ is an injective function of $\tau$.
    Hence, it suffices to show that if $\tau=0, ~\text{then}~ \bbE[\mathbf{{h}}(\cdot)=0]$.

    First, since we assume $\Sigma_0 = \Sigma_1 \neq 0$, the class controlled feature distributions of class-0 and class-1 are identical (i.e., $(X_i \vert Y) \sim \mathcal{N}(0,\Sigma_0), \forall v_i \in V$).
    Thus, the following holds:
    \begin{align} \label{eq:lemma_4.4_proof1}
        \bbE[\mathbf{d}(v_i, C_{Y_i}^+ \setminus \set{v_i})] = \bbE[\mathbf{d}(v_i, C_{Y_i}^-)] = \bbE[\mathbf{d}(v_i, V'_i)], \forall v_i \in V.
    \end{align}
    Recall that $C_{\ell}^+$ denotes the node set of class $\ell$, whereas $C_{\ell}^-$ denotes the set of the rest of the nodes.

    Second, if $\tau=0$, then  $\phi_{ij} = 1, \forall (i,j) \in V \times V$.
    This means that the edge sampling probabilities are identical for all the same-class node pairs and for all the different-class node pairs, respectively.
    Then, for each node $v_i$, the same-class neighbor set $N_i^+$ is chosen from $C^+_{Y_i} \setminus \set{v_i}$ uniformly at random.
    Likewise, the different-class neighbor set $N_i^-$ is chosen from $C^-_{Y_i}$ uniformly at random.
    Thus,
    \begin{align} \label{eq:lemma_4.4_proof2}
        \bbE[\mathbf{d}(v_i, N_i^+)] &= \bbE[\mathbf{d}(v_i, C_{Y_i}^+ \setminus \set{v_i})], \forall v_i \in V \nonumber \\
        \bbE[\mathbf{d}(v_i, N_i^-)] &= \bbE[\mathbf{d}(v_i, C_{Y_i}^-)], \forall v_i \in V.
    \end{align}

    Combining Eqs.~\eqref{eq:lemma_4.4_proof1} and~\eqref{eq:lemma_4.4_proof2}, the following holds if $\tau =0$: 
    \[
        \bbE[\mathbf{d}(v_i, N_i^+)] = \bbE[\mathbf{d}(v_i, N_i^-)] = \bbE[\mathbf{d}(v_i, N_i)], \forall v_i \in V.
    \]

    Since $\bbE[b(v_i)] = \bbE[\mathbf{d}(b_i, V'_i)]$ by definition, the following holds if $\tau =0$:
    \begin{align*}
        \bbE[\mathbf{h}^{(v)}_{i}] &= \bbE[b(v_i)] - \bbE[\mathbf{d}(v_i, N_i)] = 0, \\
        \bbE[\mathbf{h}^{(G)}] &= \frac{1}{\abs{V}} \sum_{v_j \in V} \bbE[\mathbf{h}^{v}_{i}] = 0.
    \end{align*}
    % $\mathbf{h}^{(v)}_{i}$ is the numerator when calculating \hn, and $\mathbf{h}^{(G)}$ is the numerator when calculating \hg.
    % Thus, both $\bbE[\mathbf{\Tilde{h}}^{(v)}_{i}]$ and $\bbE[\mathbf{\Tilde{h}}^{(G)}]$ become 0s, 
    % completing the proof.
\end{proof}

\smallsection{Proof of Lemma~\ref{lem:tau_indep_hdot} 
($\tau$ controls CFH $\mathbf{h}(\cdot)$ \textit{only})}.
\begin{proof}
    It is straightforward, 
    since the values of $\text{FD}(\calG)$     
    and $\mathbf{h}_c(\calG)$ are directly controlled by the other parameters and are independent of the value of $\tau$.    
\end{proof}

%%%%%%%%%%%%%%%%%%%%%%%%%%%%%%%%%%%%%%%%%%%%%%%%%%%%%%%%%%%%%%%%%%%%%%%%%%%%%%%%%%%%%%%%%%

\subsection{Proofs for Graph Convolution in \gmodel Graphs}

%%%%%%%%%%%%%%%%%%%%%%%%%%%%%%%%%%%%%%%%%%%%%%%%%%%%%%%%%%%%%%%%%%%%%%%%%%%%%%%%%%%%%%%%%%%

\begin{reptheorem}{thm:main_result}
    Following the analysis setting, i.e., we
    assume 
    (\textit{i}) 1-dimensional node features $\mu_\ell, \Sigma_\ell, X_i \in \mathbb{R}$,
    (\textit{ii}) symmetric feature means $\mu_0 = -\mu_1 \neq 0$ with identical variances $\Sigma_0 = \Sigma_1 = 1$, and
    we focus on asymptotic setting with 
    (\textit{iii}) fixed $p^- \neq p^+ \in (0, \frac{1}{2})$ and
    (\textit{iv}) $n \rightarrow \infty$ with $d^+ = np^+$ and $d^- = np^-$.    
    % Fix $p^+ = o(1)$, and $p^- = o(1)$, as $n \to \infty$, let $d^+ = np^+$ and $d^- = np^-$.
    % Assume $d^+ \neq d^-$, {$\mu_0 \neq \mu_1$, and $(\Sigma_0 + \Sigma_1) > 0$}.
    Use the prior distribution $\Pr[Y_i = 0] = \Pr[Y_i = 1] = 1/2$
    and fix the other parameters except for $\tau$, 
    after a step of graph convolution $D^{-1}AX$,
    the Bayes error rate (BER) of $\calF$,
    denoted by $\mathcal{B}_{\mathcal{F}}(\mathcal{G}(\cdot, \tau))$
    is     
    minimized at $\tau = 0$ and 
    {strictly} increases as $\abs{\tau}$ increases, i.e.,
    $\argmin_\tau \mathcal{B}_{\mathcal{F}}(\mathcal{G}(\cdot, \tau)) = 0$;
    $\mathcal{B}_\mathcal{F}(\mathcal{G}(\cdot, \tau_0)) < \mathcal{B}_\mathcal{F}(\mathcal{G}(\cdot, \tau_1))$ for any $\tau_0$ and $\tau_1$ such that $\abs{\tau_0} < \abs{\tau_1}$ and $\tau_0 \tau_1 > 0$.
\end{reptheorem}
\begin{proof}    
    To provide a high-level idea, after a step of graph convolution, higher $\abs{\tau}$ makes more nodes have features far from the mean of its whole class and, thus, results in a higher error in classification.

    For the simplicity of presentation, {we assume $\mu_0 = -\mu_1$}.
    Also, we illustrate with one-dimension node features $X_i \in \bbR$ for each node $v_i$ here, but the reasoning can be extended to high-dimensional features in general.
    
    WLOG, we assume {$\Sigma_0 = \Sigma_1 = 1$}, which can be ensured by feature normalization.
    As $n \to \infty$, the sample mean and variance of the node features in each class approach $\pm \mu$ and $\Sigma$.
    For a node $v_i$, WLOG (due to the symmetry), we assume $Y_i = 1$, and let its feature be $\mu + x_i$ 
    (i.e., its class-controlled feature is $x_i$).
    Let $\varphi$ be the PDF of standard normal distribution $\calN(0, 1)$, then the homophily baseline $b(v_i)$ of node $v_i$ is
    \[
    b(v_i) = \int_{-\infty}^{\infty} \varphi(x) \abs{x - x_i} dx = 
    \frac{1}{2} e^{-\frac{x_i^2}{2}} \left(e^{\frac{x_i^2}{2}} x_i \text{erf}\left(\frac{x_i}{\sqrt{2}}\right)-e^{\frac{x_i^2}{2}} x_i \text{erfc}\left(\frac{x_i}{\sqrt{2}}\right)+e^{\frac{x_i^2}{2}} x_i+2 \sqrt{\frac{2}{\pi }}\right) =
    \exp({-\frac{x_i^2}{2}}) + x_i \text{erf}(\frac{x_i}{\sqrt{2}}),
    \]    
    where ``erf'' is the Gauss error function defined as 
    \[
    \text{erf}(z) = \frac{2}{\sqrt{\pi}} \int_0^z e^{-t^2} dt
    \]
    and
    ``erfc'' is the complementary error function defined as 
    \[\text{erfc}(z) = 1 - \text{erf}(z).\]
    Hence, the CFH between $x_i$ and another node $v_j$ with class-controlled feature $x_j$ (i.e., $v_j$ has 
    feature $-\mu + x_j$ if $Y_j = 0$, and it has
    feature $\mu + x_j$ if $Y_j = 1$) is
    \[
    \mathbf{{h}}_{ij}^{(p)} = b(v_i) - {\abs{x_i - x_j}} = 
    % \frac{1}{2} e^{-\frac{x_i^2}{2}} \left(e^{\frac{x_i^2}{2}} x_i \text{erf}\left(\frac{x_i}{\sqrt{2}}\right)-e^{\frac{x_i^2}{2}} x_i \text{erfc}\left(\frac{x_i}{\sqrt{2}}\right)+e^{\frac{x_i^2}{2}} x_i+2 \sqrt{\frac{2}{\pi }}\right) - \abs{x_i - x_j},
    \exp({-\frac{x_i^2}{2}}) \sqrt{\frac{2}{\pi}} + x_i \text{erf}\left(\frac{x_i}{\sqrt{2}}\right) - \abs{x_i - x_j},
    \]
    which gives
    \[
    \phi_{ij} = \exp(\tau \mathbf{{h}}_{ij}^{(p)}) = 
    % \exp \left(\tau \left(\frac{1}{2} e^{-\frac{x_i^2}{2}} \left(e^{\frac{x_i^2}{2}} x_i \text{erf}\left(\frac{x_i}{\sqrt{2}}\right)-e^{\frac{x_i^2}{2}} x_i \text{erfc}\left(\frac{x_i}{\sqrt{2}}\right)+e^{\frac{x_i^2}{2}} x_i+2 \sqrt{\frac{2}{\pi }}\right)-| x_i-x_j| \right)\right)
    \exp \left(\tau \left(-| x_i-x_j| +x_i \text{erf}\left(\frac{x_i}{\sqrt{2}}\right)+\sqrt{\frac{2}{\pi }} \exp({-\frac{x_i^2}{2}})\right)\right)
    \]
    Since $n \to \infty$, weighted sampling without replacement approaches weighted sampling with replacement approaches, and the probability of $v_j$ being sampled as one of $v_i$'s neighbors is
    \[
    \Pr[v_j \in N_i] = \frac{\phi_{ij}}{\int_{-\infty}^{\infty} \phi_{ij'} \varphi(x_{j'}) dx_{j'}} 
    % =     \frac{2 \exp \left(-\frac{1}{2} \tau \left(2 | x_i-x_j| -x_i \text{erf}\left(\frac{x_i}{\sqrt{2}}\right)-x_i \text{erfc}\left(\frac{x_i}{\sqrt{2}}\right)+\tau-x_i\right)\right)}{\text{erfc}\left(\frac{\tau+x_i}{\sqrt{2}}\right) e^{\tau x_i \left(\text{erf}\left(\frac{x_i}{\sqrt{2}}\right)+\text{erfc}\left(\frac{x_i}{\sqrt{2}}\right)+1\right)}+\text{erfc}\left(\frac{\tau-x_i}{\sqrt{2}}\right)},
    = \frac{2 \exp({-\frac{1}{2} \tau (2 | x_i-x_j| +\tau-2 x_i)})}{\text{erfc}\left(\frac{\tau-x_i}{\sqrt{2}}\right)+\exp({2 \tau x_i}) \text{erfc}\left(\frac{\tau+x_i}{\sqrt{2}}\right)},
    \]
    and in each sampling step, 
    the sampled neighbor has a class-controlled feature equal to $x_j$ is
    \[
    \Pr[v_j \in N_i] \varphi(x_j).
    \]
    In other words, let $x_{nbr}$ denote the random variable of the class-controlled feature of a sampled neighbor, we have
    \[
    \Pr[x_{nbr} = x^*] 
    = \frac{2 \exp({-\frac{1}{2} \tau (2 | x_i-x^*| +\tau-2 x_i)})}{\text{erfc}\left(\frac{\tau-x_i}{\sqrt{2}}\right)+\exp({2 \tau x_i} )\text{erfc}\left(\frac{\tau+x_i}{\sqrt{2}}\right)} \varphi(x^*)
    = 
    \frac{\sqrt{\frac{2}{\pi }} \exp{(-\frac{1}{2} \tau (2 | x_i-x^*| +\tau-2 x_i)-\frac{x_i^2}{2})}}{\text{erfc}\left(\frac{\tau-x_i}{\sqrt{2}}\right)+\exp({2 \tau x_i}) \text{erfc}\left(\frac{\tau+x_i}{\sqrt{2}}\right)}    
    % \frac{2 \exp \left(-\frac{1}{2} \tau \left(2 | x_i-x^*| -x_i \text{erf}\left(\frac{x_i}{\sqrt{2}}\right)-x_i \text{erfc}\left(\frac{x_i}{\sqrt{2}}\right)+\tau-x_i\right)\right)}{\text{erfc}\left(\frac{\tau+x_i}{\sqrt{2}}\right) e^{\tau x_i \left(\text{erf}\left(\frac{x_i}{\sqrt{2}}\right)+\text{erfc}\left(\frac{x_i}{\sqrt{2}}\right)+1\right)}+\text{erfc}\left(\frac{\tau-x_i}{\sqrt{2}}\right)} \varphi(x^*)
    , \forall x^* \in \bbR.
    \]
    We can compute the closed-form expectation of $x_{nbr}$, which is
    \[
    \bbE[x_{nbr}] =     
    \int_{-\infty}^{\infty} x^* \Pr[x_{nbr} = x^*] dx^* =
    % \frac{\tau e^{-\frac{1}{2} (\tau-x_i)^2} \left(e^{\frac{1}{2} (\tau-x_i)^2} \text{erfc}\left(\frac{\tau-x_i}{\sqrt{2}}\right)-e^{\frac{1}{2} (\tau+x_i)^2} \text{erfc}\left(\frac{\tau+x_i}{\sqrt{2}}\right)\right)}{\text{erfc}\left(\frac{\tau-x_i}{\sqrt{2}}\right)+e^{2 \tau x_i} \text{erfc}\left(\frac{\tau+x_i}{\sqrt{2}}\right)},
    \frac{\tau \left(\text{erfc}\left(\frac{\tau-x_i}{\sqrt{2}}\right)-\exp{(2 \tau x_i)} \text{erfc}\left(\frac{\tau+x_i}{\sqrt{2}}\right)\right)}{\text{erfc}\left(\frac{\tau-x_i}{\sqrt{2}}\right)+\exp{(2 \tau x_i)} \text{erfc}\left(\frac{\tau+x_i}{\sqrt{2}}\right)},
    \]
    and its variance $\Var[x_{nbr}]$ does not depend on the value of $n$.
    After one step of graph convolution, the new node feature of $v_i$ would be
    \[
    \hat{x}_i = \frac{d^+ \mu - d^- \mu + \sum_{t = 1}^{d^+ + d^-} x_{t}}{d^+ + d^-},
    \]   
    where each $x_t$ i.i.d. follows $x_{nbr}$.
    By the central limit theorem, as $n \to \infty$ and, thus, $d^+$ and $d^-$ approaches infinity,
    $\hat{x}_i$ asymptotically follows
    $\calN(\frac{d^+ \mu - d^- \mu}{d^+ + d^-} + \bbE[x_{nbr}], \Var[x_{nbr}] / \sqrt{d^+ + d^-})$.
    WLOG, we assume $d^+ > d^-$ here (when $d^+ < d^-$, the classifier is flipped in a symmetric manner).
    The classifier would be equivalent to
    \[
    \calF(x) = 
    \begin{cases}
    1 &\quad\text{if } x \geq 0  \\
    0 &\quad\text{otherwise }
    \end{cases},
    \]
    and the probability that $\calF$ misclassifies $v_i$ is
    $\Pr[\hat{x}_i < 0]$, which would approach a binary function (since $\Var[x_{nbr}] / \sqrt{d^+ + d^-}$ approaches zero)
    $\bm{1}[\frac{d^+ \mu - d^- \mu}{d^+ + d^-} + \bbE[x_{nbr}] < 0]$.

    Now, we first claim that $\tau = 0$ asymptotically gives the lowest error probability $0$.
    Indeed, when $\tau = 0$, $\bbE[x_{nbr}] = 0$ regardless of the value of $x_i$, and $\bm{1}[\frac{d^+ \mu - d^- \mu}{d^+ + d^-} + \bbE[x_{nbr}] < 0] \to \bm{1}[\frac{d^+ \mu - d^- \mu}{d^+ + d^-} < 0] = 0$.

    Then, we claim that in both directions, the Bayes error rate (BER) of $\calF$ increases as $\abs{\tau}$ increases.
    First, by the symmetric prior, the BER can be written as
    \begin{equation*}
        \mathcal{B}_{\mathcal{F}}(\mathcal{G}(\cdot, \tau))
        = \frac{1}{2} \left( 
    \Pr[\calF(x_i) = 1 \mid Y_i = 0] +
    \Pr[\calF(x_i) = 0 \mid Y_i = 1] 
    \right)
    \end{equation*}
    Again, due to the symmetry, it is equal to
    \begin{equation*}
    \Pr[\calF({x}_i) = 0 \mid Y_i = 1]
    = \int \Pr[\calF({x}_i) = 0] \Pr[{x}_i \mid Y_i = 1]  d {x}_i.
    \end{equation*}
    By the above analysis, after a step of graph convolution, the BER is
    \[
    \int \Pr[\calF({x}_i) = 0] \Pr[{x}_i \mid Y_i = 1] d {x}_i 
    = \int_{-\infty}^{\infty} \Pr[\hat{x}_i < 0] \varphi[{x}_i] d {x}_i        
    \]
    approaching
    \[
    \int_{-\infty}^{\infty} \bm{1}[\frac{d^+ \mu - d^- \mu}{d^+ + d^-} + \bbE[x_{nbr}] < 0] \varphi[{x}_i] d {x}_i.
    \]
    
    When $\tau > 0$, $\bbE[x_{nbr}]$ has the same sign as $x_i$. 
    In such case, we only need to consider $x_i < 0$, since $\bbE[x_{nbr}] > 0$ when $x_i > 0$.
    We claim that for any fixed $x_i < 0$,
    $\bbE[x_{nbr}]$ (w.r.t. $x_i$ and $\tau$) is decreasing w.r.t $\tau > 0$,
    and thus 
    $\bm{1}[\frac{d^+ \mu - d^- \mu}{d^+ + d^-} + \bbE[x_{nbr}] < 0]$ is non-decreasing for all $x_i < 0$, which implies the increase in the BER.
    Indeed,
    
    \resizebox{\textwidth}{!}{
    \begin{minipage}{1.7\textwidth}
    \[    
    \frac{\partial}{\partial \tau} \bbE[x_{nbr}] = 
    \frac{\exp{(\tau x_i)} \left(2 \tau \exp({\tau^2+x_i^2}) \left(\sqrt{\frac{2}{\pi }}-2 x_i \exp({\frac{1}{2} (\tau+x_i)^2}) \text{erfc}\left(\frac{\tau+x_i}{\sqrt{2}}\right)\right) \text{erfc}\left(\frac{\tau-x_i}{\sqrt{2}}\right)+\exp({(\tau-x_i)^2+\frac{1}{2} (\tau+x_i)^2}) \text{erfc}\left(\frac{\tau-x_i}{\sqrt{2}}\right)^2-\exp({(\tau+x_i)^2}) \text{erfc}\left(\frac{\tau+x_i}{\sqrt{2}}\right) \left(\exp({\frac{1}{2} (\tau+x_i)^2}) \text{erfc}\left(\frac{\tau+x_i}{\sqrt{2}}\right)+2 \sqrt{\frac{2}{\pi }} \tau\right)\right)}{\left(\text{erfc}\left(\frac{\tau-x_i}{\sqrt{2}}\right)+\exp({2 \tau x_i}) \text{erfc}\left(\frac{\tau+x_i}{\sqrt{2}}\right)\right)^2},
    \]     
    \end{minipage}
    }
    
    which is negative for all $x_i < 0$ (the denominator is always positive and the numerator is negative when $\tau > 0$ and $x_i < 0$).

    Similarly, we claim that for any fixed $x_i > 0$,
    $\bbE[x_{nbr}](x_i; \tau)$ is decreasing w.r.t $\tau < 0$.
    When $\tau < 0$, $\bbE[x_{nbr}]$ has the opposite sign as $x_i$ and we only need to consider $x_i > 0$ since $\bbE[x_{nbr}] > 0$ when $x_i < 0$.
    We claim that for any fixed $x_i > 0$,
    $\bbE[x_{nbr}]$ is decreases as $\tau < 0$ decreases (i.e., $\tau$ moves from $0$ to $-\infty$), and thus
    $\bm{1}[\frac{d^+ \mu - d^- \mu}{d^+ + d^-} + \bbE[x_{nbr}] < 0]$ is non-decreasing for all $x_i$ values, which implies the increase in the BER.
    Indeed, the partial derivative is the same as above, where
    the denominator is always positive and the numerator is positive when $\tau < 0$ and $x_i > 0$.

    When node features have higher dimensions, 
    obtaining elegant closed-form equations as above would be challenging, but we still have the property that
    $\bbE[\bm{x}_{nbr}] = \bm{0}$ if and only if $\tau = 0$.    
    Moreover, $\bm{x}_{nbr}$ moves further from $\bm{0}$ as $\abs{\tau}$ increases, which increases the BER.
    Specifically, in the above reasoning, one needs to replace 
    $x > 0$ with $\Tilde{\bm{\mu}}^\top \bm{x} > 0$ with $\tilde{\bm{\mu}} 
    = \frac{d^+ \bm{\mu} - d^- \bm{\mu}}{d^+ - d^-}    
    = \frac{p^+ \bm{\mu} - p^- \bm{\mu}}{p^+ - p^-}$ (features that would be classified as the positive class, class-$1$), and similarly replace
    $x < 0$ with $\Tilde{\bm{\mu}}^\top \bm{x} < 0$.
\end{proof}

\clearpage

\section{In-Depth Analysis of Measure}\label{app:measure}

\subsection{\hdot Interpretation} \label{app:hdot_interpretation}
In this subsection, we discuss the details of \hdot interpretation. 
For high-level ideas, refer to Sec.~\ref{sec:measure_interpretation}.

{\textbf{\textit{Magnitude: node-level CFH}}}.
We first rephrase node-level CFH \hn:
\begin{align*}
    \mathbf{\Tilde{h}}^{(v)}_{i} 
    = \frac{\mathbf{h}^{(v)}_{i}} {\max(b(v_{i}) ~,~ \mathbf{d}(v_{i}, N_i))} 
    = \frac{b(v_{i}) - \mathbf{d}(v_{i}, N_i)} {\max(b(v_{i}) ~,~ \mathbf{d}(v_{i}, N_i))} 
    =
    \begin{cases}
        1 - \frac{\mathbf{d}(v_{i}, N_i)} {b(v_{i})} & \text{, if}~ \mathbf{\Tilde{h}}_{i}^{(v)} \geq 0\\
        \frac{b(v_{i})}{\mathbf{d}(v_{i}, N_i)} - 1 & \text{, otherwise}
    \end{cases}     
\end{align*}

For positive (or negative) \hn, the node $v_i$ has 
$\frac{ \vert {\mathbf{\Tilde{h}}^{(v)}_i} \vert }{1- \vert {\mathbf{\Tilde{h}}^{(v)}_i} \vert}$
times smaller (or larger) distance to neighbors ${\mathbf{d}(v_{i}, N_i)}$ than its homophily baseline $b(v_i)$.
For example, if ${\mathbf{d}(v_{i}, N_i)} = 1$ and $b(v_i) = 10$, then \hn $=0.9$, indicating that ${\mathbf{d}(v_{i}, N_i)}$ is 9 times smaller than $b(v_i)$.

{\textbf{\textit{Magnitude: graph-level CFH}}}.
We also rephrase graph-level CFH \hg:
\begin{align*}
    \mathbf{\Tilde{h}}^{(G)} 
    = \frac{\mathbf{h}^{(G)}} {\frac{1}{\abs{V}} \max(\sum_{v_i \in V} b(v_{i}) ~,~  \sum_{v_i \in V} \mathbf{d}(v_{i}, N_i))}
    = 
    \begin{cases}
        1 - \frac{\sum_{v_i \in V} \mathbf{d}(v_{i}, N_i)} {\sum_{v_i \in V} b(v_{i})}  
        & \text{, if}~ \mathbf{\Tilde{h}}^{(G)} \geq 0\\
        \frac{\sum_{v_i \in V} b(v_{i})}{\sum_{v_i \in V} \mathbf{d}(v_{i}, N_i)} - 1
        & \text{, otherwise}
    \end{cases}       
\end{align*}
Like in node-level interpretation, for positive (or negative) \hg, the graph $G$ has $\frac{\vert {\mathbf{\Tilde{h}}^{(G)}} \vert }{1- \vert {\mathbf{\Tilde{h}}^{(G)}} \vert}$
times smaller (or larger) mean distance to neighbors $\frac{1}{\abs{V}} \sum_{v_i \in V}{\mathbf{d}(v_{i}, N_i)}$ than the mean homophily baseline $\frac{1}{\abs{V}}\sum_{v_i \in V} b(v_i)$.
For example, if $\frac{1}{\abs{V}} \sum_{v_i \in V}{\mathbf{d}(v_{i}, N_i)} = 1$ and $\frac{1}{\abs{V}} \sum_{v_i \in V} b(v_i) = 10$, then \hg $=0.9$, indicating that $\frac{1}{\abs{V}} \sum_{v_i \in V}{\mathbf{d}(v_{i}, N_i)}$ is 9 times smaller than $\frac{1}{\abs{V}} \sum_{v_i \in V} b(v_i)$.

{\textbf{\textit{Zero}}}.
If node $v_i$'s feature is identical to all other nodes (i.e., $X_i = X_j,\forall v_j \in V$), $\mathbf{\Tilde{h}}^{(v)}_i = 0$,
because its $b(v_i) = \mathbf{d}(v_i, N_i) = 0$.
\footnote{Recall that we define \hn and \hg to be 0, if $b(\cdot)=0$.}
A fully connected node $v_i$ has $\mathbf{\Tilde{h}}^{(v)}_i = 0$, 
because its $b(v_i) = \mathbf{d}(v_i, N_i)$.
For the same reason, a graph $G$ has $\mathbf{\Tilde{h}}^{(G)} = 0$ if 
(\textit{i}) it is fully connected and/or
(\textit{ii}) has all identical node features.

If a node $v_i$ chooses the non-zero number of neighbors by a random probability, 
$\mathbb{E}[\mathbf{\Tilde{h}}^{(v)}_i] = 0$.
For the same reason, a graph $G$ has $\mathbb{E}[\mathbf{\Tilde{h}}^{(G)}] = 0$ if 
each node $v_i \in V$ chooses a non-zero number of neighbors by a random probability.

It is important to note that there are many other conditions in which \hn and \hg become 0.
That is, while \hn and \hg being 0 may suggest no \ax, 
they are not conclusive.
In-depth analysis of the microscopic patterns, such as distributions of \hn and \hp, may better elucidate the levels of \ax.

%%%%%%%%%%%%%%%%%%%%%%%%%%%%%%%%%%%%%%%%%%%%%%%%%%%%%%%%%%%%%%%%%%%%%%%%%%%%%%%%%%%%%%%%%%%

\subsection{On Class Control} \label{app:class-control}

{\textbf{\textit{Connection to Part and Partial Correlation}}}.
The class control mechanism in Eq.~\eqref{eq.class_control} is analogous to the variable control method of part and partial correlation.
We focus on part correlation here.

The goal of its variable control is to control the effect of the third variable when analyzing the correlation between two variables.
Let $X^{(P)}, A^{(P)} \in \mathbb{R}^N$ be two variables of interest and $Y^{(P)} \in \mathbb{R}^{N \times d}$ be the third variable, where $N$ is the number of observed samples.
\begin{align}
    \beta^{*} &= \argmin_\beta \lVert (X^{(P)} - Y^{(P)}\beta)\rVert^{2}_{2} \label{eq:optimization_goal}\\
    X \vert Y &= X^{(P)} - (Y^{(P)}\beta^{*}) \label{eq:optimization_result},
\end{align}
where $\beta \in \mathbb{R}^{d}$ is a regression coefficient.
Geometric interpretation of this mechanism is the projection of original $X^{(P)}$ onto the orthogonal space of $Y^{(P)}$ with the least approximation L2-error, expecting that the information of $X^{(P)}$ is maximally maintained given the removal of $Y^{(P)}$ intervention.
In part correlation, correlation is measured between $X \vert Y$ and $A$.

{Now, we show how Eq.~\eqref{eq.class_control} relates to the above equation.}
Let $Y^{(P)} \in \{0,1\}^{\vert V\vert \times c}$ be the one-hot labeled class matrix for each node (i.e., $Y^{(P)}_{ij} = 1$ for $y_{i} = j, \forall v_{i}\in V$, 0 otherwise).
Let $X^{(P)} \in \mathbb{R}^{\vert V\vert}$ be the original node feature.
Now, in this analysis, we let $X^{(P)} \coloneqq X$ and $Y^{(P)} \coloneqq Y$ for notational simplicity.
We optimize Eq.~\eqref{eq:optimization_goal} as below:
\begin{align}
    \beta^{*} &= \argmin_\beta \lVert X - Y\beta\rVert^{2}_{2} = \argmin_\beta (X - Y\beta)^{T}(X - Y\beta) \coloneqq \argmin_\beta \mathcal{L} \\ 
    \frac{\partial \mathcal{L}}{\partial \beta} &= -2Y^{T}X + 2(Y^{T}Y)\beta = 0 \equiv \beta^{*} = (Y^{T}Y)^{-1}Y^{T}X. \label{eq:optimal_derive}
\end{align}
As we take a closer look at the form of $\beta^{*}$ in Eq.~\eqref{eq:optimal_derive}: 
\begin{itemize}
    \item $Y^{T}Y \in \mathbb{R}^{c\times c}$ is a diagonal matrix where each $i-$th diagonal entry indicates the number of nodes belonging to the class $i$ (i.e., $(Y^{T}Y)_{ii} = \vert C^{+}_{i}\vert, \forall i \in [c]$).
    \item $Y^{T}X \in \mathbb{R}^{c}$ is a vector where $j-$th entry indicates the sum of node features that belong to the class $i$ (i.e., $(Y^{T}X)_{i} = \sum_{v_{k} \in C^{+}_{i}}X_{k}, \forall i \in [c]$).
\end{itemize}
Thus, $\beta^{*} \in \mathbb{R}^{c}$ is a vector where $k-$ entry indicates the mean of node features that belong to the class $i$. 
In the given setting, by applying obtained $\beta^{*}$, Eq.~\eqref{eq:optimization_result} is equivalent to $(X\vert Y)_{i} = X_{i} - \frac{1}{\vert C^{+}_{y_{i}}\vert}\sum_{v_{k} \in C^{+}_{y_{i}}}X_{k}$, which is equal to Eq.~\eqref{eq.class_control}.
Therefore, we conclude that Eq.~\eqref{eq.class_control} is a special case of the variable control method of part correlation.

%%%%%%%%%%%%%%%%%%%%%%%%%%%%%%%%%%%%%%%%%%%%%%%%%%%%%%%%%%%%%%%%%%%%%%%%%%%%%%%%%%%%%%%%%%%

\subsection{Generalizing \hdot}
CFH \hdot measures \ax, while controlling for potential confounding by node class.
However, with its good properties, 
we can generalize it to measure topology-feature and topology-class dependence \textit{without} confound control. 

{\textbf{\textit{Generalized distance function}}}.
Denote the distance function (Eq.~\eqref{eq.distance_func}) with a matrix $\mathbf{X} \in \mathbb{R}^{n \times k}$ as 
\begin{align}\label{eq.distance_func_generalized}
    \mathbf{d}^{*}(v_{i}, V'_{i}, \mathbf{X}) \coloneqq \frac{1}{\vert V'_i\vert}\sum_{v_{j} \in V'_i}\lVert \mathbf{X}_{i} - \mathbf{X}_{j} \rVert_{2}.    
    \vspace{-3mm}
\end{align}
Eq.~\eqref{eq.distance_func} is a special case of Eq.~\eqref{eq.distance_func_generalized}, where $\mathbf{X} = X \vert Y$.
Likewise, we generalize homophily baseline as $b^{*}(v_i) = \mathbf{d}^{*}(v_i, V'_i, \mathbf{X})$. 

{\textbf{\textit{Generalized homophily measure}}}.
Based on $\mathbf{d}^{*}(\cdot)$, we propose a generalized homophily measure $\Tilde{H}(\cdot)$.

\begin{enumerate}[start=1,label={\bfseries G\arabic*)}]
    \item Generalized node pair-level homophily ${H}^{(p)}_{ij}$: 
    \begin{equation}\label{eq.edge_homophily_generalized}
        H^{(p)}_{ij}(\mathbf{X}) = H((v_{i}, v_{j}) ~\vert~ E, \mathbf{X}) \coloneqq b^{*}(v_{i}) - \mathbf{d}^{*}(v_{i}, \{v_{j}\}, \mathbf{X}) 
    \end{equation}
    \vspace{-3mm}
    
    \item Generalized node-level homophily ${H}^{(v)}_{i}$: 
    \begin{equation}\label{eq.node_homophily_generalized}
        H^{(v)}_{i}(\mathbf{X}) = H(v_{i}~\vert~E, \mathbf{X}) \coloneqq \frac{1}{\vert N_i \vert} \sum_{v_{j} \in N_i} H^{(p)}_{ij}
    \end{equation}
    \vspace{-3mm}

    \item Generalized graph-level homophily ${H}^{(G)}$: 
    \begin{equation}\label{eq.graph_homophily_generalized}
        H^{(G)}(\mathbf{X}) = H(G~\vert~E, \mathbf{X}) \coloneqq \frac{1}{\vert V\vert} \sum_{v_{j} \in V} H^{(v)}_{j}
    \end{equation}
    \vspace{-3mm}
    
    \item Generalized node-level normalization: 
    \begin{equation}\label{eq.node_norm_generalized}
        \Tilde{H}^{(v)}_{i}(\mathbf{X}) = 
        \frac{H^{(v)}_{i}}
        {\max(b^{*}(v_{i}) , \mathbf{d}^{*}(v_{i}, N_i, \mathbf{X}))}.
    \end{equation}
    \vspace{-3mm}
    
    \item Generalized graph-level normalization:
    \footnote{
        For completeness, if $b^{*}(\cdot) = 0$, we let $\Tilde{H}^{(v)}_{i}(\cdot), \Tilde{H}^{(G)}(\cdot)=0$.
    }
    \begin{equation}\label{eq.graph_norm_generalized}
        \Tilde{H}^{(G)}(\mathbf{X}) = 
        \frac{H^{(G)}}
        {\frac{1}{\abs{V}} 
        \max(\sum_{v_i \in V} b^{*}(v_{i}),  \sum_{v_i \in V} \mathbf{d}^{*}(v_{i}, N_i, \mathbf{X}))}.
    \end{equation} 
    \vspace{-3mm}
\end{enumerate}

CFH measure \hdot is a special case of the proposed generalized homophily measure \Hdot.
With the generalized homophily measure, we can measure feature homophily $\Tilde{H}^{(G)}(X)$ and class homophily $\Tilde{H}^{(G)}(Y)$, where $Y \in \mathbb{R}^{n \times c}$ is a node class matrix.

%%%%%%%%%%%%%%%%%%%%%%%%%%%%%%%%%%%%%%%%%%%%%%%%%%%%%%%%%%%%%%%%%%%%%%%%%%%%%%%%%%%%%%%%%%%

\clearpage

\begin{table*}[t]
\begin{center}
\caption{Comparison of Graph-Level CFH Scores with Different Features} \label{tab:CFH_score_comparison}
    \resizebox{\textwidth}{!}{
    \renewcommand{\arraystretch}{1.0}
        \centering
        \begin{tabular}{*{13}{c}}
            \toprule            
            \textbf{Dataset} & Cora & CiteSeer & PubMed & Cora-ML & Cora-Full & DBLP & Wiki-CS & CS & Physics & Photo & Computers & Ogbn-ArXiv \\
            \midrule
            $X^{(orig)}$ & 0.0562 & 0.0802 & 0.1072 & 0.0390 & 0.0388 & 0.0786 & 0.2182 & -0.0042 & 0.0760 & -0.0150 & -0.0207 & 0.0755 \\
            $X^{(rand)}$ & 0.0204 & 0.0151 & -0.0061 & 0.0036 & 0.0053 & 0.0031 & -0.0018 & 0.0128 & 0.0031 & 0.0131 & 0.0067 & 0.0036  \\
            $X^{(conv)}(1)$ & 0.4612 & 0.5706 & 0.4399 & 0.4217 & 0.3991 & 0.4605 & 0.4442 & 0.3991 & 0.4603 & 0.3421 & 0.3278 & 0.3855  \\
            $X^{(conv)}(2)$ & 0.6238 & 0.7177 & 0.6095 & 0.5956 & 0.5621 & 0.6227 & 0.4956 & 0.5326 & 0.5835 & 0.5214 & 0.4878 & 0.4768  \\
            $X^{(conv)}(4)$ & 0.7221 & 0.8003 & 0.7181 & 0.7017 & 0.6513 & 0.7271 & 0.5120 & 0.5848 & 0.6479 & 0.6515 & 0.5995 & 0.5242  \\

            \midrule
            \multicolumn{13}{c}{}\\[0.3em]
            \midrule
            
            \textbf{Dataset} & Chameleon & Squirrel & Actor & Texas & Cornell & Wisconsin & RM-Emp & AMZ-Rts & Tolokers & Penn94 & Flickr & ArXiv-Year \\
            \midrule
            $X^{(orig)}$ & -0.0714 & -0.0538 & -0.0199 & -0.0803 & 0.0041 & -0.0324 & 0.0199 & 0.1266 & 0.1296 & 0.0870 & 0.0018 & 0.1206 \\
            $X^{(rand)}$ & -0.0368 & 0.0136 & 0.0060 & -0.0398 & 0.0390 & 0.0165 & 0.0020 & -0.0003 & -0.0131 & -0.0013 & -0.0023 & 0.0037 \\
            $X^{(conv)}(1)$  & 0.3835 & 0.3529 & 0.2745 & 0.2279 & 0.2121 & 0.2590 & 0.4165 & 0.5893 & 0.5162 & O.O.M. & 0.2052 & 0.4709 \\
            $X^{(conv)}(2)$  & 0.5299 & 0.5316 & 0.3983 & 0.3111 & 0.3017 & 0.3584 & 0.5899 & 0.7667 & 0.6524 & O.O.M. & 0.2097 & 0.5784 \\
            $X^{(conv)}(4)$  & 0.6404 & 0.6227 & 0.4974 & 0.3704 & 0.3476 & 0.4237 & 0.6852 & 0.8563 & 0.6881 & O.O.M. & 0.3185 & 0.6386 \\

            \bottomrule
        \end{tabular}
        }
\end{center}
\begin{tablenotes}
    \footnotesize
    \item ($\ast$) $X^{(orig)}$ denotes the original node features. RM-Emp stands for Roman-Empire, and AMZ-Rts stands for Amazon-Ratings. O.O.M. denotes out-of-memory.
\end{tablenotes}
\vspace{3mm}
\end{table*}

\section{In-Depth Analysis of the Benchmark Datasets}\label{app:patterns}
We further analyze the benchmark datasets.
Specifically, we buttress Observations~\ref{obs:absolte_homophily}-\ref{obs:shuffle_homophily} with additional results. 
We also briefly discuss the Roman-Empire dataset, delving into why GNN performance degrades consistently over the feature shuffles.

\subsection{Observation 1: The Full Results} \label{app:obs1}

We focus on the lowness of CFH \hdot in the benchmark graphs. Specifically, we further support Observation~\ref{obs:absolte_homophily} with 
(\textit{i}) comparison of CFH scores with different features and 
(\textit{ii}) node-level analysis.

{\textbf{\textit{Comparison to different features}}}.
First, we investigate how low the CFH \hdot scores are for the benchmark graphs, compared to different features.
Recall that their mean $\abs{\mathbf{\Tilde{h}}^{(G)}} = 0.06$.
For comparison, we consider two other node features.
\begin{itemize}
    \item Random Baseline: Random node features $X_i^{(rand)} \sim \mathcal{N}(0,1), \forall v_i \in V.$
    \item Homophilic Baseline: Convoluted node features $X^{(conv)}(l) = ((D+I)^{-1}(A+I))^{l}X, $
\end{itemize}
where $I \in \mathbb{R}^{n \times n}$ is an identity matrix and $l \in \set{1,2,4}$.

In Table~\ref{tab:CFH_score_comparison}, we report the \hg for each feature type and dataset.
Averaging the scores for all 24 datasets, 
$X^{(conv)}(4)$ has the mean \hg score of 0.61, 
while $X^{(rand)}$ has the mean near 0. 
The mean \hg of the original node feature $X^{(orig)}$ is much closer to that of the random features $X^{(rand)}$, further supporting Observation~\ref{obs:absolte_homophily}.

{\textbf{\textit{Node-level analysis}}}.
Second, we report that node-level CFH \hn scores also tend to be positive and low.
As shown in Fig.~\ref{fig:hn_over_shuffle(full)}(a), most nodes in most graphs have $\vert \mathbf{h}^{(v)}_i \vert < 0.3$.
As observed at the graph level, each node's distance to the neighbors is close to its homophily baseline.

\textbf{\textit{Conclusion}}.
From the analyses, we conclude again that CFH \hdot are generally positive and low in the benchmark datasets. 

\begin{figure}[t] 
    \centering
    
    \begin{subfigure}[b]{\linewidth}
    \centering
    \includegraphics[width=\textwidth]{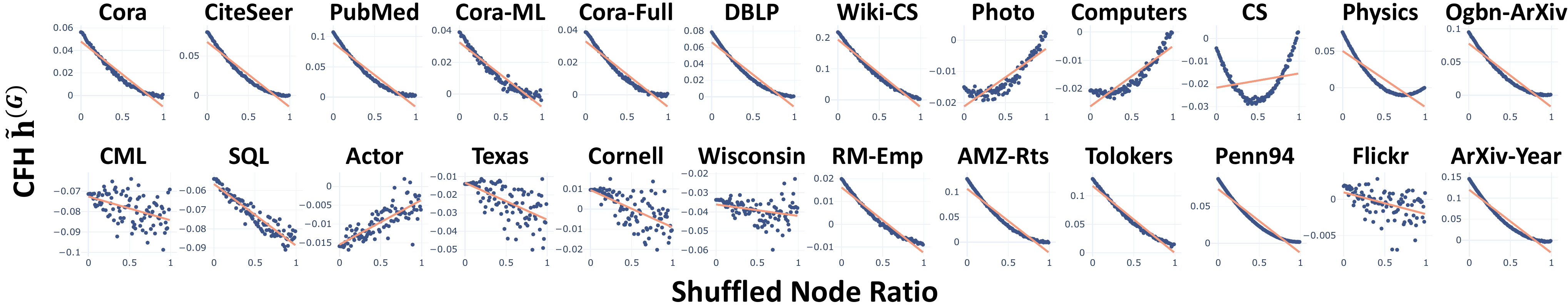}
    \caption{Graph-level CFH score \hg's over the \textit{class-wise} feature shuffles.}
    \end{subfigure}

    \vspace{2mm}
    
    \begin{subfigure}[b]{\linewidth}
    \centering
    \includegraphics[width=\textwidth]{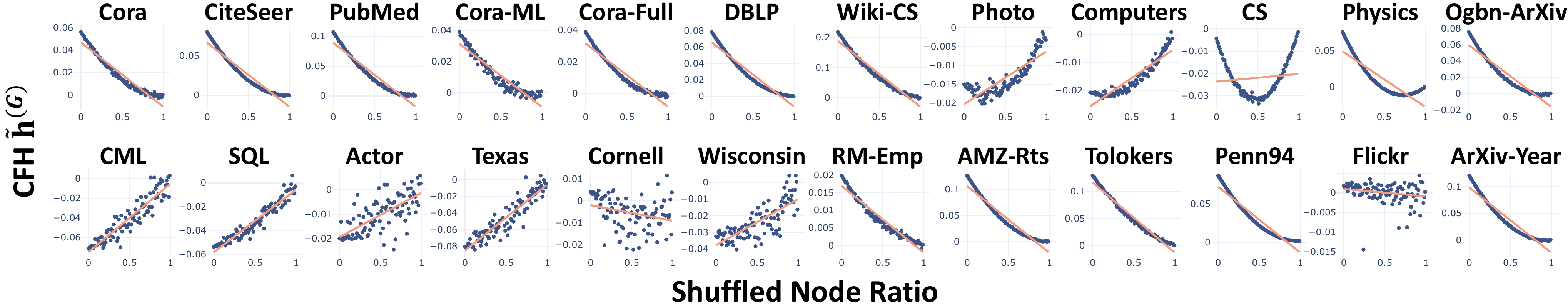}
    \caption{Graph-level CFH score \hg's over the \textit{non-class-wise} feature shuffles.}
    \end{subfigure}

    \caption{\label{fig:hg_over_shuffle(full)}
    \bolden{Graph-Level CFH Statistics in Real-World Graphs and their Relationship to the Feature Shuffle.}
    }
    \vspace{2mm} 
\end{figure}

\begin{figure}[t] 
    \centering
    
    \begin{subfigure}[b]{\linewidth}
    \centering
    \includegraphics[width=\textwidth]{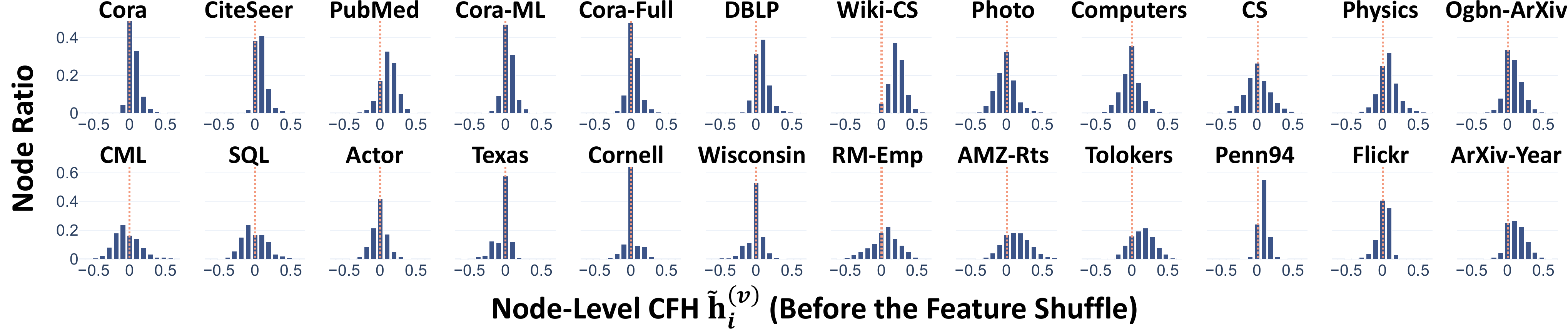}
    \caption{Histogram of node-level CFH score \hn's before the feature shuffle.}
    \end{subfigure}

    \vspace{2mm}

    \begin{subfigure}[b]{\linewidth}
    \centering
    \includegraphics[width=\textwidth]{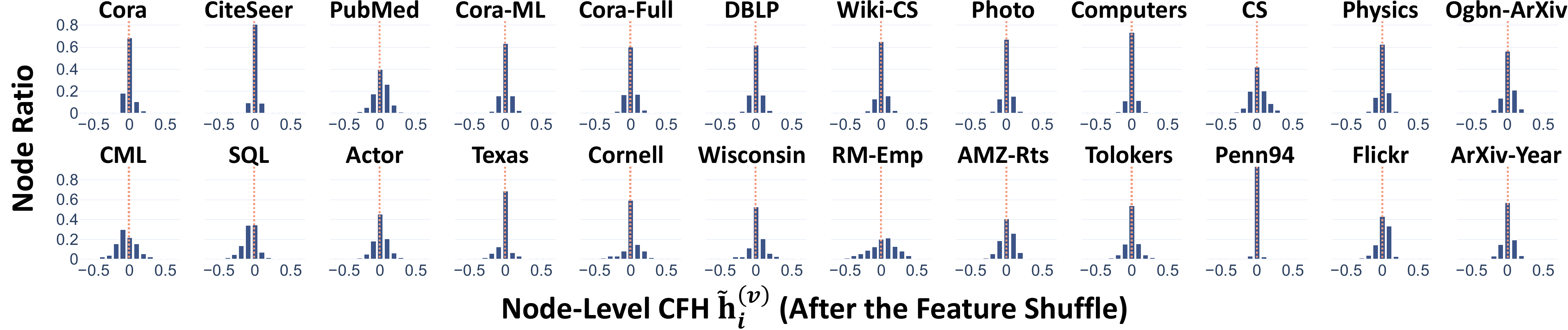}
    \caption{Histogram of node-level CFH score \hn's after the feature shuffle.}
    \end{subfigure}

    \vspace{2mm}

    \begin{subfigure}[b]{\linewidth}
    \centering
    \includegraphics[width=\textwidth]{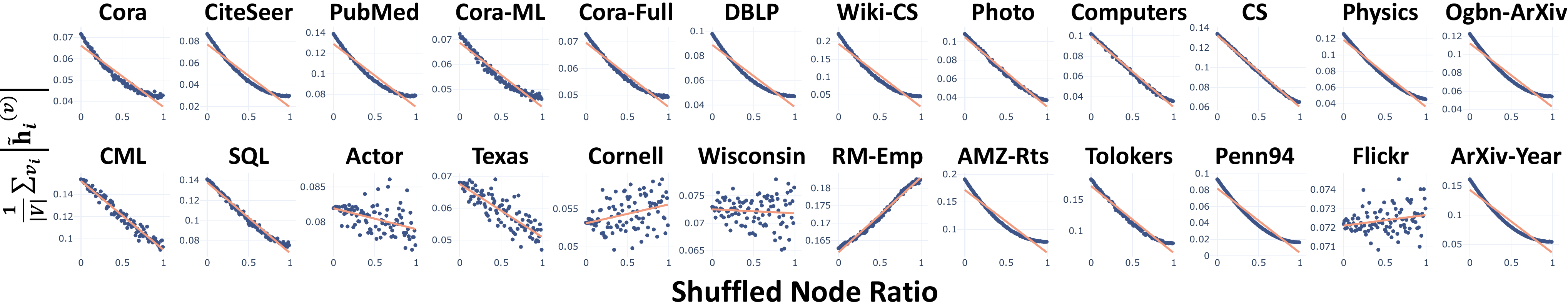}
    \caption{The mean node-level CFH magnitude $\vert {\mathbf{h}^{(v)}_i} \vert$'s over the feature shuffles.}
    \end{subfigure}

    \caption{\label{fig:hn_over_shuffle(full)}
    \bolden{Node-Level CFH Statistics in Real-World Graphs and Their Relationship to the Feature Shuffle.}
    }

    \vspace{2mm} 
\end{figure}

%%%%%%%%%%%%%%%%%%%%%%%%%%%%%%%%%%%%%%%%%%%%%%%%%%%%%%%%%%%%%%%%%%%%%%%%%%%%%%%%%%%%%%%%%%%%%%%

%%%%%%%%%%%%%%%%%%%%%%%%%%%%%%%%%%%%%%%%%%%%%%%%%%%%%%%%%%%%%%%%%%%%%%%%%%%%%%%%%%%%%%%%%%%%%%%

\subsection{Observation 2: The Full Results} \label{app:obs2}

We further demonstrate that CFH and class-homophily have a small, positive correlation with two additional evidence.
We (\textit{i}) measure correlations between CFH and the other class-homophily measures and 
(\textit{ii}) conduct node-level correlation analysis.

\textbf{\textit{Other class-homophily measures}}.
First, we complement Observation~\ref{obs:relative_homophily} by measuring correlations between CFH and different measures of class-homophily, defined by~\citet{Pei2020Geom-gcn:Networks} and~\citet{Zhu2020BeyondDesigns}. 
Class-homophily defined by~\citet{Zhu2020BeyondDesigns} and CFH have correlation coefficients of 0.403 (Pearson's $r$) and 0.196 (Kendall's $\tau$).
Class-homophily defined by~\citet{Pei2020Geom-gcn:Networks} and CFH have correlation coefficients of 0.401 (Pearson's $r$) and 0.225 (Kendall's $\tau$).
While we find slightly stronger correlations between CFH and the other measures, the correlations are not consistently strong, such that there exist non-negligible gaps between Pearson's $r$ and Kendall's $\tau$.
We, thus, do not find counter-evidence for Observation~\ref{obs:relative_homophily}.

\textbf{\textit{Node-level analysis}}.
Now, we complement Observation~\ref{obs:relative_homophily} with node-level analysis.
Specifically, we analyze the correlation between node-level CFH \hn and node-level class-homophily ${\Tilde{H}}^{(v)}_i(Y)$ (Eq.~\eqref{eq.node_norm_generalized}).
\footnote{We use ${\Tilde{H}}^{(v)}_i(Y)$ as the node-level class-homophily measure because \hc has not been defined at node-level.}
We find that the Pearson correlations in most of the 24 graphs are very low, such that the absolute values of their Pearson's $r$ scores are below 0.2 in 22/24 graphs (Table~\ref{tab:data_stat}).
Also, 19/24 have positive Pearson correlations.

\textbf{\textit{Conclusion}}.
From the analyses, we conclude again that CFH \hdot has a small, positive correlation with class homophily.

%%%%%%%%%%%%%%%%%%%%%%%%%%%%%%%%%%%%%%%%%%%%%%%%%%%%%%%%%%%%%%%%%%%%%%%%%%%%%%%%%%%%%%%%%%%%%%%

%%%%%%%%%%%%%%%%%%%%%%%%%%%%%%%%%%%%%%%%%%%%%%%%%%%%%%%%%%%%%%%%%%%%%%%%%%%%%%%%%%%%%%%%%%%%%%%

\subsection{Observation 3: The Full Results} \label{app:obs3}
Here, we report how the feature shuffle affects CFH in all 24 benchmark datasets (Figs.~\ref{fig:hg_over_shuffle(full)}-\ref{fig:hn_over_shuffle(full)}).
While most datasets follow the pattern reported in Observation~\ref{obs:shuffle_homophily}, 
few of them (Chameleon, Squirrel, Texas, Wisconsin, and Cornell) do not fully obey it.
Specifically, their graph-level CFH \hg score moves away from 0 over increasing shuffled node ratio.
We reason their deviation by answering two questions: 
(\textit{i}) Why do the \hg scores become larger after the feature shuffle?;
(\textit{ii}) Why do the \hg scores not approach 0 after the feature shuffle?

\textbf{\textit{Answer to question (\textit{i})}}.
Node-level CFH \hn distributions before and after the feature shuffle (Fig.~\ref{fig:hn_over_shuffle(full)}) reveals that the mean $\vert \mathbf{\Tilde{h}}^{(v)}_i \vert$ decreases after the feature shuffle in all five datasets.
The finding indicates that the magnitude in which the distance to neighbors (i.e., $\mathbf{d}(v_i, N_i)$) deviates from the homophily baseline (i.e., $b(v_i)$) becomes smaller after the feature shuffle.
In short, the finding demonstrates (\textit{i}) that \ax is perturbed after the feature shuffle and (\textit{ii}) an in-depth analysis of \hdot is necessary to reveal the pattern.
The graph-level CFH \hg does not fully capture the subtlety as it mean-aggregates the positive and negative \hn scores.

\textbf{\textit{Answer to question (\textit{ii})}}.
The \hg scores may not approach 0 due to the imperfect class-control.
We evidence our claim with \textit{non-class-wise} feature shuffle, which means that the feature vectors of all nodes, irrespective of their class membership, are shuffled together.
After \textit{non-class-wise} feature shuffle, we find that the graph-level CFH \hg's approach 0 in 23/24 datasets (Fig.~\ref{fig:hg_over_shuffle(full)}(b)).
The finding suggests that feature distribution difference between node classes hinders CFH \hg from approaching 0.
An advanced class-control method may mitigate such a problem, and we leave it up to future studies.

\textbf{\textit{Conclusion}}. 
The series of analyses underscore the complexity of quantifying \ax.
We claim that while the feature shuffle effectively perturbs \ax, \hg may not approach 0 due to 
(\textit{i}) node-level discrepancies
and (\textit{ii}) the complex nature of feature distribution.
Therefore, we further argue that a few datasets' deviations from Observation~\ref{obs:shuffle_homophily} do not undermine the integrity of our conclusion that \ax mediates the effect of graph convolution.

%%%%%%%%%%%%%%%%%%%%%%%%%%%%%%%%%%%%%%%%%%%%%%%%%%%%%%%%%%%%%%%%%%%%%%%%%%%%%%%%%%%%%%%%%%%%%%%

\subsection{The Roman-Empire Dataset} \label{app:roman_empire}
The Roman-Empire dataset has an unusual, chain-like graph topology. 
Its number of nodes is 22,662, with a diameter of 6,824.
In short, there is no small world effect observed, 
making the effect of the feature shuffle different from the rest of the datasets.
A node-level analysis reveals its unique patterns of CFH \hdot over the feature shuffle.
In Fig.~\ref{fig:hn_over_shuffle(full)}(a,b), we \textit{uniquely} observe that its histograms of \hn before and after the feature shuffle are highly similar.
In Fig.~\ref{fig:hn_over_shuffle(full)}(c), we further \textit{uniquely} observe that its mean $\vert \mathbf{\Tilde{h}}^{(v)}_i \vert$ increase significantly (2\%p) after the feature shuffle.
The qualitative and quantitative uniqueness of the Roman-Empire dataset may have contributed to the degrading GNN performance over the feature shuffles.

%%%%%%%%%%%%%%%%%%%%%%%%%%%%%%%%%%%%%%%%%%%%%%%%%%%%%%%%%%%%%%%%%%%%%%%%%%%%%%%%%%%%%%%%%%%%%%%
\begin{table*}[t]
\begin{center}
\caption{Statistics of the Benchmark Datasets} \label{tab:data_stat}
    \resizebox{\textwidth}{!}{
    \renewcommand{\arraystretch}{0.9}
        \centering
        \begin{tabular}{*{13}{c}}
            \toprule            
            \textbf{Dataset} & Cora & CiteSeer & PubMed & Cora-ML & Cora-Full & DBLP & Wiki-CS & CS & Physics & Photo & Computers & Ogbn-ArXiv \\
            \midrule
            \textbf{\# Nodes} & 2,708 & 3,327 & 19,717 & 2,995 & 19,793 & 17,716 & 11,701 & 18,333 & 34,393 & 7,650 & 13,752 & 169,343\\
            \textbf{\# Edges} & 10,556 & 9,104 & 88,648 & 16,316 & 126,842 & 105,734 & 431,726 & 163,788 & 495,924 & 238,162 & 491,722 & 1,166,243 \\
            \textbf{\# Features} & 1,433 & 3,704 & 500 & 2,879 & 8,710 & 1,639 & 300 & 6,805 & 8,415 & 745 & 767 & 128 \\
            \textbf{\# Class} & 7 & 6 & 3 & 7 & 70 & 4 & 10 & 15 & 5 & 8 & 10 & 40 \\
            % \textbf{Node Homophily} & 0.8252 & 0.7062 & 0.7924 & 0.8098 & 0.5861  & 0.8128 & 0.6588 & 0.8320 & 0.9154 & 0.8365 & 0.7853 & 0.4280 \\
            % \textbf{Edge Homophily} &  0.8100 & 0.7355 & 0.8024 & 0.7886  & 0.5670  & 0.8279 & 0.6547 & 0.8081 & 0.9314 & 0.8272 & 0.7772 & 0.6551 \\
            \textbf{Class Homophily \hc} &  0.7657 & 0.6267 & 0.6641 & 0.7401 & 0.4959  & 0.6522 & 0.5681 & 0.7547 & 0.8474 & 0.7722 & 0.7002 & 0.4445 \\
            \textbf{CFH \hg} & 0.0562 & 0.0802 & 0.1072 & 0.0390 & 0.0388 & 0.0786 & 0.2182 & -0.0042 & 0.0760 & -0.0150 & -0.0207 & 0.0755 \\
            \textbf{Pearson(\hn, $\Tilde{H}^{(v)}_i(Y)$)} & 0.1158 & 0.1907 & 0.0660 & 0.1602 & 0.1090 & 0.1015 & 0.2993 & 0.1938 & 0.1344 & 0.1332 & 0.0287 &  0.2535\\

            \midrule
            \multicolumn{13}{c}{}\\[0.3em]
            \midrule
            
            \textbf{Dataset} & Chameleon & Squirrel & Actor & Texas & Cornell & Wisconsin & RM-Emp & AMZ-Rts & Tolokers & Penn94 & Flickr & ArXiv-Year \\
            \midrule
            \textbf{\# Nodes} & 890 & 2,223 & 7,600 & 183 & 183 & 251 & 22,662 & 24,292 & 11,758 & 41,554 & 89,250 & 169,343  \\
            \textbf{\# Edges} & 17,708 & 93,996 & 30,019 & 325 & 298 & 515 & 65,854 & 186,100 & 1,038,000 & 2,724,458 & 899,756 & 1,166,243  \\
            \textbf{\# Features} & 2,325 & 2,089 & 932 & 1,703 & 1,703 & 1,703 & 300 & 300 & 10 & 4814 & 500 & 128 \\
            \textbf{\# Class} & 5 & 5 & 5 & 5 & 5 & 5 & 18 & 5 & 2 & 2 & 7 & 5  \\
            \textbf{Class Homophily \hc} & 0.0444 & 0.0398 & 0.0061 & 0.0000 & 0.1504 & 0.0839 & 0.0208 & 0.1266 & 0.1801 & 0.0460 & 0.0698 & 0.1910  \\
            \textbf{CFH \hg} & -0.0714 & -0.0538 & -0.0199 & -0.0803 & 0.0041 & -0.0324 & 0.0199 & 0.1266 & 0.1296 & 0.0870 & 0.0018 & 0.1206 \\
            \textbf{Pearson(\hn, $\Tilde{H}^{(v)}_i(Y)$)} & 0.1390 & -0.0759 & -0.0272 & 0.0178 &-0.1718 &0.1539 &0.1308 &-0.0697 &-0.0715 &0.1523 &0.0217 &0.1721 \\
            
            \bottomrule
        \end{tabular}
        }
\end{center}

\footnotesize

($\ast$) For undirected graphs, their edges are counted as two directed edges.    

($\ast \ast$) $\Tilde{H}^{(v)}_i(Y)$ is a node-level class-homophily measure, defined in Eq.~\eqref{eq.graph_norm_generalized} of Appendix~\ref{app:measure}.
\vspace{1.5mm}

\end{table*}

\subsection{Dataset Description} \label{app:data_stat}
We provide a comprehensive description of the benchmark datasets, with their statistics in Table~\ref{fig:data_overall_statistics}.
\begin{itemize}
    \item The \textit{Cora}, \textit{CiteSeer}, \textit{PubMed}, \textit{Cora-ML}, \textit{Cora-Full}, \textit{DBLP}, and \textit{Ogbn-ArXiv}~\cite{planetoid, citataion_full, ogb_graphs} 
    datasets are citation networks.
    Each node represents a document, and two nodes are adjacent if a citation exists between the two corresponding articles.
    For each node, the features are the text features of the corresponding article,
    and the node class is the category of the research/subject domain of the document.    

    \item The \textit{Wiki-CS} 
    dataset is a webpage network of Wikipedia~\cite{Mernyei2020Wiki-cs:Networks}.
    Each node represents a Wikipedia webpage,
    and two nodes are adjacent if a hyperlink exists between the two webpages. 
    For each node, the features are the GloVe word embeddings of the webpage, 
    and the node class represents the article category of the webpage.

    \item The \textit{Computer} and \textit{Photo} 
    datasets are Amazon co-purchase networks~\cite{coauthor_copurchase}.
    Each node represents a product, and two nodes are adjacent if the two products are frequently purchased together. 
    For each node, the features are the bag-of-words of its customer reviews,
    and the node class is the product category.

    \item The \textit{CS} and \textit{Physics} 
    datasets are coauthor networks~\cite{coauthor_copurchase}.
    Each node represents an author, and two nodes are adjacent if the two corresponding authors have coauthored a paper together. 
    For each node, the features are the author's paper keywords,
    and the class is the most active field of the author's study.

    \item The \textit{Chameleon} and \textit{Squirrel} 
    datasets are webpage networks of Wikipedia~\cite{Pei2020Geom-gcn:Networks}.
    Each node represents a webpage on Wikipedia, and two nodes are adjacent if mutual links exist between the two corresponding web pages.
    For each node, the features are informative nouns on the corresponding webpage,
    and the node class represents the category of the average monthly traffic of the corresponding webpage.
    We use the version of the datasets provided by~\citet{Platonov2023AProgress}, which has filtered the possible duplicate nodes.

    \item The \textit{Actor} 
    dataset is the actor-only induced subgraph of a film-director-actor-writer network obtained from Wikipedia webpages~\cite{actor_graph, Pei2020Geom-gcn:Networks}.
    Each node represents an actor, and two nodes are adjacent if the two actors appear on the same Wikipedia webpage.
    For each node, the features are derived from the keywords on the Wikipedia webpage of the corresponding actor,
    and the node class is determined by the words on the webpage.
    
    \item The \textit{Texas}, \textit{Cornell}, and \textit{Wisconsin} 
    datasets are extracted from the WebKB dataset~\cite{Pei2020Geom-gcn:Networks}.
    Each node represents a webpage, and two nodes are adjacent if a hyperlink between the two webpages.
    For each node, the features are the bag-of-words features of the corresponding webpage,
    and the node class is the category of the webpage.    

    \item The \textit{Roman-Empire} 
    dataset is a network of texts in a Wikipedia article~\cite{Platonov2023AProgress}.
    Each node represents each word in the article, 
    and two nodes are adjacent if the words follow each other in the text or if one word depends on the other.
    For each node, the features is word embedding of the text,
    and the node class is the syntactic role of the text. 

    \item The \textit{Amazon-Ratings} 
    dataset is a co-purchase network of Amazon products~\cite{Platonov2023AProgress}.
    Each node represents a product, 
    and two nodes are adjacent if they are frequently purchased together.
    For each node, the features are text embedding of the product description,
    and the node class is its ratings. 

    \item The \textit{Tolokers} 
    dataset is an online social network of Toloka crowdsourcing platform~\cite{Platonov2023AProgress}.
    Each node represents a worker, 
    and two nodes are adjacent if they have worked on the same task.
    For each node, the features consist of the worker profile and task performance,
    and the node class is whether or not the worker has been banned. 

    \item The \textit{Penn94} 
    dataset is a social network on Facebook~\cite{Lim2021LargeMethods}.
    Each node represents a user, 
    and two nodes are adjacent if they are friends.
    For each node, the features are the user profile,
    and the node class is the reported gender. 

    \item The \textit{Flickr} 
    dataset is a network of images to Flickr website~\cite{flickr}.
    Each node represents an image, 
    and two nodes are adjacent if the two share some common properties (e.g. the same geographic location). 
    For each node, the features are bag-of-word representations of the image,
    and the class is the image's tag.
    
    \item The \textit{ArXiv-Year} 
    dataset~\cite{Lim2021LargeMethods} is a version of the \textit{Ogbn-ArXiv} dataset,
    where the original node class is replaced with the article publication year.
\end{itemize}

\clearpage
\begin{figure*}[t]
\begin{center}
    
    \includegraphics[width=\textwidth]{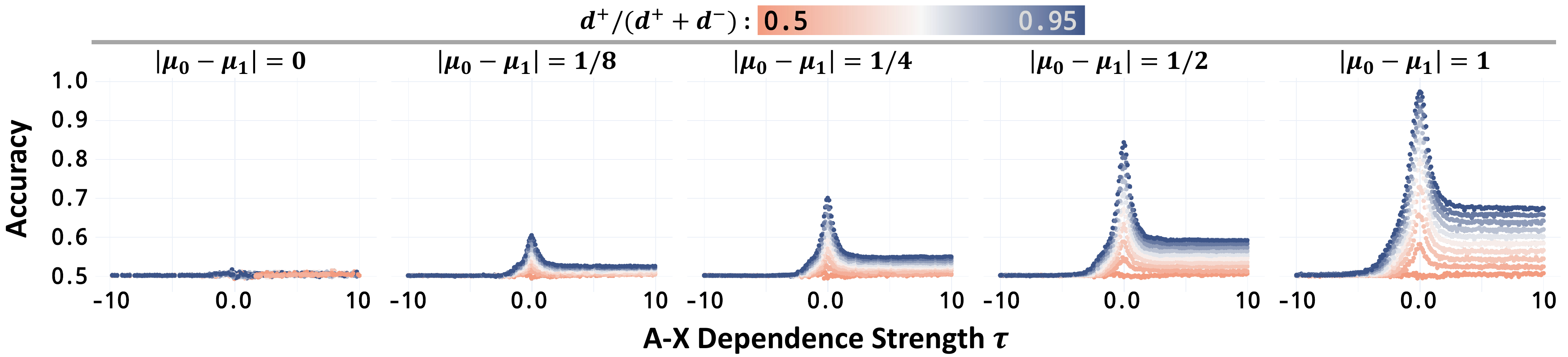}

    \caption{\label{fig:csbm-x-tau-exp}
    \bolden{The Simplified GNN Performance in \gmodel Graphs: \textit{Wider $\tau$ Range}}.
    With $\tau \in \set{-10, -9.9, \ldots, 9.9, 10}$, 
    the findings are consistent with those from Fig.~\ref{fig:csbm-x-main-exp}.
    }
\end{center}
\end{figure*}

\section{In-Depth Analysis of \gmodel}\label{app:csbmx}

In this section, we provide the full experimental result with \gmodel, together with its formal description. 
We use the same setting as in Sec.~\ref{sec:graph_model}, unless otherwise specified.

\subsection{Full Experiments: Large Range of $\tau$}

Fig.~\ref{fig:csbm-x-tau-exp} shows the experimental results with larger range of $\tau \in \set{-10, -9.9, \ldots, 9.9, 10}$.
That is, the \gmodel graph $\mathcal{G}$ has extremely large CFH \hdot.
Still, the results are consistent with the conclusion of Sec.~\ref{sec:graph_model}.

\subsection{Full Experiments: Feature Parameter Variations}

\textbf{\textit{{High-dimensional features}}}.
Fig.~\ref{fig:csbm-x-param-exp} shows the experimental results with feature dimension $k \in \set{4,16}$.
Specifically, we let $\mu_0 = - \mu_1$, and all elements within each mean vector are identical (i.e., $\mu_0 = [c,c, \ldots, c] \in \mathbb{R}^k$ and $\mu_1 = [-c,-c, \ldots ,-c] \in \mathbb{R}^k$, where $c$ is a constant). 
To control FD, we generate \gmodel graph $\mathcal{G}$'s with 
(\textit{i}) $2c \in \set{0,~ 1/8,~ 1/4,~ 1/2,~ 1}$ and 
(\textit{ii}) $\Sigma_0 = \Sigma_1 = \operatorname{diag}(\mathbf{1})$.
% The higher dimension appears to dampen the beneficial effect of $\tau$. 
The results are consistent with the conclusion of Sec.~\ref{sec:graph_model}.

\textbf{\textit{{Imbalanced feature variances}}}.
Fig.~\ref{fig:csbm-x-param-exp} shows the experimental results with imbalanced feature variances (i.e., $\Sigma_0 \neq \Sigma_1$).
To control FD with imbalanced feature variances, 
we generate \gmodel graph $\mathcal{G}$'s with $\Sigma_0 = 1$ and $\Sigma_1 \in \set{0.5, 0.25}$.
The results are consistent with the conclusion of Sec.~\ref{sec:graph_model}.

\subsection{Full Experiments: Graph Convolution Variations}

\textbf{\textit{{The number of graph convolution layers}}}.
Fig.~\ref{fig:csbm-x-model-exp} shows the experimental results with two graph convolution layers.
Specifically, we use $(D^{-1}A)^2X$ as the simplified GNN model.
We find that with 2 layers, the beneficial effect of small $\tau$ is larger. 
The finding possibly relates to the sparse topology of the generated \gmodel graph $\mathcal{G}$'s,
\footnote{Recall that the number of nodes is 10,000, whereas the node degree is 20 for all nodes.}
such that two convolution layers do not trigger over-smoothing.
Overall, the results are consistent with the conclusion of Sec.~\ref{sec:graph_model}.

\textbf{\textit{{Symmetric normalized graph convolution}}}.
Fig.~\ref{fig:csbm-x-model-exp} shows the experimental results with symmetrically normalized graph convolution.
Specifically, we use $((D+I)^{-\frac{1}{2}}(A+I)(D+I)^{-\frac{1}{2}})X$ as the simplified GNN model, where $I \in \mathbb{R}^{n \times n}$ is the identity matrix.
To do so, we conduct two pre-processing for \gmodel graph $\mathcal{G}$.
First, since the symmetric normalization assumes an undirected graph, we transform all its directed edges into (unweighted) undirected edges.
Second, all nodes have added self-loops.
Even with symmetric normalization, the results are consistent with the conclusion of Sec.~\ref{sec:graph_model}.

The extensive experiments empirically support our conclusion that \ax mediates the effect of graph convolution.

\begin{figure*}[t]
\begin{center}
    % \hspace{10mm}
    
    \includegraphics[width=\textwidth]{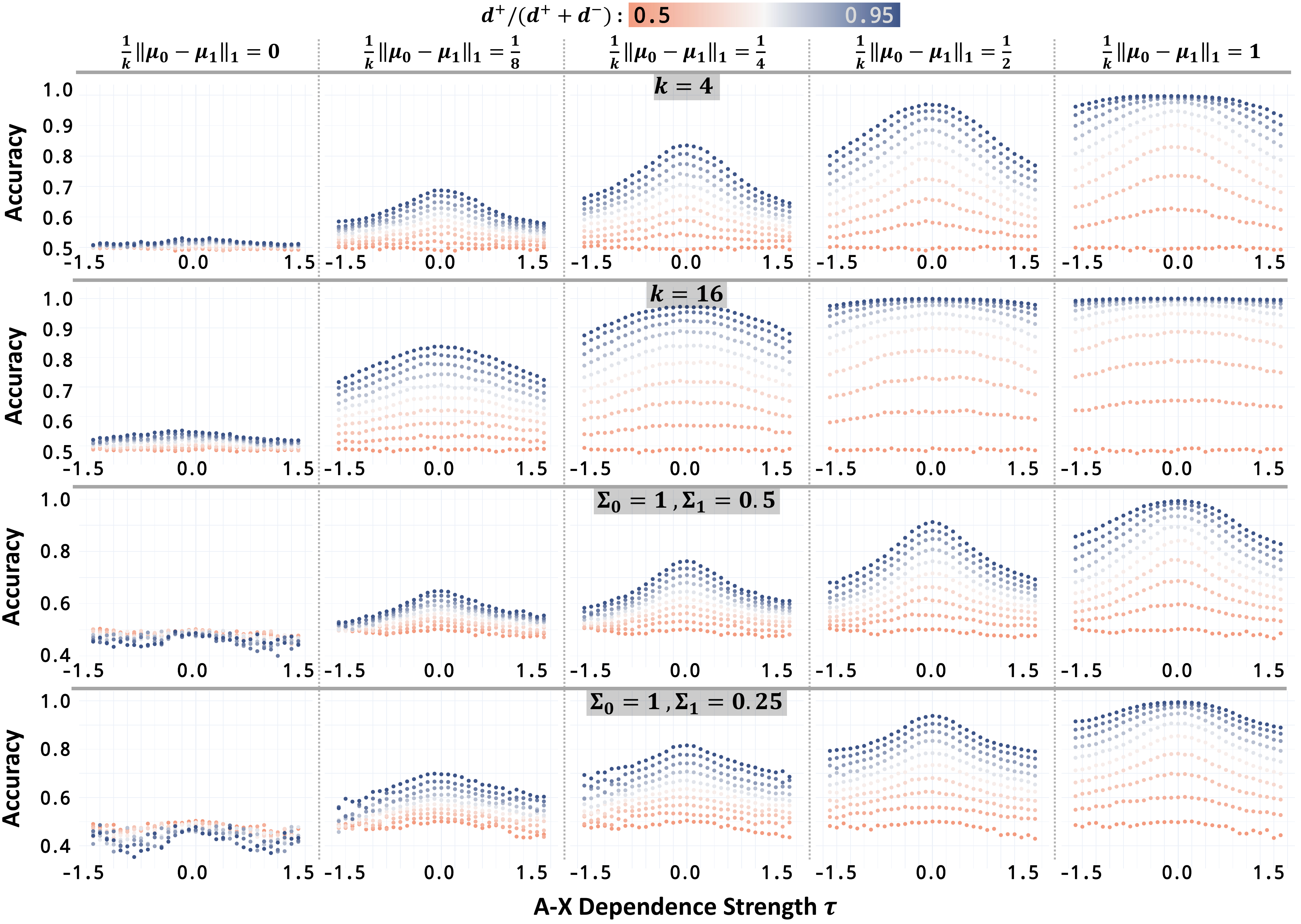}

    \caption{\label{fig:csbm-x-param-exp}
    \bolden{The Simplified GNN Performance in \gmodel Graphs: \textit{Feature Variations}.}
    With different feature parameter configurations, 
    including (\textit{i}) high-dimensional features (i.e. $k>1$) and 
    (\textit{ii}) imbalanced variances (i.e. $\Sigma_0 \neq \Sigma_1$), 
    the findings are consistent with those from Fig.~\ref{fig:csbm-x-main-exp}.
    }
\end{center}
\end{figure*}
\begin{figure*}[t]
\begin{center}
    % \hspace{10mm}
    
    \includegraphics[width=\textwidth]{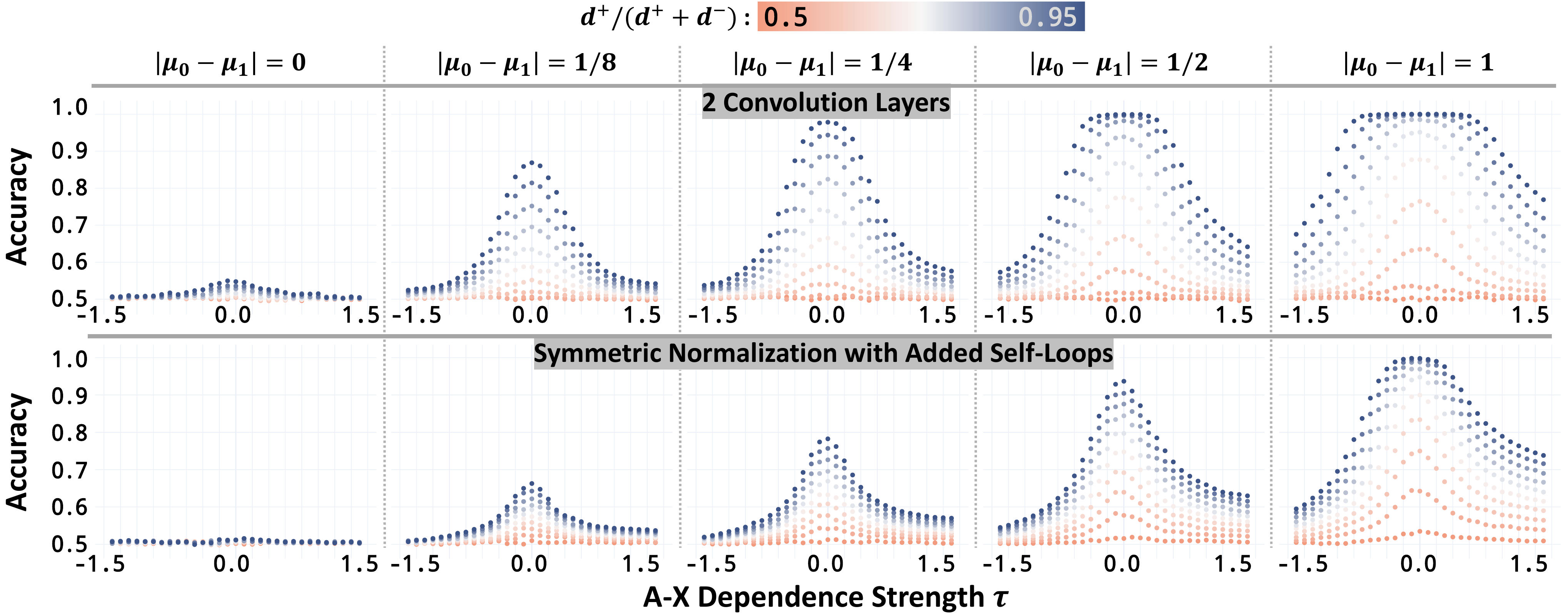}

    \caption{\label{fig:csbm-x-model-exp}
    \bolden{The Simplified GNN Performance in \gmodel Graphs: \textit{GNN Variations}.}
    With two simplified GNN variations, including graph convolution with (\textit{i}) two layers and (\textit{ii}) symmetrically normalized adjacency matrix with self-loops, the findings are consistent with those from Fig.~\ref{fig:csbm-x-main-exp}.
    }
\end{center}
\end{figure*}

\subsection{CSBM-X: Formal Description}
Here, we provide a formal mathematical description of CSBM-X. Note that we use slightly different notations from Sec.~\ref{sec:graph_model}.

{\textbf{\textit{Input.}}} 
{We consider a binary class setting (WLOG class 0 and 1).
Each input parameter is as:}
number of nodes $n$ (we assume that $n$ is even), 
feature mean vector $(\mu_{0}, \mu_{1})$ and feature covariance matrix $(\Sigma_0, \Sigma_1)$, {each corresponds to the class $\ell$,} 
same- and different-class degree $(d^+, d^-)$, respectively,
and \ax strength $\tau$.
Let $\mathcal{I} \coloneqq (n, \mu_0, \mu_1, \Sigma_0, \Sigma_1, d^+, d^-)$ denotes the set of input parameters.

\textbf{\textit{{Node classes.}}}
Given input $\mathcal{I} = (n, \mu_0, \mu_1, \Sigma_0, \Sigma_1, d^+, d^-, \tau)$, the node set is (deterministically) $V = V(\mathcal{I}) = V(n) = [n]$, determined by $n$ only.
Hence, $v_i = i, \forall i \in [n]$.
%We consider two different classes, represented by $0$ and $1$, respectively, and
We assume that the numbers of nodes in two classes are the same, i.e., $\frac{n}{2}$.
The node classes is represented by a vector 
\begin{itemize}
    \item $Y \in \mathcal{Y}_n \coloneqq \{\set{0,1}^n \colon \sum_{i \in [n]} Y_i = \frac{n}{2}\}$, 
\end{itemize}
where $\mathcal{Y}_n$ is the possible set of node-class vectors.
For each $Y \in \mathcal{Y}_n$,
\begin{itemize}
    \item $\Pr[Y | \mathcal{I}] = \Pr[Y | n] = \frac{1}{\abs{\mathcal{Y}_n}} = \frac{1}{\binom{n}{\frac{n}{2}}}$, 
\end{itemize}
where the probability of $Y$ is only decided by $n$, independent of the other parameters in the input $\mathcal{I}$, and all the possible $Y$'s have the same probability (i.e., follow a uniform distribution on $\mathcal{Y}_n$).

\textbf{\textit{{Node features.}}}
Assume the feature dimension is $k \in \bbN$.
Conditioned on the node classes $Y$, the features $X \in \bbR^{n \times k}$ follow the corresponding Gaussian distributions, where
each node feature $X_i \in \bbR^k$ {is an i.i.d. sample} from a Gaussian with mean $\mu_{Y_i}$ and variance $\Sigma_{Y_i}$.
Specifically,
\begin{itemize}
    \item $\Pr[X_i | Y, \calI] = \Pr[X_i | Y_i, \mu_{Y_i}, \Sigma_{Y_i}] = (2\pi)^{-k/2}\det (\Sigma_{Y_i})^{-1/2} \, \exp \left( -\frac{1}{2} (X_i - \mu_{Y_i})^\mathrm{T} \Sigma_{Y_i}^{-1}(X_i - \mu_{Y_i}) \right)$, 
\end{itemize}
which is the PDF of multivariate Gaussian $\calN(\mu_{Y_i}, \Sigma_{Y_i})$.
%Moreover, $\Pr[X | Y, \calI] = \prod_{i \in [n]} \Pr[X_i | Y, \calI] = \prod_{i \in [n]} \Pr[X_i | Y_i, \mu_{Y_i}, \Sigma_{Y_i}]$.

\textbf{\textit{{Edges.}}}
Conditioned on node classes and features, directed edges are sampled by weighted sampling without replacement.
For each node $i \in [n]$, we define
\begin{itemize}
    \item $\mathbf{C}^{+}_{i} = \mathbf{C}^{+}_{i}(Y) = {C}^{+}_{Y_{i}} \setminus \set{v_i} = \set{\mathbf{C}^{+}_{i, 1}, \mathbf{C}^{+}_{i, 2}, \ldots, \mathbf{C}^{+}_{i, \abs{\mathbf{C}^{+}_i}}}$ (fix order of the nodes in $\mathbf{C}^{+}_{i}$),
    \item $\mathbf{C}^{-}_{i} = \mathbf{C}^{-}_{i}(Y) = {C}^{-}_{Y_{i}} = \set{\mathbf{C}^{-}_{i, 1}, \mathbf{C}^{-}_{i, 2}, \ldots, \mathbf{C}^{-}_{i, \abs{\mathbf{C}^{-}_i}}}$,
    \item $\mathbf{\Phi}^+_i = \mathbf{\Phi}^+_i(\mathbf{C}^{+}_{i}, X) = (\operatorname{exp}({\tau \mathbf{h}^{(p)}_{ij}}), j = \mathbf{C}^{+}_{i, t} \colon t \in [\abs{\mathbf{C}^{+}_i}]) \in \bbR^{\abs{\mathbf{C}^{+}_{i}}}$,
    \item $\mathbf{\Phi}^-_i = \mathbf{\Phi}^-_i(\mathbf{C}^{-}_{i}, X) = (\operatorname{exp}({\tau \mathbf{h}^{(p)}_{ij}}), j = \mathbf{C}^{-}_{i, t} \colon t \in [\abs{\mathbf{C}^{-}_i}]) \in \bbR^{\abs{\mathbf{C}^{-}_{i}}}$.
\end{itemize}
Note that $X$ is used here to compute $\mathbf{h}^{(p)}_{ij}$'s.
Let $\mathbf{C}^{+}$ denote $(\mathbf{C}^{+}_i \colon i \in [n])$, and $\mathbf{C}^{-}$, $\mathbf{\Phi}^+$, and $\mathbf{\Phi}^-$ are similarly defined.

For each node $i \in [n]$, 
let $N_i$ denote its neighbor set, which is a random variable here.
Recall that $N_i = \set{j \in [n] \colon (i, j) \in E}$.
The set of all possible $N_i$'s are 
\begin{itemize}
    \item $\calB_i = \calB_i(\mathbf{C}^{+}_{i}, \mathbf{C}^{-}_{i}, d^+, d^-) \coloneqq \set{N^+_i \cup N^-_i \colon N^+_i \in \calB^+_i, N^-_i \in \calB^-_i}$, where
    \item $\calB^+_i = \calB^+_i(\mathbf{C}^{+}_{i}, d^+) = \set{N^+_i \colon N^+_i \subseteq \mathbf{C}^{+}_{i}, \abs{N^+_i} = d^+}$,
    \item $\calB^-_i = \calB^-_i(\mathbf{C}^{-}_{i}, d^-) = \set{N^-_i \colon N^-_i \subseteq \mathbf{C}^{-}_{i}, \abs{N^-_i} = d^-}$.
\end{itemize}

For each possible $N^+_i =\set{N^+_{i, 1}, N^+_{i, 2}, \ldots, N^+_{i, d^+}} \in \calB^+_i$,
\begin{itemize}
    \item $\Pr[N^+_i | \mathbf{C}^{+}_i, \mathbf{\Phi}^{+}_i, d^+] = \sum_{\pi \in S_{d^+}} \prod_{t=1}^{d^+} {(\mathbf{\Phi}^{+}_{N^+_{\pi(i, t)}})} / {(1-
\sum_{t' = 1}^{t - 1} \mathbf{\Phi}^{+}_{N^+_{\pi(i, t')}})}$,
    \item $\Pr[N^-_i | \mathbf{C}^{-}_i, \mathbf{\Phi}^{-}_i, d^-] = \sum_{\pi \in S_{d^-}} \prod_{t=1}^{d^-} {(\mathbf{\Phi}^{-}_{N^-_{\pi(i, t)}})} / {(1-
\sum_{t' = 1}^{t - 1} \mathbf{\Phi}^{-}_{N^-_{\pi(i, t')}})}$,
\end{itemize}
where $S_{d^+}$ and $S_{d^-}$ are the sets of all permutations on $[d^+]$ and $[d^-]$, respectively.

For each possible $N_i = N^+_i \cup N^-_i \in \calB_i$,
\begin{itemize}
    \item $\Pr[N_i | \mathbf{C}^{+}_i, \mathbf{\Phi}_i^{+}, \mathbf{C}^{-}, \mathbf{\Phi}_i^{-}, d^+, d^-] = \Pr[N^+_i | \mathbf{C}^{+}_i, \mathbf{\Phi}_i^{+}, d^+] \Pr[N^-_i | \mathbf{C}^{-}, \mathbf{\Phi}_i^{-}, d^-]$.
\end{itemize}

The neighbor set of each node is sampled independently, i.e.,
\begin{itemize}
    \item $\Pr[(N_i \colon i \in [n]) | \mathbf{C}^{+}_i, \mathbf{\Phi}_i^{+}, \mathbf{C}^{-}, \mathbf{\Phi}_i^{-}, d^+, d^-] = \prod_{i \in [n]} \Pr[N_i | \mathbf{C}^{+}_i, \mathbf{\Phi}_i^{+}, \mathbf{C}^{-}, \mathbf{\Phi}_i^{-}, d^+, d^-]$.
\end{itemize}
Here, $N_i, \forall i \in [n]$ fully determine the topology of each generated graph.

\subsection{CSMB-X2: a Multi-Community, Degree-Preserving CSBM-X}
While the proposed \gmodel can capture many properties of the real-world graphs, such as sparsity, class-homophily, and CFH, it does not reflect some of their key properties.
Thus, we further propose CSBM-X2 to reflect a multi-community structure and varying degree distributions.

\smallsection{CSBM-X2 design.} 
Two major differences are noticeable in comparison to \gmodel.
First, the number of classes can be larger than 2, allowing for a multi-community structure.
Second, the vector $\mathbf{d} \in \mathbb{R}^n$ representing node degrees for each node serves as the degree parameter.

We describe the key differences of the CSBM-X2.
Specifically, to allow for a multiple community structure, the nodes are equally divided into $c$ classes (instead of 2 in CSBM-X).
Also, for each node $v_i \in V$, CSBM-X2 assigns its node degree $\mathbf{d}_i$.
By the same-class degree ratio parameter $r$, the node $v_i$'s degree $\mathbf{d}_i$ is divided into the same-class and different-class degrees $(d_i^+, d_i^-)$.
Then, CSBM-X2 samples same- and different-class neighbors for each node by their sampling weights, like in CSBM-X.

\smallsection{Experiment setting.}
We generate CSBM-X2 graphs with various parameter configurations.
The setting generally follows the one in Sec.~\ref{sec:graph-model-exp}.
We fix the number of nodes $n=10,000$ and number of classes $c=10$. 
We use input node degrees $\mathbf{d}$ sampled from Pareto distribution (a power-law distribution), using Pareto's $\alpha=1.5$. 
Since the sampled values are bounded by $[0, \infty)$, 
$d_{min} = 20$ is added to the sampled value, and $d_{max} = 1,000$ is used to clip it.
Thus, the node degree is bounded $d_i \in [d_{min}, d_{max}]$, $\forall v_i \in V$.

We set $c$-dimensional features, such that the feature mean vectors are $\mu_\ell \in \mathbb{R}^c$ and covariance matrices are $\Sigma_\ell \in \mathbb{R}^{c \times c}, \forall \ell \in [c]$.
The generated CSBM-X2 graphs have a wide range of {feature distance} FD, {class homophily} \hc, and {CFH} \hdot.
\vspace{-2mm}
\begin{itemize}[leftmargin=*]
    \item {FD}: $\Vert \mu_{\ell_0} - \mu_{\ell_1} \Vert_1 \in \set{0, ~1/4, ~1/2, ~1, ~2}$; $\Sigma_{\ell} = \mathbf{I}$.
    \item {\hc}: $r \in \set{0.50, 0.55, ..., 0.95}$.
    \item {\hdot}: $\tau \in \set{-1.5, -1.4,\ldots, -0.1,0,0.1,\ldots, 1.4, 1.5}$.
\end{itemize}
\vspace{-2mm}

\smallsection{Experiment results.}
Figure~\ref{fig:rebuttal-csbm-x} shows the experimental results with CSBM-X2. 
We observe consistent findings from the experiments with \gmodel (Figure~\ref{fig:csbm-x-main-exp}), 
such that the findings 1 and 2 from Sec.~\ref{sec:graph-model-exp} are reproduced with  CSBM-X2.
Specifically, the simplified GNN performance gradually increases over decreasing CFH magnitude $\vert \mathbf{\tilde{h}(\cdot)} \vert$, 
and the effect of CFH \hdot on GNN performance is moderated by feature distance FD and class homophily \hc.
Our results highlight the generalizability of our conclusion that \ax mediates the effect of graph convolution.

\begin{figure*}[htb]
\begin{center}
    \includegraphics[width=\textwidth]{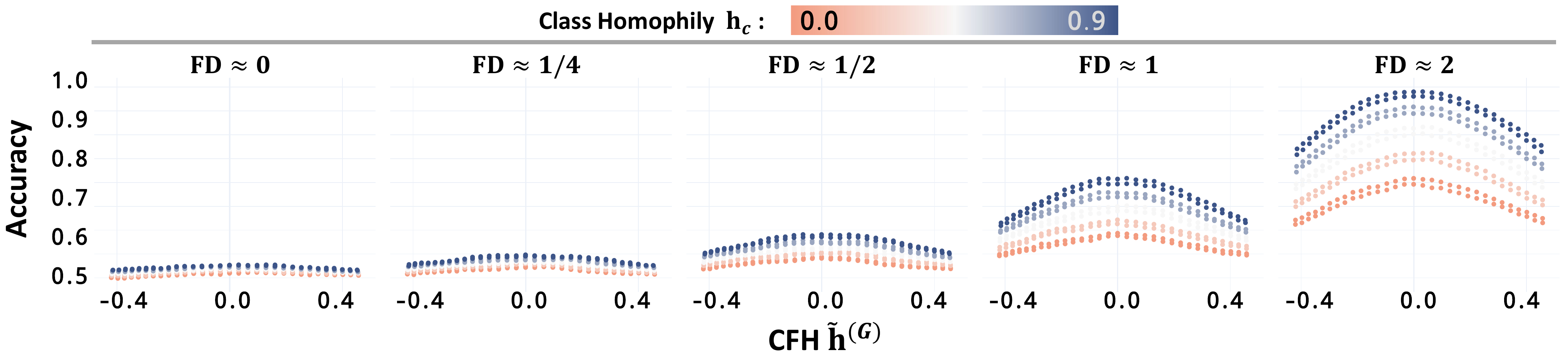}

    \caption{\label{fig:rebuttal-csbm-x}
    \bolden{The Simplified GNN Performance in CSBM-X2 Graphs.}
    The degrees follow the power-law distribution, and the graphs have multi-community structures.
    Consistent with Theorem~\ref{thm:main_result} and Figure~\ref{fig:csbm-x-main-exp},
    for given feature distance $\text{FD} > 0$ and class homophily $\mathbf{h}_c >0$,
    the simplified GNN performance increases as graph-level CFH $\mathbf{\Tilde{h}}^{(G)} \rightarrow 0~(\tau \rightarrow 0)$.
    }
\end{center}
\end{figure*}
\clearpage
\section{Additional Experimental Results with Feature Shuffle}\label{app:shuffle}

In this section, we provide additional experimental results of the feature shuffle with real-world graphs.
We use the same setting as in Sec.~\ref{sec:swapping_exp}, unless otherwise specified.

\textbf{\textit{Models}}.
First, Fig.~\ref{fig:shuffle_exp(full)}(a) shows GCN performance after the feature shuffle in 6 high \hc and 6 low \hc datasets.
While GCN performance increases over the feature shuffles, 
GCNII benefits more from the feature shuffle (i.e., larger positive slopes by GCNII). 
This outcome may relate to the difference in their number of layers.
We claim two complementary pieces of evidence.
For one, \gmodel experiment in Fig.~\ref{fig:csbm-x-model-exp} suggests that a larger number of layers can further improve the beneficial effect of small $\tau$.
Also, one of the main differences between the GCN and GCNII is their capability in stacking deeper layers.
The relationship between GNN depth and \ax, however, is beyond the scope of the present work, and we leave it up to future studies.
Overall, consistent with the conclusion of Sec.~\ref{sec:swapping_exp}, we conclude that all GNN models benefit from the feature shuffle.

\textbf{\textit{Train ratios}}.
Second, Fig.~\ref{fig:shuffle_exp(full)}(b) shows GCNII performance over the feature shuffle with different splits. 
We use three different train/val splits while fixing the test split.
Two findings are worth noting. 
First, model performance increases over the feature shuffles in all splits, 
highlighting that our conclusion is consistent with varying train and validation node ratios.
Second, the performance gap between different splits generally reduces over the increasing shuffled node ratio.
That is, the effect of feature shuffle may also interact with the number of train labels, hinting that \ax may influence the generalization capacity of GNNs.
Analysis of GNN generalization, however, is beyond the scope of the present work, and we leave it up to future studies.

% \textbf{\textit{Noisy features}}.
% Third, Fig.~\ref{fig:shuffle_exp(full)}(c) shows the full results of Fig.~\ref{fig:shuffle-feat} for all 12 high \hc datasets.
% We draw the same conclusion as in Sec.~\ref{sec:swapping_exp}.

% \textbf{\textit{Proximity-based features}}.
% Last, Fig.~\ref{fig:shuffle_exp(full)}(d) shows the full results of Fig.~\ref{fig:shuffle-feat} for all 12 high \hc datasets.
% We draw the same conclusion as in Sec.~\ref{sec:swapping_exp}.

The extensive experiments empirically support our conclusion that \ax mediates the effect of graph convolution.

\begin{figure}[t] 
    \centering

    \begin{subfigure}[b]{\linewidth}
    \centering
    \includegraphics[width=\linewidth]{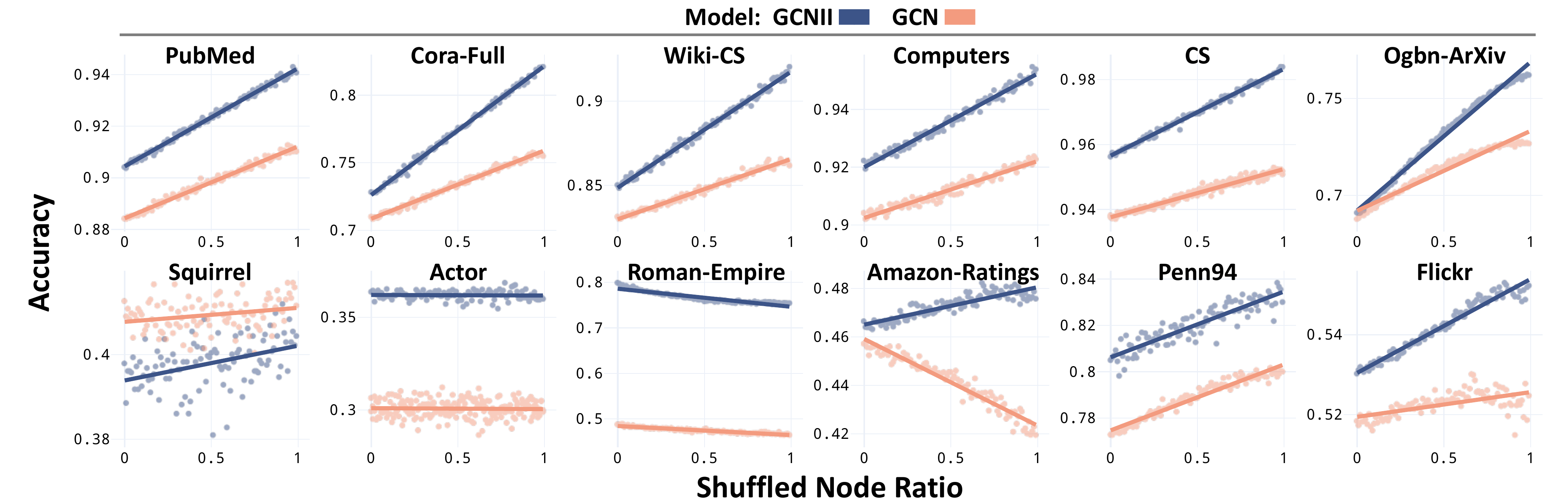}
    \caption{Additional results on the feature shuffle: Models}
    \end{subfigure}
    
    \vspace{2mm}

    \begin{subfigure}[b]{\linewidth}
    \centering
    \includegraphics[width=\linewidth]{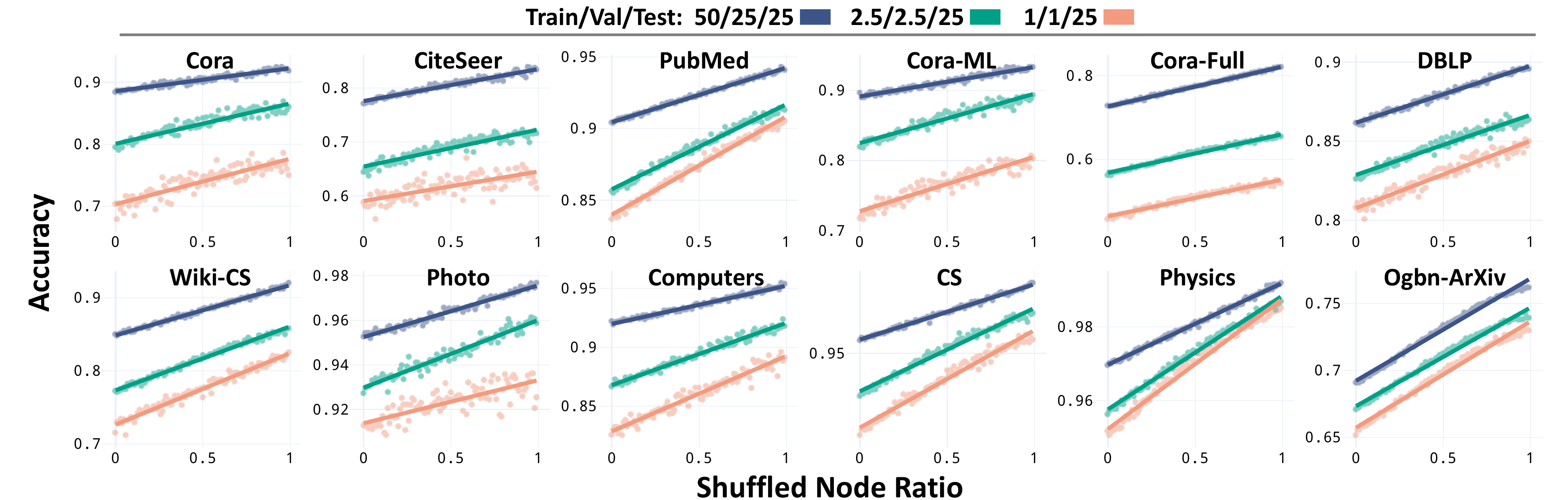}
    \caption{Additional results on the feature shuffle: Train node ratio}
    \end{subfigure}

    \caption{\label{fig:shuffle_exp(full)}
    \bolden{Additional Results on the Feature Shuffle.}
    }
    \vspace{2mm} 
\end{figure}

% \subsection{Feature Shuffle Algorithm}

% \subsection{Feature Shuffle: The Full Results}

% \smallsection{The role of SNR}.

% \smallsection{The effect of train node ratio}.

% \smallsection{The effect of input features}.

\section{Potential Applications}\label{app:applications}

In this section, we explore a potential application of our findings with node feature shuffle.
As discussed in Section~\ref{sec:swapping_exp}, the feature shuffle can significantly enhance GNN node classification.
However, the algorithm requires test labels, which are not known.
We, thus, propose an algorithm that extends the feature shuffle for practical scenarios where the labels of many nodes are unknown.

\smallsection{Algorithm.}
Here, we describe our proposed pseudo-label-based feature shuffle algorithm. 
The algorithm has three steps: 

\textbf{\textit{Step 1) Initial GNN training.}}
Consider we are given a graph $G=(V,E)$, node features ${X}$, and node classes of train/val nodes $V^{(train/val)} \subset V$.
We train a node classification GNN, which we denote as $\texttt{GNN} : (X,E) \mapsto \mathbf{\hat{Y}}$, where $\mathbf{\hat{Y}} \in [0,1]^{n \times c}$ is a node class prediction matrix.

\textbf{\textit{Step 2) Pseudo-labeling.}}
By using the trained \texttt{GNN}, we assign pseudo-labels to unlabeled nodes.
First, we choose nodes to be labeled by using confidence scores of \texttt{GNN}.
To this end, we use prediction matrix $\mathbf{\hat{Y}}$ of \texttt{GNN}.
We label nodes whose prediction probability lies within the predefined confidence range.
Specifically, we define labeling nodes $V^{(label)}$ as follows: 
\begin{equation}
    V^{(label)} = \{v_{k} \in V : o_{l} \leq \max_{j \in [c]}\mathbf{\hat{Y}}_{kj}  \leq o_{u}\} \cup V^{(train/val)},
\end{equation}
where $0 \leq o_{l} \leq  o_{u} \leq 1$ are hyperparameters.
After choosing labeling nodes $V^{(label)}$, we assign pseudo-labels $y'_{k},\forall v_{k} \in V^{(label)}$ as follows:
\begin{equation}
    y'_{k}=\begin{cases}
			Y_{k} & \text{if $v_{k} \in V^{(train/val)}$}\\
            \arg\max_{j \in [c]}\mathbf{\hat{Y}}_{kj} & \text{otherwise}
		 \end{cases}, \forall v_{k} \in V^{(label)}.
\end{equation}  

\textbf{\textit{Step 3) Classwise feature shuffle.}}
Lastly, with obtained pseudo-labels $y_{k}, \forall v_{k} \in V^{(label)}$, we shuffle node features that share the same pseudo-label.
Specifically, we shuffle the features of all nodes that share the same pseudo-labels, without considering the train/val/test split.
Train/val/test node indices and true node labels are not shuffled.
By performing the shuffle, we obtain new node features $X'$.

\textbf{\textit{Step 4) GNN fine-tuning.}}
Finally, we fine-tune the trained \texttt{GNN} with the shuffled node features $X'$. Specifically, 
$\texttt{GNN} : (X',E) \mapsto \mathbf{\hat{Y}}$. We make the test inference with the fine-tuned \texttt{GNN} and shuffled feature $X'$.

\smallsection{Experiment.}
We test the algorithm on the node classification benchmark datasets.
We use 12 high class-homophily graphs.
Similar experimental procedures and hyperparameter settings are used as in Sec.~\ref{sec:exp_setting} and Appendix~\ref{app:train_details}.
To focus on more practical scenarios, we use 20 train nodes and 30 validation nodes per class, whereas the rest of the nodes serve as the test nodes.
We use GCNII as the GNN encoder for the pseudo-label-based feature shuffle algorithm.
For pseudo-labeling hyperparameters, we fix $o_{l}=0.7$ and $o_{u}=1$.
We compare the performance of vanilla GCNII versus GCNII enhanced with pseudo-label-based feature shuffle.
We repeat 30 trials and report the mean performance in Table~\ref{tab:pseudo-label-shuffle}.
In 9/12 datasets, the enhanced GCNII outperforms the vanilla GCNII by a small margin.

\begin{table*}[t]
\begin{center}
\caption{Node Classification Performance of a GNN Enhanced with Pseudo-label-based Feature Shuffle Algorithm} \label{tab:pseudo-label-shuffle}
    \resizebox{\textwidth}{!}{
    \renewcommand{\arraystretch}{0.9}
        \centering
        \begin{tabular}{*{13}{c}}
            \toprule            
            \textbf{Method} & Cora & CiteSeer & PubMed & Cora-ML & DBLP & Wiki-CS & Cora-Full & Photo & Computers & CS & Physics & Ogbn-ArXiv \\
            \midrule
            \textbf{Vanila GCNII} & 79.70 & 67.26 & 76.23 & 82.05 & 75.42 & 73.37 & 60.94 & 90.44 & 81.10 & 91.31 & 92.97 & 51.60\\
            \textbf{Enhanced GCNII} & 80.83 & 68.02 & 76.26 & 82.77 & 77.07 & 73.16 & 61.58 & 90.50 & 81.66 & 91.27 & 93.31 & 50.80 \\
            \textbf{Performance Gap} & +1.13 & +0.76 & +0.03 & +0.72 & +1.65 & -0.21 & +0.64 & +0.06 & +0.56 & -0.04 & +0.34 & -0.80 \\
            \midrule            
            \bottomrule
        \end{tabular}
        }
\end{center}

\vspace{1.5mm}

\end{table*}

\section{Experiment Settings: Pre-processing, Training, Hyperparameters, and Details}\label{app:train_details}

\subsection{Dataset Pre-processing}

\textbf{\textit{Measurement}}.
No dataset pre-processing is done when measuring class homophily \hc and feature distance FD.
If the dataset has self-loops, they are removed when measuring CFH \hdot.

\textbf{\textit{\gmodel}}.
For experiments with symmetrically normalized graph convolution in Fig.~\ref{fig:csbm-x-model-exp}, (\textit{i}) directed edges are converted into undirected edges (without edge weights) and (\textit{ii}) self-loops are added.
In other experiments, no dataset pre-processing is done.

\textbf{\textit{The real-world graphs}}.
All the considered GNN models assume undirected graph topology.
Thus, directed edges are converted into undirected edges (without edge weights).
Also, self-loops are added.

\subsection{Model Training}

All models are trained with Adam~\cite{Kingma2015Adam:Optimization} optimizer.
We fix 500 train epochs. 
The best model is chosen based on early stopping, with a patience of 100.
In feature shuffle experiments (Sec.~\ref{sec:swapping_exp}, Appendix~\ref{app:shuffle}), 
a new model is initialized and trained for each shuffled graph.

\subsection{Hyperparameters}

In \gmodel experiments (Sec.~\ref{sec:graph_model}, Appendix~\ref{app:csbmx}), 
we do not tune hyperparameters since the coefficient $W \in \mathbb{R}$ is the only learnable parameter.
In feature shuffle experiments (Sec.~\ref{sec:swapping_exp}, Appendix~\ref{app:shuffle}), 
we tune the hyperparameters on the \textit{original graphs}.
That is, the feature-shuffled graphs are unknown to the models during the hyperparameter search.

For all models, we set their hidden feature dimension as 64 and the learning rate as 0.01.
Below, we provide the hyperparameter search space for each considered model.

\begin{enumerate}
    \item \textbf{GCN}: 
    \begin{itemize}[leftmargin=-0.1mm]
        \item $\text{Optimizer weight decay} \in \{5e-3, 1e-3, 5e-4, 1e-4\}$
        \item $\text{Dropout} \in \{0.5, 0.6, 0.7\}$
        \item Number of layers $\in$ $\{2, 3, 4\}$
    \end{itemize}   
    \item \textbf{GCN-II}: 
    \begin{itemize}[leftmargin=-0.1mm]
        \item {Optimizer weight decay} $\in$ $\{1e-3, 5e-4, 1e-4, 5e-5\}$
        \item Dropout $\in$ $\{0.5, 0.6, 0.7\}$
        \item Number of layers $\in$ $\{4, 8, 16\}$
        \item Residual connection weight $\alpha$ $\in$ $\{0.1, 0.3, 0.5\}$
        \item Weight decay $\lambda$ $\in$ $\{0.5, 1.0, 1.5\}$
    \end{itemize}    
    \item \textbf{GPR-GNN}: 
    \begin{itemize}[leftmargin=-0.1mm]
        \item $\text{Optimizer weight decay} \in \{5e-3, 1e-3, 5e-4, 1e-4\}$
        \item Dropout $\in$ $\{0.5, 0.6, 0.7, 0.8\}$
        \item Number of layers  $\in$ $\{10\}$
        \item Return probability $\alpha$ $\in$ $\{0.1, 0.3, 0.5\}$    
    \end{itemize}
    \item \textbf{AERO-GNN}: 
    \begin{itemize}[leftmargin=-0.1mm]
        \item {Optimizer weight decay} $\in$ $\{5e-3, 1e-3, 5e-4\}$
        \item Dropout $\in$ $\{0.5, 0.6, 0.7\}$
        \item Number of MLP layers $\in$ $\{1, 2\}$
        \item Number of convolution layers  $\in$ $\{4, 8, 16\}$
        \item Weight decay $\lambda$ $\in$ $\{0.5, 1.0, 1.5\}$
    \end{itemize}    
\end{enumerate}

\subsection{Other Details}

\textbf{\textit{Train/val/test split}}.
For each node class, the train/val/test set is split randomly by the ratio of 50/25/25, unless otherwise specified.
In \gmodel experiments (Sec.~\ref{sec:graph_model}, Appendix~\ref{app:csbmx}), for each generated \gmodel graph $\mathcal{G}$, we obtain 5 different splits.
In feature shuffle experiments (Sec.~\ref{sec:swapping_exp}, Appendix~\ref{app:shuffle}), we use 5 different splits consistent across the shuffled node ratio.

% \textbf{\textit{Noisy features}}.
% To obtain noisy features for Figs.~\ref{fig:shuffle-feat} and~\ref{fig:shuffle_exp(full)}, we (\textit{i}) randomly chose 50\% of all nodes and (\textit{ii}) randomly permute their feature vectors \textit{irrespective of their class}. 
% Thereby, we contaminate the dependence between node features and class.
% This is distinguished from the feature shuffle since the feature shuffle is done only among the same class nodes  (Fig.~\ref{fig:strange_phe}).

\textbf{\textit{Node2Vec}}.
For Fig.~\ref{fig:shuffle-feat}, we use Node2Vec~\cite{Grover2016Node2vec:Networks} as the node features.
For each graph, the Node2Vec vector is 256-dimensional.
To obtain the vector, we train the Node2Vec model with a walk length of 20, a context size of 10, walks per node of 10, and 100 epochs.

\textbf{\textit{Shuffle and CFH \hdot}}.
When measuring CFH after the feature shuffle, we average the outcomes over 5 trials.

\textbf{\textit{Evaluation of WebKB datasets}}.
The WebKB datasets (i.e., Texas, Cornell, and Wisconsin) have very small number of nodes, ranging from 183 to 251.
Thus, the variance of GNN node classification accuracy on the datasets often tend to be very large, because mis-classifying one test node can drop near 2\%p test accuracy.
To enhance reliability of the empirical results in Sec.~\ref{sec:swapping_exp}, we report the mean performance \textit{over 30 trials} only for the WebKB datasets.

\end{document}